\documentclass{ifm-tr}
\usepackage{amsmath,amssymb,mathtools}
\usepackage{amsthm}
\usepackage{tabularx}
\usepackage{nicematrix}
\usepackage{colortbl}
\usepackage{enumitem}
\usepackage{etoc}
\usepackage{wrapfig}
\usepackage{xspace}
\usepackage{algorithm}
\usepackage{algpseudocode}
\usepackage{fontawesome5}
\usepackage{tikz}
\usetikzlibrary{arrows.meta,positioning,calc}
\newcommand{\E}{\mathbb E}
\newcommand{\R}{\mathbb R}
\newcommand{\norm}[1]{\left\lVert #1\right\rVert}
\newcommand{\sg}{\operatorname{stopgrad}}

\newcommand{\softmax}{\operatorname{softmax}}
\newcommand{\CE}{\operatorname{CE}}
\newcommand{\ModelStep}{\operatorname{ModelStep}}
\newcommand{\Adam}{\operatorname{Adam}}
\newcommand{\EMA}{\operatorname{EMA}}

\newcommand{\secref}[1]{{Section~\ref{#1}}}
\newcommand{\appref}[1]{Appendix~\ref{#1}}

\newtheorem{lemma}{Lemma}
\theoremstyle{definition}

\theoremstyle{plain}
\setlist[itemize]{leftmargin=1.30em,itemsep=2pt,topsep=3pt}
\setlist[enumerate]{leftmargin=1.50em,itemsep=2pt,topsep=3pt}
\ifdefined\titlespacing
  \titlespacing*{\paragraph}{0pt}{0.8ex plus 0.2ex minus 0.1ex}{0.6em}
\fi
\definecolor{ifmtableink}{HTML}{202020}
\definecolor{ifmtablerule}{HTML}{202020}
\definecolor{ifmtableband}{HTML}{EFEFEF}
\newcommand{\papertablestyle}{%
  \small
  \renewcommand{\arraystretch}{1.10}%
  \setlength{\tabcolsep}{4pt}%
  \setlength{\heavyrulewidth}{0.75pt}%
  \setlength{\lightrulewidth}{0.35pt}%
  \setlength{\cmidrulewidth}{0.30pt}%
  \arrayrulecolor{ifmtablerule}%
}
\newsavebox{\papertabularbox}
\newsavebox{\papernicetabularbox}
\newcommand{\appendixtablefontsize}{6.9}
\newif\ifappendixtables
\AddToHook{cmd/appendix/after}{\global\appendixtablestrue}
\newif\ifnestedpapertable
\makeatletter
\def\papertable@captype{table}
\newcommand{\fitpapertable}[1]{%
  \ifappendixtables\ifx\@captype\papertable@captype\ifnestedpapertable\else
    \sbox#1{\scalebox{\fpeval{\appendixtablefontsize/\f@size}}{\usebox#1}}%
  \fi\fi\fi
  \ifdim\wd#1>\linewidth
    \resizebox{\linewidth}{!}{\usebox{#1}}%
  \else
    \usebox{#1}%
  \fi
}
\makeatother
\AddToHook{env/tabular/before}{\begin{lrbox}{\papertabularbox}\nestedpapertabletrue}
\AddToHook{env/tabular/after}{\end{lrbox}\fitpapertable{\papertabularbox}}
\AddToHook{env/NiceTabular/before}{\begin{lrbox}{\papernicetabularbox}\nestedpapertabletrue}
\AddToHook{env/NiceTabular/after}{\end{lrbox}\fitpapertable{\papernicetabularbox}}
\newcommand{\tablehead}[1]{{\color{ifmtableink}\bfseries #1}}
\newcommand{\tableband}[2]{%
  \multicolumn{#1}{@{}l@{}}{\rule{0pt}{2.00ex}\tablehead{#2}}\\%
}

\newcommand{\pairtextstart}{\begin{minipage}[t]{0.49\linewidth}\vspace{0pt}\raggedright}

\newcommand{\pairend}{\unskip\par\end{minipage}}

\definecolor{papergray}{HTML}{EFEFEF}
\newsavebox{\pairtablebox}
\newlength{\pairtablewidth}
\newlength{\pairprosewidth}
\newenvironment{papertablepair}[1]{%
  \par\medskip
  \def\pairtablealign{t}%
  \input{#1}%
  \setlength{\pairtablewidth}{\wd\pairtablebox}%
  \setlength{\pairprosewidth}{\dimexpr\linewidth-\pairtablewidth-18pt\relax}%
  \typeout{Natural table #1: \the\pairtablewidth; prose: \the\pairprosewidth}%
  \noindent\begin{minipage}[\pairtablealign]{\pairprosewidth}\vspace{0pt}\raggedright
}{%
  \par\end{minipage}\hfill
  \begin{minipage}[\pairtablealign]{\pairtablewidth}\vspace{0pt}
    \captionsetup{hypcap=false,justification=raggedright,singlelinecheck=false}%
    \pairtableheading
    \noindent\usebox{\pairtablebox}\par
  \end{minipage}\par\medskip
}

\definecolor{resultgrouprule}{HTML}{AEB4BE}
\newcommand{\resulttablestyle}{%
  \papertablestyle\fontsize{9}{10.80}\selectfont
  \setlength{\tabcolsep}{2.30pt}%
  \renewcommand{\arraystretch}{1.08}%
}
\newcommand{\resultmissing}{\textemdash}
\makeatletter
\newcommand{\sidefloat}[4][]{%
  \begin{minipage}[t]{#2}\vspace{0pt}\centering
    \def\@captype{#3}%
    \captionsetup{width=.95\linewidth,justification=raggedright,singlelinecheck=false}%
    \ifx\relax#1\relax\else\captionsetup{skip=#1}\fi
    \renewenvironment{table}[1][]{}{}\renewenvironment{figure}[1][]{}{}%
    \input{#4}%
  \end{minipage}}
\makeatother
\newcommand{\resultrowstretch}{1.18}

\newcommand{\titlesubtitlelist}{}
\newcommand{\titlesubtitle}[1]{\def\titlesubtitlelist{#1}}
\renewcommand{\title}[1]{\def\titlelist{#1}}

\newcommand{\patchpartititle}[2]{%
  \patchcmd{\mymaketitle}{#1}{#2}{}{%
    \PackageError{parti-title}{The IFM title template changed}{Check title_layout.tex against ifm-tr.cls.}%
  }%
}
\patchpartititle
  {{\fontsize{19pt}{16pt}\selectfont\sffamily\bfseries\color{black}\titlelist}\par}
  {{\fontfamily{lmr}\fontsize{24pt}{27pt}\selectfont
      \bfseries\color{black}\titlelist\par}%
   \ifdefempty{\titlesubtitlelist}{}{%
     \vskip 0.10cm
     {\fontfamily{lmr}\fontsize{12pt}{15pt}\selectfont
       \mdseries\color{ifmmedgray}\titlesubtitlelist\par}%
   }}
\patchpartititle{\vspace{0.15cm}}{\vspace{0.10cm}}
\patchpartititle{\vspace{0.5cm}}{\vspace{0.28cm}}
\patchpartititle{\vskip 0.3cm}{\vskip 0.24cm}
\patchpartititle{\vskip 0.15cm}{\vskip 0.10cm}
\patchpartititle{\vskip 0.05cm}{\vskip 0.04cm}
\patchpartititle{\vskip 0.25cm}{\vskip 0.18cm}
\patchpartititle{\vskip 0.3cm}{\vskip 0.22cm}

\title{Towards Looped Models Done Right}
\titlesubtitle{Part II: Rethinking \textbf{at} Fixed Points}

\metadata[\faGithub\ Code]{\url{https://github.com/ifm-ai/xllm-loop}}
\date{September 29, 2026}
\hypersetup{pdftitle={Towards Looped Models Done Right: Part II - Rethinking at Fixed Points},pdfauthor={Benhao Huang, Chufan Shi, Junlin Chen, Shicheng Wen, Zhengzhong Liu, Eric Xing, Xuezhe Ma}}
\abstract{Every recurrence of a looped language model adds cost in training, decoding, prefill, and reinforcement learning (RL). {The closer recurrent states get to fixed points, the less the path to them matters. This enables} truncated backpropagation in training; terminal key-value (KV) sharing for decoding with almost no loss in accuracy; a distilled student that prefills up to {$1.79\times$} faster; and RL updates that compute gradients from saved rollout states, {$2\times$} faster than {backpropagating through the replayed trajectory}. We therefore improve the two components of training that shape these fixed points: the depth prior and input injection. Fixed-depth training breaks KV sharing, and Huginn's broad depth prior {supports sharing but dilutes supervision at the target depth more than sharing requires}; we learn the prior from prediction feedback, with an entropy term that keeps it broad. Existing injection schemes let the state's component along the input amplify or cancel the injection; we remove this component with orthogonal injection. From $100$M to $1.6$B parameters, the learned prior and orthogonal injection lower perplexity at every scale relative to Huginn's prior and existing injection schemes, respectively. At $1.6$B, the learned prior with a $3\times$ smaller KV cache matches the downstream average of fixed-depth training with the full cache.}

\author[\(\ddagger\ast\)]{\href{https://huskydoge.github.io/}{Benhao Huang}}\authorsep{, }
\author[\(\ddagger\dagger\)]{Chufan Shi}\authorsep{, }
\author[\(\ddagger\)]{Junlin Chen}\authorsep{, }
\author[\(\ddagger\dagger\)]{Shicheng Wen}\authorsep{\\}
\author[\(\ddagger\)]{Zhengzhong Liu}\authorsep{, }
\author[\(\ddagger\)]{Eric Xing}\authorsep{, }
\author[\(\ddagger\dagger\)]{Xuezhe Ma}\authorsep{}

\affiliation[\(\ddagger\)]{Institute of Foundation Models}\affiliationsep{, }
\affiliation[\(\dagger\)]{USC}\affiliationsep{, }
\affiliation[\(\ast\)]{CMU}

\begin{document}
\maketitle
\begin{center}
\includegraphics[width=\linewidth]{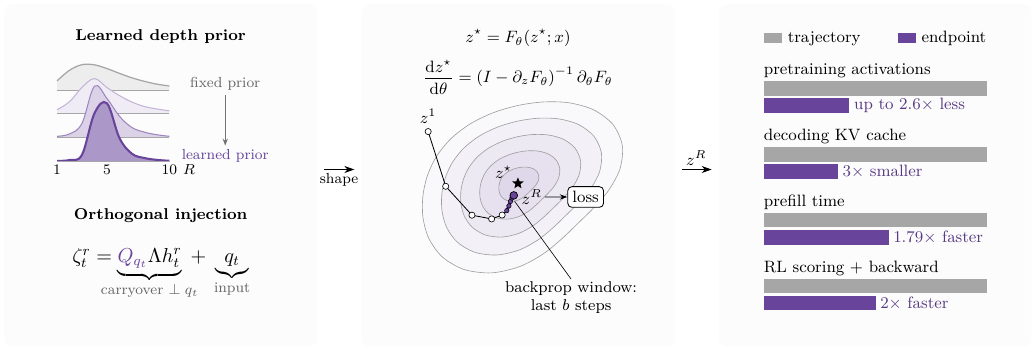}
\captionof{figure}{{\textbf{{What shapes fixed points, and what fixed points enable.}} Near a fixed point, the recurrence trajectory can be skipped: training backpropagates through only the last $b$ steps, decoding keeps terminal KV banks, a student distills prefill, and RL reuses saved states to score and update (\secref{sec:why}). {The depth prior and input injection shape the fixed points; we learn the prior from prediction feedback and project the carryover off the input (\secref{sec:promoting}). }}}
\label{fig:teaser}
\end{center}
\etocdepthtag.toc{main}
\begingroup
\setcounter{tocdepth}{2}
\etocsettagdepth{main}{section}
\etocsettagdepth{appendix}{none}
\endgroup
\clearpage
\section{Introduction}
\label{sec:introduction}

Looped language models repeatedly apply a shared block of layers~\citep{universal_transformer,huginn}, but whether they should learn fixed points remains contested. One camp forgoes fixed points and keeps the loop count small and fixed~\citep{gao2026loopie,wang2026smelt,ouro}. The other camp shapes fixed points~\citep{bai2021stabilizing,yang2026stars,eqr} and loops much deeper. {In typical implementations, every loop adds cost}: activations to store during training, keys and values (KV) to cache during decoding, compute to prefill each prompt, and a replay to score each rollout during reinforcement learning (RL). In this work, we show that fixed points hold the potential to cut each of these costs (Figure~\ref{fig:teaser}), which makes deep looped models cheaper at every stage: pretraining, serving, and post-training.

The fixed-point equation $z^\star = F_\theta(z^\star; x)$ involves only the endpoint, not the path to it, so once a model's states are near a fixed point, the endpoint can stand in for the full trajectory (\secref{sec:why}). Reusing terminal KV caches across loops tests this claim: it barely changes the zero-shot GSM8K accuracy of Huginn~\citep{huginn} but collapses that of Ouro~\citep{ouro} at four loops~\citep{schwethelm2026cdb}. The split tracks convergence: Huginn's states approach fixed points~\citep{blayney2026mechanistic}; Ouro's do not.

Similarly, the same fixed point gives each remaining cost a shortcut. In pretraining, backpropagating through only the last few loops approximates the equilibrium gradient~\citep{neumann_rbp,phantom_grad} without the warm-up that implicit gradients require. In decoding, \emph{terminal KV sharing} keeps only each prefix token's terminal KV (Figure~\ref{fig:prefix-paths}b) at little cost in performance. In prefilling, a student distilled to predict the fixed point directly replaces the loops (Figure~\ref{fig:prefix-paths}c). In RL, scoring rollouts from saved fixed-point states avoids replaying the loops. Each shortcut removes a depth-dependent cost, so its savings grow with every added loop. Looped models already take the first two~\citep{huginn,hrm,ouro}; fixed points explain why they work and yield the other two.

Realizing these benefits requires training that shapes the fixed points. Two components do the shaping, and current choices for each fall short (\secref{sec:promoting}). The first is the \emph{depth prior}, the distribution of training depths. Because tokens converge at different depths (Figure~\ref{fig:prefix-dynamics}), the sampled depth decides whether training sees the converged prefix context that terminal KV sharing supplies at inference (Figure~\ref{fig:prefix-paths}a). A model trained at a single depth predicts best there but breaks KV sharing; Huginn's fixed broad prior~\citep{huginn} supports sharing but dilutes supervision at the target depth more than sharing requires. We learn the prior from prediction feedback, which pulls supervision toward depths that predict well, and add an entropy term that keeps it broad enough for sharing. The second component is \emph{input injection}~\citep{bai2019deep,huginn,parcae,huang2026loopedmodels}, which adds the input to the state at every recurrence. Huginn and Parcae leave the state's component along the input unconstrained; depending on its sign, this component amplifies or cancels the injection. Our orthogonal input injection (OrthoInj) removes this component, so every recurrence receives the input at the same strength.

\begin{itemize}[leftmargin=*,topsep=2pt,itemsep=1pt]
\item Fixed points explain the two existing shortcuts (\secref{sec:why}): truncated backpropagation alone shapes the fixed points that terminal KV sharing relies on. Fixed-depth training with full backpropagation through time does not; at $1.6$B, sharing drops its GSM8K accuracy from $50.6$ to $21.2$. With tolerant fixed points, sharing cuts the KV cache $3\times$ at five recurrences, keeping the entropy-regularized learned-prior models within $0.4$ points of their full-cache downstream average.
\item Fixed points yield two new shortcuts. Distilled prefill (\secref{sec:exp-distill}), trained on a quarter of the teacher's pretraining tokens, is up to $1.79\times$ faster on $8$K prompts and trails the teacher's downstream average by $1.0$--$4.5$ points; at matched latency, it scores $0.5$--$0.9$ points above the teacher prefilling with two recurrences. Rollout-state reuse (\secref{sec:why-rl}) halves the scoring and backward time of each RL update relative to replaying the loops with full backpropagation through time, and stays within $1.6$ points of that baseline's GSM8K pass@1.
\item The learned depth prior keeps prediction quality under terminal KV sharing (\secref{sec:exp-prior}): from $100$M to $1.6$B, it lowers perplexity and raises the downstream average over Huginn's fixed prior, and at $1.6$B, with a $3\times$ smaller KV cache, it even matches the full-cache downstream average of fixed-depth training.
\item OrthoInj improves on the strongest injection baseline, Parcae's (\secref{sec:exp-architecture}). From $100$M to $1.6$B, it trains to a lower loss, lowers validation perplexity by $0.3$--$1.4\%$, and raises the downstream average.
\end{itemize}
\section{Fixed Points in Looped Language Models}
\label{sec:fixed-points}
\label{sec:background}

Recurrent states in looped language models {tend to approach fixed points as recurrence depth grows}~\citep{huginn,blayney2026mechanistic}.\footnote{{``Fixed points'' here include approximately stationary states at finite depths.}} {For an input $x=(x_1,\ldots,x_N)$, the state after $r$ recurrent updates is $z^r=(\mathbf{H}^r,C^r)$: hidden states $\mathbf{H}^r=(h_1^r,\ldots,h_N^r)$ and the \emph{KV banks} $C^r$ written on update $r$, one per physical attention layer. The shared update map gives $z^{r+1}=F_\theta(z^r;x)$, and $z^\star$ is a \emph{fixed point} if $F_\theta(z^\star;x)=z^\star$.}   {We use Huginn~\citep{huginn} for experiments. A prelude of non-shared blocks embeds the input as $e=P_\theta(x)$. A recurrent core of shared blocks starts from a random state $z^0$ and iterates $z^r = F_\theta(z^{r-1};x)$ for $r=1,\dots,R$, injecting $e$ at every recurrence (\secref{sec:architecture}); we write $z_\theta^R(x)$ for the final state. A coda of non-shared blocks maps $z_\theta^R(x)$ to next-token logits.}

\begin{figure}[!htbp]
  \centering
  \includegraphics[width=\linewidth]{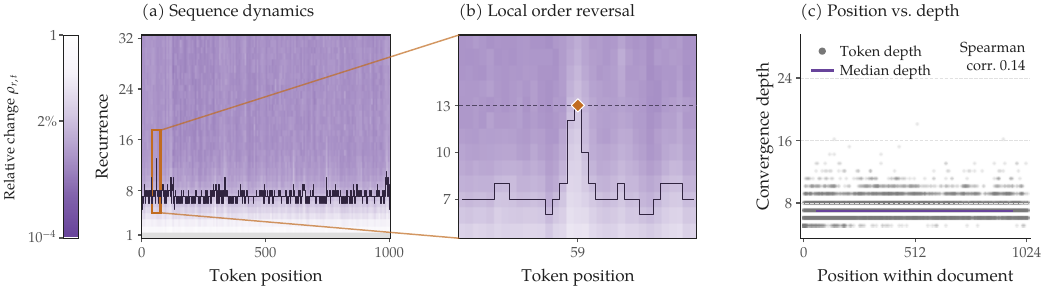}

  \caption[Heterogeneous token convergence within a sequence.]{{\textbf{Heterogeneous token convergence in {Huginn}.}}
(a) Relative hidden-state changes $\rho_{r,t}$, with {convergence} threshold {$\tau=2\%$}.
(b) {Later tokens can approach approximate fixed points before earlier ones.}
(c) Position versus {convergence depth}; the purple line shows medians over $128$-position intervals.
Details in \appref{app:huginn-prefix-probe}.}
  \label{fig:huginn-prefix-stabilization}\label{fig:prefix-dynamics}

\end{figure}

{We measure convergence by the tokenwise relative change $\rho_{r,t}=\norm{h_t^r-h_t^{r-1}}_2/(\norm{h_t^{r-1}}_2+\epsilon)$, with $\epsilon=10^{-8}$. Token $t$ is \emph{{converged} from depth $r$} if $\rho_{r',t}<\tau$ for all $r\le r'\le R_{\rm obs}$, the observation horizon; its \emph{{convergence depth}} is the smallest such $r$.}

{{Convergence} in Huginn is \emph{heterogeneous}: tokens in {a} sequence {converge} at different depths, not necessarily in causal order (Figure~\ref{fig:prefix-dynamics}b), and {convergence depth} correlates only weakly with position (Spearman correlation $0.14$; Figure~\ref{fig:prefix-dynamics}c). {A} prefix $x_{<t}$ is a \emph{{converged prefix}} from depth $r$ if every token in it is {converged} from $r$; \secref{sec:inference} relates such prefixes to KV sharing.}

\FloatBarrier
\section{What Fixed-Point Shaping Enables}
\label{sec:why}
\emph{Fixed-point shaping} trains a looped model with supervision near the fixed points of its recurrence~\citep{bai2019deep,eqr}. At a fixed point, the endpoint can stand in for the trajectory. Conventional looped language models pay for every recurrence in both training and inference; fixed-point shaping decouples both costs from depth.

\begin{figure}[H]
\centering
\includegraphics[width=\linewidth]{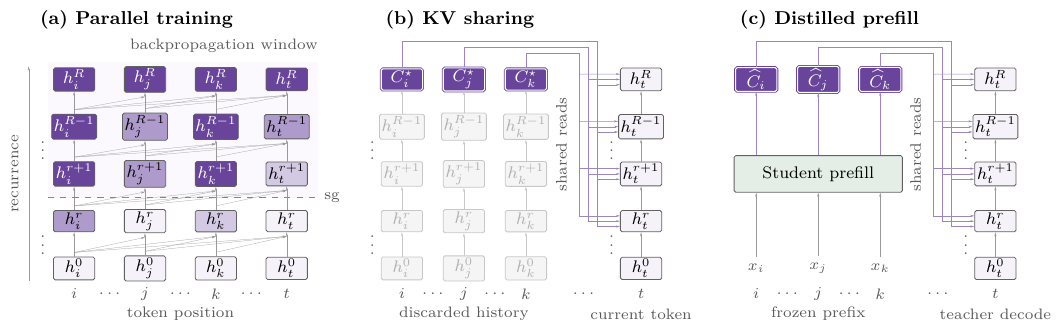}
\caption{\textbf{The endpoint stands in for the recurrent trajectory.} Recurrence runs upward; ellipses mark omitted positions and depths. (a) {In training, each} update reads prefix states from the previous depth{; darker states are closer to convergence (as in Figure~\ref{fig:prefix-dynamics}). Prefix tokens} converge at different depths (\secref{sec:fixed-points}), so the sampled depth $R$ decides whether reads inside the backpropagation window see converged context; gradients stop at sg. (b) Decoding reads only the terminal prefix KV banks $C^\star_{<t}$ and discards the others, with nearly unchanged downstream accuracy. (c) During prefill, a distilled student replaces the recurrence and supplies approximate terminal banks $\widehat C_{<t}$.}
\label{fig:prefix-paths}
\end{figure}

\subsection{Training: {Endpoint Gradients}}
\label{sec:why-training}
{Full backpropagation through time (BPTT) differentiates through all \(R\) recurrences, so its cost grows with depth. When near a fixed point, the gradient can instead be computed from the endpoint.}

\paragraph{Pretraining: truncated backpropagation.}

At a fixed point \(z^\star\), let
\(J_\star=\partial_z F_\theta(z^\star;x)\) and
\(B_\star=\partial_\theta F_\theta(z^\star;x)\).
If \(F_\theta\) is continuously differentiable and the spectral radius of \(J_\star\) is below one,
implicit differentiation gives the \emph{equilibrium sensitivity}~\citep{bai2019deep} \(D_\star\equiv\mathrm dz^\star/\mathrm d\theta\),
and truncating its Neumann series gives
\(D_b\)~\citep{neumann_rbp,phantom_grad} (derivation in \appref{app:gradients}):
\begin{equation}
 D_\star=(I-J_\star)^{-1}B_\star=\sum_{k=0}^{\infty}J_\star^kB_\star,  \quad
 D_b=\sum_{k=0}^{b-1}J_\star^kB_\star,
 \quad
 {\tilde D_b=\sum_{k=0}^{b-1}\Bigl(\prod_{j=1}^{k}J_{R-j}\Bigr)B_{R-1-k}.}
 \label{eq:fixed-point-sensitivity}
\end{equation}

{\(D_\star\) depends only on \(z^\star\), and it equals the \(R\to\infty\) limit of full BPTT when that trajectory converges. Its truncated Neumann estimate \(D_b\) has a bias \(D_\star-D_b=J_\star^bD_\star\) that decays geometrically in \(b\). Both, however, are valid only once a stable fixed point exists. Applied from initialization, such endpoint gradients destabilize training (\appref{app:tbptt-study}), so implicit models first warm up with unrolled backpropagation~\citep{bai2019deep,bai2020multiscale,nguyen2023efficient}.}

Truncated BPTT (TBPTT) instead differentiates through only the last \(b\) recurrences~\citep{williams1990online,aicher2020adaptive}, so its memory and compute scale with \(b\) rather than \(R\). Its gradient \(\tilde D_b\) evaluates {the Jacobians \(J_r=\partial_zF_\theta(z^r;x)\) and \(B_r=\partial_\theta F_\theta(z^r;x)\) at visited states \(z^r\)} rather than at \(z^\star\). This frees TBPTT from needing a fixed point: from any state, \(\tilde D_b\) is the exact gradient of the loss through the last \(b\) recurrences with the earlier state held fixed, so every update descends a well-defined objective. As the visited states converge to \(z^\star\), \(\tilde D_b\) approaches \(D_b\) regardless of \(R\). TBPTT is therefore as stable as full BPTT early in training and recovers the fixed-point gradient once fixed points form, using a single estimator with no warm-up phase or switch. 

\textbf{Furthermore, truncation alone shapes fixed points.} {A fixed point has two testable properties: (i)~extra recurrences leave the state unchanged, and (ii)~terminal KV sharing costs little (Lemma~\ref{lem:joint-frozen}). At fixed depth $R=5$, full BPTT fails both: extrapolating to $R=64$ raises perplexity by over $25\%$, and terminal KV sharing raises it by $21\%$. In contrast, TBPTT through the last two recurrences ($b=2$) raises perplexity by at most $2\%$ and $3\%$, respectively} (\appref{app:tbptt-study}). {With sampled depths enabled, training with $b\geq4$ also matches full BPTT within $1\%$ perplexity at $400$M (Table~\ref{tab:tbptt-pln5-windows} in \appref{app:tbptt-study}).}

\paragraph{Post-training: reusing rollout states for RL.}
\label{sec:why-rl}
RL usually runs the recurrent forward pass twice: once without gradients to sample rollouts, and again with gradients to score them. 
\begin{algorithm}[tbp]
\caption[Policy update from saved rollout states]{Policy update from saved rollout states (Neumann-$(b-1)$). A single forward pass through the recurrent block records its graph $\mathcal G$ (line~\ref{alg:rl-reuse-record}); the Neumann loop (lines~\ref{alg:rl-reuse-loop-start}--\ref{alg:rl-reuse-loop-end}) and the parameter gradient (line~\ref{alg:rl-reuse-param}) reuse $\mathcal G$ without running the block again. In code, lines~\ref{alg:rl-reuse-loop-start}--\ref{alg:rl-reuse-loop-end} run as a backward hook on $\widetilde z$.}
\label{alg:rl-reuse}
\small
\begin{algorithmic}[1]
\Require saved endpoints $z^R$ of the sampled sequences $x$ (no gradient), advantages, window $b$
\State $\bar z\gets\sg(z^R)$, marked as requiring gradients
\State $f\gets F_\theta(\bar z;x)$, recording its autograd graph $\mathcal G$\label{alg:rl-reuse-record}\Comment{the only pass through the recurrent block}
\State $\widetilde z\gets\bar z+\bigl(f-\sg(f)\bigr)$\Comment{forward value $z^R$}
\State $\mathcal L\gets$ policy-gradient loss of the coda readout of $\widetilde z$ on the sampled tokens
\State $g\gets\partial\mathcal L/\partial\widetilde z$\Comment{backward through the readout}
\State $v\gets g$\label{alg:rl-reuse-loop-start}
\For{$k=1,\ldots,b-1$}
    \State $v\gets g+J_R^{\top}v$\Comment{backward through $\mathcal G$ to $\bar z$: \texttt{autograd.grad}$(f,\bar z,v)$, $\mathcal G$ reused}
\EndFor\label{alg:rl-reuse-loop-end}
\State backpropagate $v$ through $\mathcal G$ to $\theta$\label{alg:rl-reuse-param}\Comment{$\mathcal G$ reused once more: $\theta$ receives $B_R^{\top}v=(\partial_\theta F_\theta)^{\top}v$}
\State optimizer step on $\theta$ with these and the readout gradients; discard $z^R$
\end{algorithmic}
\end{algorithm}

Near a fixed point, the second pass need not replay the trajectory. Full BPTT computes $\tilde D_R$ of equation~\eqref{eq:fixed-point-sensitivity}; when every visited state is close to $z^\star$, its Jacobians are close to $J_\star$ and $B_\star$, so $\tilde D_R$ approaches $D_\star$, whose truncated Neumann series can be evaluated at the last state $z^R$:
\begin{equation}
 \underbrace{\tilde D_R=\sum_{k=0}^{R-1}\Bigl(\prod_{j=1}^{k}J_{R-j}\Bigr)B_{R-1-k}}_{\text{full BPTT: all $R$ recurrences}}
 \;\approx\;
 \underbrace{\sum_{k=0}^{b-1}J_R^{k}B_R}_{\text{reuse: endpoint $z^R$ only}}
 \label{eq:rl-reuse}
\end{equation}
The update is on-policy, so the endpoints $z^R$ saved during rollout equal those a fresh forward pass would produce. Reuse applies the recurrent block once at $z^R$ and evaluates the right side of equation~\eqref{eq:rl-reuse} by looping the backward pass $b-1$ times over that one recorded recurrence, $v\leftarrow g+J_R^{\top}v$ (Neumann-$(b-1)$, Algorithm~\ref{alg:rl-reuse}). Full BPTT instead runs and backpropagates through all $R$ recurrences, and recomputation reruns them without gradients to rebuild $z^R$. The update cost and memory of reuse are thus independent of $R$, though rollout generation still takes most of the training time (\secref{sec:exp-rl}).

\FloatBarrier

\subsection{Inference: {Endpoint Caching and Distillation}}
\label{sec:why-inference}
\paragraph{Decoding: cache only terminal KV banks.}
\label{sec:inference}
The KV cache of conventional looped models grows linearly with recurrence depth. \emph{Terminal KV sharing} removes this growth: each prefix token keeps only its terminal KV, and every recurrence of later tokens attends to it (Figure~\ref{fig:prefix-paths}b).\footnote{Decoding reads only terminal banks, so prefill and decoding can run at separate recurrence budgets (\appref{app:pd-results}); prefill gains saturate at moderate depth while decoding gains persist deeper.} This requires a model that predicts well from terminal context.

However, {teacher-forced} training alone does not match this setting exactly. It updates all tokens of a sequence in parallel, so recurrence $r$ of the current token reads prefix banks from depth $r$ (Figure~\ref{fig:prefix-paths}a); sharing instead supplies terminal banks from the first recurrence. This gap is small when the prefix banks change little across depth and the current token's update contracts, so an error introduced by early context decays over later recurrences. Sampled training depths push the model toward this regime: deep draws let supervised recurrences read nearly settled prefix banks, and supervising many depths penalizes states that drift as recurrence continues (Figure~\ref{fig:fixed-vs-pln-depth-scaling}).

{Lemma~\ref{lem:joint-frozen} formalizes the limiting case, in which the prefix banks converge.}
For the current token $t$, let $h^r$ denote its state at recurrence $r$, $C^r_{<t}$ the prefix-dependent input to its update (including the prefix's per-layer KV banks), and $f_\theta(h;C)$ the hidden-state component of $F_\theta$, as a function of state $h$ and context $C$.
By causality, $\{C^r_{<t}\}$ does not depend on the current token's states, so the prefix enters the current token's iteration as an external input sequence.

\begin{lemma}[Joint iteration and frozen-prefix decoding, proof in \appref{app:joint-frozen}]
\label{lem:joint-frozen}
Fix the parameters and the input sequence, and suppose $C^r_{<t} \to C^\star_{<t}$.
Let $\mathcal H \subseteq \mathbb R^d$ be nonempty and closed, and let $\mathcal C = \{C^r_{<t}\}_{r \ge 0} \cup \{C^\star_{<t}\}$.
Suppose that for every $C \in \mathcal C$, $f_\theta(\cdot;C)$ maps $\mathcal H$ into itself and is $\kappa$-contractive with $\kappa<1$, and that for every $h \in \mathcal H$, $f_\theta(h;\cdot)$ is $\beta$-Lipschitz on $\mathcal C$.
Then, from any initial states in $\mathcal H$, the joint and frozen-prefix iterations
\begin{equation}
\label{eq:joint-frozen-iterations}
h^{r+1}=f_\theta(h^r;C^r_{<t}),
\qquad
\widetilde h^{r+1}=f_\theta(\widetilde h^r;C^\star_{<t})
\end{equation}
both converge to the unique fixed point $h^\star = f_\theta(h^\star;C^\star_{<t})$ in $\mathcal H$.
Consequently, any readout continuous in state and context yields the same limiting next-token distribution under both iterations.
\end{lemma}

In the limit, therefore, sequence-parallel supervision near fixed points transfers to {decoding} with terminal KV sharing.
{Empirically, at the training depth $R=5$, sharing {keeps $4$ of $12$ banks and} raises perplexity over the full cache by at most {$0.85\%$} (\appref{app:deep-r-kv}).}

\paragraph{Prefill: distill the recurrent endpoint.}
\label{sec:why-distill}
{Because decoding reads only terminal KV banks, prefill can be distilled into a student that predicts only the recurrent endpoint; without terminal sharing, the student would have to reproduce the KV banks of every recurrence}~\citep{horton2024kvprediction}. {The student is a lightweight network $S_\psi$ that replaces the iterated recurrent core during prefill; the prelude and coda are reused from the teacher.} From the prelude output $e=P_\theta(x)$, the student predicts $\widehat z=S_\psi(e)\approx z_\theta^R(x)$; prefill then applies one teacher recurrence and the coda to $\widehat z$ and retains each physical attention layer's KV bank (Figure~\ref{fig:prefix-paths}c). Because fixed-point shaping makes the endpoint an approximate fixed point, one recurrence from a nearby $\widehat z$ yields banks close to the teacher's terminal banks. Decoding is unchanged: the teacher runs at its original depth with terminal KV sharing.

We train the student against depth-$R$ targets computed online from the frozen teacher, using normalized hidden-state regression~\citep{romero2015fitnets} plus a next-token KL term:
\begin{equation}
\mathcal L_{\rm distill}(\psi)
=\E_x\!\left[
\frac{
\norm{\widehat z-z_\theta^R(x)}_F^2
}{
N d\,\sigma_{\rm T}^2
}
+\frac1N\sum_{t=1}^{N}\mathrm{KL}\bigl(p_t\,\Vert\,\widehat p_t\bigr)
\right],
\label{eq:hidden-distillation}
\end{equation}
where $N$ is the prompt length, $d$ is the hidden width, and $\sigma_{\rm T}^2$ is a fixed teacher-state energy calibrated on training data, so the first term is a per-coordinate squared error on the teacher's scale. The second term compares the teacher's next-token distribution $p_t$ at depth $R$ with the distribution $\widehat p_t$ obtained by applying the frozen coda and output head to $\widehat z$. We initialize the student from the teacher's recurrent core, so before training it approximately outputs the teacher's state after one recurrence; equation~\eqref{eq:hidden-distillation} then moves this output toward the depth-$R$ endpoint (\appref{app:distill-ablation}).

\secref{sec:exp-distill} evaluates distilled prefill. With a suitable loss, the student may also be able to serve as a standalone language model or a drafter for speculative decoding; we leave these uses, and possible connections between recurrent state refinement and flow-based language modeling, to future work.

\section{Improving Fixed-Point Shaping}
\label{sec:promoting}
The benefits in \secref{sec:why} depend on how training shapes fixed points. {Two} common techniques do this shaping: depth sampling sets where supervision lands, and input injection sets what a fixed point retains about the input. We improve both {(\secref{sec:fixed-point-depth-prior} and~\secref{sec:architecture})}.

\subsection{Learning a Depth Prior for Fixed-Point Training}
\label{sec:fixed-point-depth-prior}

\begin{algorithm}[tbp]
\caption[Depth-prior learning]{Depth-prior learning. { The prior learns from the detached model CE (line~\ref{alg:learned-prior-reward}) and adds no forward pass. The advantage (line~\ref{alg:learned-prior-adv}) centers the reward with EMA statistics (lines~\ref{alg:learned-prior-ema-start}--\ref{alg:learned-prior-ema-end}) to reduce variance and rescales it to CE scale, so $\lambda_H$ and $\lambda_m$ need no retuning across reward designs (\appref{app:prior-results}).}}
\label{alg:learned-prior}
\small
\begin{algorithmic}[1]
\Require $\theta_{n},\phi_{n}$, EMA statistics at $n-1$;
$\nu,\lambda_H,\lambda_m,\bar R,b_{\max}$

\State $p_{\phi_{n}}\gets\softmax(\phi_{n})$
\For{microbatch $i$}
    \State $r_i\sim p_{\phi_{n}}$;\quad
    $\ell_i\gets\CE(\theta_{n},r_i)$;\quad
    $u_i\gets \exp(-\ell_i)$\label{alg:learned-prior-reward}
    \Comment{$\ell_i$ detached}
    \State accumulate CE gradients through the final
    $\min(r_i,b_{\max})$ recurrences
    \State $A_i\gets s_{n-1}(u_i-\hat u_{n-1})$\label{alg:learned-prior-adv}
\EndFor
\State $\mathcal J_{n}\gets
-\widehat{\E}_{n}[\sg(A_i)\log p_{\phi_{n}}(r_i)]
-\lambda_H H(p_{\phi_{n}})
+\lambda_m(\E_{p_{\phi_{n}}}[R]-\bar R)^2$
\State $\theta_{n+1}\gets\ModelStep(\theta_{n})$;\quad
$\phi_{n+1}\gets\Adam(\phi_{n},\nabla_{\phi_{n}}\mathcal J_{n})$
\State $\EMA_{n}[\chi]\gets
\nu\EMA_{n-1}[\chi]+(1-\nu)\widehat{\E}_{n}[\chi],
\quad \chi\in\{\ell,u,\ell^2,u^2\}$\label{alg:learned-prior-ema-start}
\State $\sigma_{\chi,n}\gets
\max\!\left\{\sqrt{\EMA_{n}[\chi^2]-\EMA_{n}[\chi]^2},\epsilon\right\},
\quad \chi\in\{\ell,u\}$
\State $\hat u_n\gets\EMA_{n}[u]$;\qquad
$s_{n}\gets\sigma_{\ell,n}/\sigma_{u,n}$\label{alg:learned-prior-ema-end}
\end{algorithmic}
\end{algorithm}

The sampled training depth sets how converged the states are when supervision arrives, and thus how well the model tolerates terminal KV sharing (\secref{sec:inference}). Huginn~\citep{huginn} samples the training depth from a shifted Poisson--log-normal (PLN) distribution:
\begin{equation}
\xi\sim\mathcal N(\mu,\sigma_\xi^2),\qquad
K\mid\xi\sim\operatorname{Poisson}\bigl(\exp(\xi)\bigr),\qquad
R=K+1,
\label{eq:huginn-pln}
\end{equation}
where $\sigma_\xi=\tfrac12$ and $\mu=\log(\bar R-1)-\sigma_\xi^2/2$, so that $\mathbb E[R]$ equals the target mean $\bar R>1$.
Under a fixed training budget, \textbf{the choice of depth prior trades prediction quality at the target depth against fixed-point shaping} (Figure~\ref{fig:fixed-vs-pln-depth-scaling}). Fixed-depth training concentrates supervision at one depth and predicts better there. PLN sampling spreads supervision across depths and yields the fixed-point behavior that supports KV sharing (Table~\ref{tab:prior-ablation}) and test-time scaling.\footnote{\appref{app:prior-results} gives the depth distributions and comparison details (Figure~\ref{fig:fixed-pln-learned-distributions}).}

\begin{figure}[!htbp]
\centering
\includegraphics[width=\linewidth]{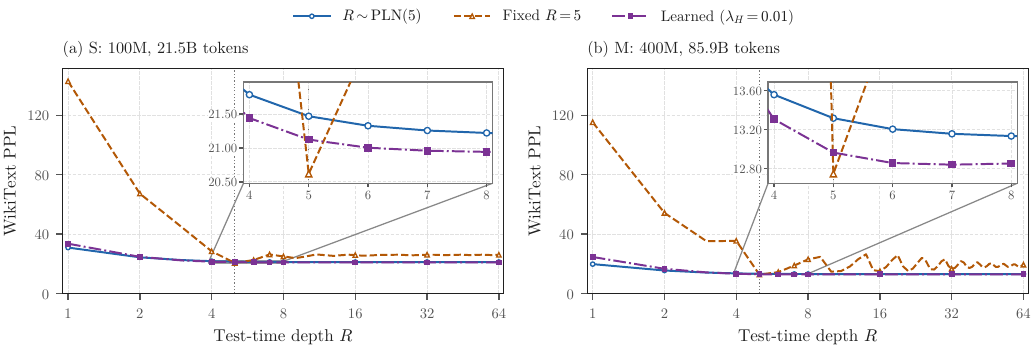}
\vspace{-3mm}
\caption{\textbf{The learned prior narrows the trade-off between fixed-depth training and PLN sampling.}
{Native WikiText-103 test PPL across test-time depths at matched training budgets.
Fixed-depth training is best at training-matched depth $R=5$ but degrades sharply at others. PLN sampling performs well across all depths but falls behind at $R=5$. The learned prior lies between them, narrowing the gap at $R=5$ while staying stable across depths.}}
\label{fig:fixed-vs-pln-depth-scaling}
\end{figure}

We therefore \emph{learn the depth prior}{: it spreads supervision} across depths but adapts its allocation {to} training feedback, balancing prediction quality against fixed-point shaping at equal mean depth.

\paragraph{Learned prior.}
We replace the PLN prior with a categorical prior $p_\phi=\operatorname{softmax}(\phi)$ over depths $\mathcal R=\{1,\ldots,R_{\max}\}$, with learnable logits $\phi$. At each optimizer step $n$, every microbatch draws a depth $r\sim p_{\phi_{n}}$ and receives supervision at that depth. We learn $p_\phi$ from this feedback by minimizing
\begin{equation}
\mathcal J_{\mathrm{prior}}(\phi)
=\underbrace{-\widehat{\mathbb E}_{n}\!\bigl[\sg(A_{n})\log p_\phi(r)\bigr]}_{\text{prediction quality}}
-\underbrace{\lambda_H H(p_\phi)\vphantom{\widehat{\mathbb E}_{n}}}_{\text{depth diversity}}
+\underbrace{\lambda_m\bigl(\mathbb E_{p_\phi}[R]-\bar R\bigr)^2}_{\text{budget control}},
\label{eq:depth-prior-objective}
\end{equation}
where $\widehat{\mathbb E}_{n}$ averages over microbatches, weighted by their unmasked positions. The advantage $A_{n}$ scores the sampled depth. We map the mean next-token CE $\ell$ to a reward $u=\omega(\sg(\ell))$ with a decreasing function $\omega$; we use inverse perplexity, $\omega(\ell)=\exp(-\ell)$. The advantage is then $A_{n}=s_{n-1}(u-\hat u_{n-1})$: an exponential moving average (EMA) baseline $\hat u_{n-1}$ centers the reward, and $s_{n-1}=\sigma_{\ell,n-1}/\sigma_{u,n-1}$ rescales it to the scale of CE, following PopArt~\citep{vanhasselt2016learning,hessel2019popart}. A positive $A_{n}$ thus raises the probability of the sampled depth.

\paragraph{Entropy regularization.}
Prediction feedback under a budget constraint favors depths that already predict well. Left alone, it concentrates supervision on a narrow range of depths and can weaken fixed-point shaping. The entropy term
\(
H(p_\phi)=-\sum_{r\in\mathcal R} p_\phi(r)\log p_\phi(r)
\)
counters this tendency. Increasing \(\lambda_H\) spreads the depth allocation. Since depths start at $1$, the wider allocation extends mainly above $\bar R$, toward $R_{\max}$. This trades finite-budget predictive performance for supervision at later depths, where settled prefix context supports KV sharing (\appref{app:m-entropy-matched}). The coefficient $\lambda_H$ thus gives a direct knob for balancing performance and memory efficiency.

\paragraph{Budget control.}
The third term penalizes deviations of the expected depth from the target $\bar R$, keeping the allocation near the prescribed forward recurrence budget. \appref{app:prior-algorithm} gives one update step (Algorithm~\ref{alg:learned-prior}) and implementation details.

\subsection{Orthogonal Input Injection}
\label{sec:architecture}
Input injection feeds the same embedding $q_t$ into every recurrence~\citep{bai2019deep,huginn,parcae,huang2026loopedmodels}, keeping the fixed points input conditioned. {A common additive form is}
\begin{equation}
{\zeta_t^r=Mh_t^r+q_t,\qquad \mathbf H^{r+1}=\Gamma_\theta(\zeta_1^r,\ldots,\zeta_N^r),}
\label{eq:additive-input-injection}
\end{equation}
{where $h_t^r$ is the recurrent state at token $t$ and recurrence $r$, $M$ controls state carryover, and $\Gamma_\theta$ is the recurrent transformer core. Parcae}~\citep{parcae} {uses a diagonal carryover $M=\Lambda=\operatorname{diag}(\exp(-\Delta\odot \alpha))$ with learned positive vectors $\Delta$ and $\alpha$, so $\Lambda$ has entries in $(0,1)$ and $\norm{\Lambda}_2<1$. It injects the prelude output $e_t=P_\theta(x)_t$ as $q_t=\Delta\odot We_t$.}\footnote{Parcae applies RMSNorm to $e_t$ before injection; OrthoInj drops it. Ablations in \appref{app:architecture-results}.}

\paragraph{Orthogonal injection.}
{{In Parcae, the carryover $\Lambda h_t^r$ can have a component along $q_t$ that amplifies or cancels the input, so the input a fixed point retains depends on the state; this can hurt performance (\appref{app:architecture-results}).} {OrthoInj} removes this component:}
\begin{equation}
{\zeta_t^r=Q_{q_t}\Lambda h_t^r+q_t,\qquad Q_{q_t}=I-\frac{q_tq_t^\top}{\norm{q_t}_2^2+\epsilon}.}
\label{eq:orthoinj}
\end{equation}
{For $q_t\ne0$ and $\epsilon=0$, the component of $\zeta_t^r$ along $q_t$ is exactly $q_t$ at every recurrence, including at a fixed point, while the orthogonal part evolves freely; the stabilizer $\epsilon>0$ makes this approximate. Since $\norm{Q_{q_t}}_2\le1$, the projected carryover stays contractive: $\norm{Q_{q_t}\Lambda}_2\le\norm{\Lambda}_2<1$.}

\FloatBarrier
\section{Experiments}
\label{sec:experiments}
We first evaluate the two training changes from \secref{sec:promoting}: the learned depth prior and OrthoInj. We then test {rollout-state reuse in RL (\secref{sec:why-rl}) and } distilled prefill (\secref{sec:why-distill}).
\subsection{Setup}
\vspace{-5mm}
\label{sec:setup}
\begin{papertablepair}{tables/main/training_scales.tex}
\textbf{Models and training.} We build looped models from dense Llama Transformer blocks~\citep{touvron2023llama} at Small (S), Medium (M), and Large (L) scales (Table~\ref{tab:training-scales}). Following Huginn~\citep{huginn}, each model has one prelude block, a two-block recurrent core, and one coda block{; we write this layout as $1{+}2{\times}R{+}1$, since $R$ recurrences give a logical depth of $2+2R$ blocks}. {Untied~4 stacks four distinct blocks, matching the physical depth and parameter count.} {{Untied~12} stacks twelve distinct blocks, matching the logical depth of five recurrences, with three times the parameters.}
\end{papertablepair}
\vspace{-2mm}
We train on a weighted multilingual mixture that includes TxT360~\citep{txt360data2024}, tokenized with the Jais64k tokenizer~\citep{sengupta2023jais}. Optimization uses AdamW~\citep{loshchilov2019decoupled} with a warmup--stable--decay schedule~\citep{hu2024minicpm}: $5\%$ linear warmup, then $10\%$ cosine decay to one-tenth of the peak learning rate. We sweep the peak learning rate at each scale and tune the main configurations separately (\appref{app:setup-details}).
\paragraph{Evaluation.}
We report perplexity (PPL) on \textsc{MixVal}, a held-out set drawn from the TxT360 sources, and on WikiText-103~\citep{merity2016pointer}. Downstream accuracy covers LAMBADA~\citep{lambada}, HellaSwag (HS)~\citep{hellaswag}, PIQA~\citep{piqa}, ARC-Easy/Challenge~\citep{arc}, OpenBookQA (OBQA)~\citep{openbookqa}, and SciQ~\citep{sciq}. Generation uses GSM8K (GSM)~\citep{gsm8k}, DROP~\citep{drop}, and MBPP+~\citep{mbpp,evalplus}. Unless noted, we evaluate at $R=5$ with terminal KV sharing (\secref{sec:inference}): prefill caches only the terminal KV of each prompt {token}, and choice tasks and generation then decode one token at a time. For PPL, we prefill the first half of each $8{,}192$-token sequence and score the second half in one parallel pass, whose tokens read the prefix through terminal banks and one another as in training. \appref{app:protocols} gives the full protocols.

\subsection{Learning the Depth Prior}
\label{sec:exp-prior}
A good depth prior should give low loss at the target depth and lose little when every iteration shares the terminal iteration's KV cache. We compare four priors with mean depth $5$: fixed-depth training ({Fixed $R{=}5$}), Huginn's {fixed} Poisson log-normal prior ({Fixed} PLN-5), and our learned prior with and without the entropy term. Both learned priors start from PLN-5 and cap depth at $R_{\max}=64$.
\begin{table}[!ht]
\centering\resulttablestyle
\caption{\textbf{{The learned prior far surpasses Untied~4 at the same parameter count and KV cache, and approaches Untied~12 with $3\times$ fewer of both.}} The looped models ($1{+}2{\times}R{+}1$) are evaluated at $R=5$, a logical depth of $12$ blocks, with terminal KV sharing, which caches $4$ KV banks, one per physical block, instead of $12$. Learned priors start from the Fixed PLN-5 prior. {FLOPs are in units of $10^{19}$.} Generation: GSM8K accuracy, DROP F1, and MBPP+ pass@1. Untied~4 and Untied~12 stack four and twelve distinct blocks and decode with their own full KV caches; Untied~4 matches the looped models' parameters and cache size, and Untied~12 their logical depth{, with $3\times$ their non-embedding parameters and KV cache}. {Fixed-depth training breaks KV sharing: it loses over half its GSM8K accuracy at M and L relative to the full cache (Appendix Table~\ref{tab:prior-ablation-fullkv}).} See \appref{app:setup-terminal} for evaluation details.}
\label{tab:prior-ablation}
\fontsize{9}{10.8}\selectfont\renewcommand{\arraystretch}{\resultrowstretch}
\begin{NiceTabular}{@{\hspace{6pt}}lcccccccccccccc@{\hspace{6pt}}}
\CodeBefore
\rowcolor{papergray}{3,10,17}
\Body
\toprule
\multicolumn{2}{c}{Training} & \multicolumn{2}{c}{PPL $\downarrow$} & \multicolumn{8}{c}{Downstream (\%) $\uparrow$} & \multicolumn{3}{c}{Generation (\%) $\uparrow$} \\
\cmidrule(lr){1-2}\cmidrule(lr){3-4}\cmidrule(lr){5-12}\cmidrule(lr){13-15}
Prior & FLOPs & Val. & Wiki & LAMB. & HS & PIQA & ARC-E & ARC-C & OBQA & SciQ & AVG & GSM & DROP & MBPP+ \\
\midrule
\multicolumn{15}{l}{\textbf{S scale: (100M, 21.5B tokens)}} \\
Fixed $R{=}5$ & {7.384} & 5.95 & 20.59 & 18.24 & 31.03 & 62.24 & 46.13 & \textbf{25.09} & 15.20 & 78.20 & 39.45 & 1.36 & 8.33 & 0.00 \\
Fixed PLN-5 & {7.292} & 5.73 & 20.19 & 26.41 & 32.42 & 61.70 & \textbf{47.56} & 23.98 & 16.20 & 81.00 & 41.32 & 1.52 & 13.29 & 0.75 \\
Learned, $\lambda_H{=}0$ & {7.336} & \textbf{5.67} & \textbf{19.85} & 26.97 & 32.23 & 62.57 & 45.66 & 24.83 & 16.80 & \textbf{81.70} & 41.54 & 1.44 & 12.54 & \textbf{0.88} \\
Learned, $\lambda_H{=}0.01$ & {7.328} & \textbf{5.67} & 19.91 & \textbf{27.36} & \textbf{32.51} & \textbf{63.00} & 45.50 & \textbf{25.09} & \textbf{17.20} & 80.90 & \textbf{41.65} & \textbf{1.59} & \textbf{14.08} & 0.62 \\
\cmidrule(lr){1-15}
Untied 4 & {3.259} & 6.57 & 25.19 & 23.95 & 30.26 & 59.96 & 44.57 & 23.21 & 16.60 & 77.60 & 39.45 & 1.97 & 11.39 & 0.12 \\
Untied 12 & {7.232} & 4.99 & 17.24 & 31.63 & 34.85 & 65.13 & 51.26 & 26.11 & 15.80 & 82.20 & 43.85 & 1.90 & 12.77 & 2.88 \\
\midrule
\multicolumn{15}{l}{\textbf{M scale: (400M, 85.9B tokens)}} \\
Fixed $R{=}5$ & {92.39} & 4.21 & 12.79 & 34.99 & 41.83 & 67.03 & 56.65 & 29.61 & \textbf{19.40} & \textbf{88.70} & 48.32 & 3.34 & 14.98 & 3.25 \\
Fixed PLN-5 & {91.00} & 4.00 & 12.55 & \textbf{39.69} & 44.64 & 67.68 & 56.61 & \textbf{29.78} & 16.40 & 88.60 & 49.06 & 13.72 & 17.75 & 10.38 \\
Learned, $\lambda_H{=}0$ & {92.01} & \textbf{3.92} & \textbf{12.24} & 39.22 & 45.63 & \textbf{68.88} & \textbf{58.75} & 28.07 & \textbf{19.40} & 88.10 & \textbf{49.72} & 14.33 & \textbf{20.29} & \textbf{13.12} \\
Learned, $\lambda_H{=}0.01$ & {92.14} & 3.93 & 12.28 & 38.58 & \textbf{46.17} & 67.41 & 57.58 & 28.75 & 17.00 & 88.00 & 49.07 & \textbf{15.62} & 17.11 & 10.38 \\
\cmidrule(lr){1-15}
Untied 4 & {36.77} & 4.68 & 15.73 & 31.19 & 37.43 & 63.71 & 50.72 & 27.65 & 15.80 & 82.10 & 44.08 & 2.35 & 13.12 & 2.50 \\
Untied 12 & {89.95} & 3.60 & 11.03 & 42.19 & 49.27 & 70.18 & 60.19 & 30.80 & 19.00 & 90.50 & 51.73 & 16.15 & 21.11 & 16.38 \\
\midrule
\multicolumn{15}{l}{\textbf{L scale: (1.6B, 343.6B tokens)}} \\
Fixed $R{=}5$ & {1272} & 3.24 & 8.92 & 48.85 & 58.19 & 72.31 & 66.29 & 36.01 & \textbf{24.40} & 92.30 & 56.90 & 21.23 & 27.98 & 12.75 \\
Fixed PLN-5 & {1253} & 3.10 & 8.76 & \textbf{55.17} & 60.80 & 73.56 & 69.74 & 38.74 & 23.80 & 93.30 & 59.30 & 47.61 & 34.11 & 27.00 \\
Learned, $\lambda_H{=}0$ & {1272} & 3.12 & 8.75 & 53.02 & 59.79 & 72.63 & 68.35 & 38.82 & 24.00 & 92.50 & 58.44 & 36.47 & 31.01 & 19.38 \\
Learned, $\lambda_H{=}0.01$ & {1273} & \textbf{3.07} & \textbf{8.68} & \textbf{55.17} & \textbf{61.57} & \textbf{73.72} & \textbf{71.09} & \textbf{40.44} & \textbf{24.40} & \textbf{93.60} & \textbf{60.00} & \textbf{47.92} & \textbf{34.56} & \textbf{27.25} \\
\cmidrule(lr){1-15}
Untied 4 & {465} & 3.66 & 11.14 & 40.35 & 48.63 & 70.29 & 62.37 & 30.63 & 16.40 & 88.90 & 51.08 & 7.20 & 18.62 & 10.50 \\
Untied 12 & {1233} & 2.89 & 7.93 & 57.62 & 64.65 & 74.32 & 72.22 & 41.89 & 25.20 & 94.40 & 61.47 & 48.82 & 39.82 & 31.38 \\
\bottomrule
\CodeAfter
\begin{tikzpicture}
\foreach \sepcol in {3,5,13}{
\draw[draw=resultgrouprule,line width=0.35pt] (1-|\sepcol) -- (3-|\sepcol);
\draw[draw=resultgrouprule,line width=0.35pt] (4-|\sepcol) -- (10-|\sepcol);
\draw[draw=resultgrouprule,line width=0.35pt] (11-|\sepcol) -- (17-|\sepcol);
\draw[draw=resultgrouprule,line width=0.35pt] (18-|\sepcol) -- (24-|\sepcol);
}
\end{tikzpicture}
\end{NiceTabular}
\end{table}

\paragraph{{Fixed-depth training breaks KV sharing.}}
{Fixed-depth training is the worst under forced KV sharing, with the highest perplexity, and sharing also cuts its GSM8K accuracy by over half at M and L (Table~\ref{tab:prior-ablation}). It needs full KV to achieve its best performance (Figure~\ref{fig:fixed-vs-pln-depth-scaling}, \appref{app:prior-results}).}

\paragraph{{The learned prior keeps sharing intact and outperforms {Fixed} PLN-5.}}
{With terminal KV sharing, it lowers validation PPL relative to Fixed PLN-5 by {$0.7$--$1.8\%$}, for at most {$1.6\%$} more training FLOPs. It raises the downstream average over Fixed PLN-5 at every scale, is the highest at S and L, and leads all three generation tasks at L. The entropy term is essential at L: without it, GSM8K drops from $47.92$ to $36.47$. {At S scale, training with its final distribution frozen from the first update lowers validation PPL by a further $0.8\%$, so the learned prior finds a better depth distribution without a search over distributions (\appref{app:static-prior}).} \appref{app:prior-results} sweeps $\lambda_H$ (Figure~\ref{fig:m-entropy-ablation}), gives training curves and additional ablations.}

\paragraph{{The learned prior steadily surpasses {Untied~4} at the same physical depth, with a lead that grows with scale, and approaches {Untied~12} at the same logical depth, with a gap that narrows.}}
{With $\lambda_H{=}0.01$, its downstream-average lead over {Untied~4} grows from $2.2$ points at S to $5.0$ at M and $8.9$ at L. Its validation PPL exceeds that of {Untied~12} by $13.6\%$ at S, $8.9\%$ at M, and $6.3\%$ at L. At L it trails {Untied~12} by $1.5$ points on the downstream average and $0.9$ on GSM8K, where it leads {Untied~4} by $40.7$.}

\subsection{Input Injection}
\label{sec:exp-architecture}\label{sec:exp-dynamics}
{We compare OrthoInj with DEQ-QKV~\citep{bai2019deep}, Huginn-Linear~\citep{huginn}, and Parcae-Decay~\citep{parcae}, all trained with PLN-5 (details in \appref{app:architecture-variants}).}

\begin{table}[!ht]
\centering\resulttablestyle
\caption{\textbf{{OrthoInj has the highest downstream average and the lowest validation PPL at every scale.}} All models train with PLN-5 and are evaluated at $R=5$ with terminal KV sharing, as in Table~\ref{tab:prior-ablation}. Recipes are in \appref{app:architecture-variants}.}
\label{tab:architecture}
\fontsize{9}{10.8}\selectfont\renewcommand{\arraystretch}{\resultrowstretch}
\setlength{\linewidth}{0.9\linewidth}
\begin{NiceTabular}{@{\hspace{6pt}}lccccccccccccc@{\hspace{6pt}}}
\CodeBefore
\rowcolor{papergray}{3,8,12}
\Body
\toprule
Model & \multicolumn{2}{c}{PPL $\downarrow$} & \multicolumn{8}{c}{Downstream (\%) $\uparrow$} & \multicolumn{3}{c}{Generation (\%) $\uparrow$} \\
\cmidrule(lr){1-1}\cmidrule(lr){2-3}\cmidrule(lr){4-11}\cmidrule(lr){12-14}
Injection & Val. & Wiki & LAMB. & HS & PIQA & ARC-E & ARC-C & OBQA & SciQ & AVG & GSM & DROP & MBPP+ \\
\midrule
\multicolumn{14}{l}{\textbf{S scale: (100M, 21.5B tokens)}} \\
DEQ-QKV & 6.30 & 23.05 & 25.17 & 30.56 & 60.99 & 43.31 & 23.81 & 15.40 & 76.10 & 39.33 & \textbf{2.05} & 9.42 & 0.12 \\
Huginn-Linear & 5.90 & 20.78 & 25.69 & 31.51 & \textbf{62.35} & 43.18 & \textbf{24.66} & \textbf{17.20} & 78.60 & 40.46 & 1.90 & 12.21 & 0.25 \\
Parcae-Decay & 5.75 & \textbf{19.90} & 26.28 & 32.14 & 61.86 & \textbf{46.89} & 23.81 & 16.40 & 81.60 & 41.28 & 1.90 & 12.08 & 0.25 \\
\midrule
OrthoInj & \textbf{5.67} & 20.02 & \textbf{27.54} & \textbf{32.69} & 61.97 & 46.46 & \textbf{24.66} & 16.40 & \textbf{81.80} & \textbf{41.65} & 1.90 & \textbf{13.66} & \textbf{1.62} \\
\midrule
\multicolumn{14}{l}{\textbf{M scale: (400M, 85.9B tokens)}} \\
Huginn-Linear & 3.99 & 12.55 & 36.56 & 44.15 & \textbf{68.12} & 56.69 & \textbf{29.44} & 17.00 & \textbf{88.70} & 48.67 & \textbf{16.68} & 16.71 & 8.00 \\
Parcae-Decay & 3.96 & 12.45 & 37.78 & 44.63 & \textbf{68.12} & \textbf{57.58} & 27.90 & \textbf{19.40} & 87.80 & 49.03 & 15.01 & \textbf{20.17} & \textbf{8.25} \\
\midrule
OrthoInj & \textbf{3.94} & \textbf{12.32} & \textbf{37.82} & \textbf{45.37} & \textbf{68.12} & 57.07 & 29.27 & 18.40 & 88.10 & \textbf{49.16} & 13.95 & 19.38 & 7.88 \\
\midrule
\multicolumn{14}{l}{\textbf{L scale: (1.6B, 343.6B tokens)}} \\
Parcae-Decay & 3.09 & \textbf{8.72} & 53.58 & 60.60 & \textbf{72.80} & \textbf{70.62} & \textbf{39.76} & 23.40 & \textbf{93.80} & 59.22 & \textbf{52.92} & 32.66 & \textbf{30.88} \\
\midrule
 OrthoInj & \textbf{3.08} & \textbf{8.72} & \textbf{54.51} & \textbf{61.62} & 72.63 & 69.53 & 38.65 & \textbf{24.80} & 93.10 & \textbf{59.26} & 51.71 & \textbf{32.77} & 27.75 \\
\bottomrule
\CodeAfter
\begin{tikzpicture}
\foreach \sepcol in {2,4,12}{
\draw[draw=resultgrouprule,line width=0.35pt] (1-|\sepcol) -- (3-|\sepcol);
\draw[draw=resultgrouprule,line width=0.35pt] (4-|\sepcol) -- (8-|\sepcol);
\draw[draw=resultgrouprule,line width=0.35pt] (9-|\sepcol) -- (12-|\sepcol);
\draw[draw=resultgrouprule,line width=0.35pt] (13-|\sepcol) -- (15-|\sepcol);
}
\end{tikzpicture}
\end{NiceTabular}

\end{table}
\begin{figure}[!htbp]
\centering
\includegraphics[width=\linewidth]{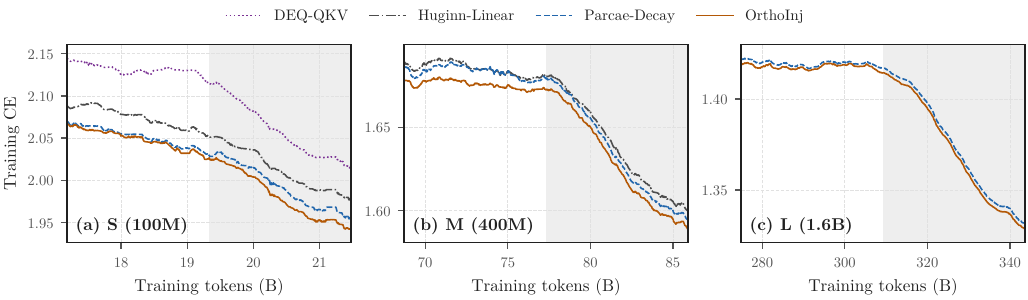}
\caption{{\textbf{OrthoInj ends training with the lowest loss at every scale.} Training CE of the Table~\ref{tab:architecture} runs over the final $20\%$ of training, as a centered rolling median over $5\%$ of each run; the gray band marks the learning-rate decay (final $10\%$ of updates). Within each scale, all runs see the same sampled depths.}}
\label{fig:injection-training-loss}
\end{figure}

\paragraph{{OrthoInj outperforms existing input injection schemes.}}
{With terminal KV sharing, it improves validation PPL over the strongest baseline, Parcae-Decay, and has the highest downstream average at {every scale} (Table~\ref{tab:architecture}).} {It also ends training with the lowest training loss, validation and WikiText PPL at every scale ({Figure~\ref{fig:injection-training-loss}}).} {Dropping prelude normalization and the projection each help, and their gains add up (\appref{app:architecture-results}).}

\subsection{{Rollout-State Reuse in RL}}
\label{sec:exp-rl}
{We fine-tune the L learned-prior model of Table~\ref{tab:prior-ablation} with Dr.~GRPO~\citep{dr_grpo} at $R=6$, the depth with its best GSM8K accuracy (\appref{app:posttraining}). The math run trains on $400$ questions from the GSM8K training split and evaluates on $500$ GSM8K test questions and on MATH500. The code run trains on $272$ MBPP+ tasks and evaluates on $100$ held-out MBPP+ tasks. In each run, every update method samples $16$ responses per training prompt.} {We compare reuse, which applies Neumann gradients to saved rollout states (\secref{sec:why-rl}), with BPTT through all recurrences. A control separates state reuse from the Neumann approximation: it computes the same Neumann gradient as reuse but rebuilds the endpoint by replaying all the recurrences without gradient tracking.}

\begin{table}[!htbp]
\centering\resulttablestyle
\caption[Rollout-state reuse during RL.]{\textbf{Reusing rollout states cuts scoring and backward time.}
{The L learned-prior model of Table~\ref{tab:prior-ablation} is fine-tuned with Dr.~GRPO at $R=6$ on GSM8K or MBPP+; the GSM8K runs are also evaluated on MATH500, which RL does not train on. Pass@$k$ uses eight samples per question. Time: median scoring and backward time per update, excluding rollout and the optimizer step. See \appref{app:posttraining} for training and evaluation details.}}
\label{tab:rl-state-reuse}
\fontsize{9}{10.8}\selectfont\renewcommand{\arraystretch}{1.18}\setlength{\tabcolsep}{5pt}
\begin{NiceTabular}{@{\hspace{6pt}}lcccccccc@{\hspace{6pt}}}
\CodeBefore
\rowcolor{papergray}{4}
\Body
\toprule
 & \multicolumn{5}{c}{RL on GSM8K} & \multicolumn{3}{c}{{RL on MBPP+}} \\
\cmidrule(lr){2-6}\cmidrule(lr){7-9}
Training & \multicolumn{2}{c}{GSM8K (\%) $\uparrow$} & \multicolumn{2}{c}{{MATH500 (\%) $\uparrow$}} & Time (s) $\downarrow$ & \multicolumn{2}{c}{{MBPP+ (\%) $\uparrow$}} & {Time (s) $\downarrow$} \\
\cmidrule(lr){1-1}\cmidrule(lr){2-3}\cmidrule(lr){4-5}\cmidrule(lr){6-6}\cmidrule(lr){7-8}\cmidrule(lr){9-9}
Update method & Pass@1 & Pass@8 & {Pass@1} & {Pass@8} & Update & {Pass@1} & {Pass@8} & {Update} \\
\midrule
\multicolumn{9}{l}{\textbf{L scale: (1.6B, 343.6B tokens)}} \\
Before RL & 40.63 & 77.60 & {13.40} & {42.80} & \resultmissing & {30.88} & {63.00} & {\resultmissing} \\
Full BPTT & 63.20 & 88.00 & {18.27} & {52.00} & 2.77 & {40.88} & {69.00} & {1.03} \\
Recompute + Neumann-4 & 60.05 & 89.60 & {17.32} & {50.60} & 2.09 & {37.12} & {69.00} & {0.65} \\
\midrule
Reuse + Neumann-4 & 61.65 & 90.20 & {17.82} & {49.60} & 1.39 & {39.38} & {64.00} & {0.51} \\
Reuse + Neumann-3 & 61.30 & 89.80 & {18.77} & {50.20} & 1.30 & {38.00} & {68.00} & {0.48} \\
\bottomrule
\CodeAfter
\begin{tikzpicture}
\foreach \sepcol in {2,7}{
\draw[draw=resultgrouprule,line width=0.35pt] (1-|\sepcol) -- (4-|\sepcol);
\draw[draw=resultgrouprule,line width=0.35pt] (5-|\sepcol) -- (10-|\sepcol);
}
\foreach \sepcol in {4,6,9}{
\draw[draw=resultgrouprule,line width=0.35pt] (2-|\sepcol) -- (4-|\sepcol);
\draw[draw=resultgrouprule,line width=0.35pt] (5-|\sepcol) -- (10-|\sepcol);
}
\end{tikzpicture}
\end{NiceTabular}
\end{table}

\paragraph{{Reusing rollout states {makes the update $2\times$ faster than full BPTT}.}}
{Reuse with Neumann-4 cuts scoring and backward time by $1.99\times$ on GSM8K and $2.02\times$ on MBPP+ (Table~\ref{tab:rl-state-reuse}). All four methods raise GSM8K pass@1 from $40.6\%$ to $60$--$63\%$ and MBPP+ pass@1 from $30.9\%$ to $37$--$41\%$.} {The {GSM8K} gains transfer to MATH500, which RL does not train on: every method raises pass@1 by $3.9$--$5.4$ points.}

\paragraph{{Reuse matches full BPTT's pass@1 within run-to-run variation.}}
{Reuse and the recomputation control compute nearly the same gradient (whole-model cosines $0.997$--$0.999$, \appref{app:posttraining}), so the gap between their pass@1 measures run-to-run variation: $1.6$ points on GSM8K, $0.5$ on MATH500, and $2.3$ on MBPP+. Reuse trails full BPTT's pass@1 by no more than this variation on any benchmark, and truncating to Neumann-3 keeps GSM8K pass@1 within $0.35$ points.}

\subsection{Distilled Prefill}
\label{sec:exp-distill}
{We distill the S, M, and L learned-prior models of \secref{sec:exp-prior} into two-block students initialized from the teacher's recurrent core (\secref{sec:why-distill}) and trained on $1/4$ of the teacher's pretraining tokens. Every prefill path is followed by the same teacher decoding at $R=5$, on the tasks of Table~\ref{tab:prior-ablation}. We compare against full teacher prefill ($R=5$), teacher prefill stopped after three recurrences ($R=3$) or one ($R=1$), and the Untied~4 baseline, which prefills and decodes on its own.}

\paragraph{{Distilled prefill trades downstream accuracy for $1.5$--$1.8\times$ faster prefill.}}
{On $8$K prompts, it trails the teacher's downstream average by $1.0$ (S), $2.3$ (M), and $4.5$ (L) points (Table~\ref{tab:distilled-prefill-main}). At matched latency, it edges out stopping the teacher early: it is $1$--$4\%$ faster than teacher prefill at $R=2$ and $0.5$--$0.9$ points higher on the downstream average. Prefill gains saturate early (Figure~\ref{fig:pd-scaling}): prefill at $R=3$ stays within $1.6$ points of the teacher, but it saves only $21$--$29\%$ of the teacher's latency, against $35$--$44\%$ for distilled prefill. A single recurrence is $12$--$23\%$ faster than distilled prefill but loses $2.7$--$3.5\times$ as much downstream accuracy. Better distillation objectives, student architectures such as mixture-of-experts layers, and the allocation of the distillation token budget could widen this margin and close the growing gap to the full teacher; we leave them to future work.}

\paragraph{{Teacher initialization drives the student's quality, and the KL term raises its generation scores.}}
{Under the same hidden and KL objective, teacher initialization adds $0.7$ (S) and $1.1$ (M) downstream points and raises M MBPP+ from $2.50$ to $6.12$ (Table~\ref{tab:distilled-prefill-ablation}); adding the KL term raises a randomly initialized student's M MBPP+ from $0.25$ to $2.50$. \appref{app:distill-ablation} ablates the loss terms and other uses of the student, and \appref{app:distill-r25} distills a deeper $R=25$ teacher.}

\begin{table}[h]
\centering\resulttablestyle
\caption{\textbf{Distilled prefill trades downstream accuracy for lower prefill latency.} Each path prefills the prompt, then the learned-prior teacher decodes at $R=5$ from the path's four terminal KV banks; Untied~4 prefills and decodes on its own. The distilled student is initialized from the teacher and trained with the hidden and KL objective of equation~\eqref{eq:hidden-distillation} on one quarter of its teacher's pretraining tokens; the teacher rows prefill at $R=5$, $3$, $2$, or $1$. Latency: mean $\pm$ s.d.\ of six $8$K prefills at batch size $8$. See \appref{app:distill-ablation} for evaluation details.}
\label{tab:distilled-prefill-main}
\fontsize{9}{10.8}\selectfont\renewcommand{\arraystretch}{\resultrowstretch}
\begin{NiceTabular}{@{\hspace{6pt}}lcccccccccccccc@{\hspace{6pt}}}
\CodeBefore
\rowcolor{papergray}{3,10,17}
\Body
\toprule
Prefill & \multicolumn{2}{c}{PPL $\downarrow$} & \multicolumn{8}{c}{Downstream (\%) $\uparrow$} & \multicolumn{3}{c}{Generation (\%) $\uparrow$} & Latency $\downarrow$ \\
\cmidrule(lr){1-1}\cmidrule(lr){2-3}\cmidrule(lr){4-11}\cmidrule(lr){12-14}\cmidrule(lr){15-15}
Path & Val. & Wiki & LAMB. & HS & PIQA & ARC-E & ARC-C & OBQA & SciQ & AVG & GSM & DROP & MBPP+ & 8K (ms) \\
\midrule
\multicolumn{15}{l}{\textbf{S scale: (100M, 21.5B tokens)}} \\
Untied 4 & 6.57 & 25.19 & 23.95 & 30.26 & 59.96 & 44.57 & 23.21 & 16.60 & 77.60 & 39.45 & 1.97 & 11.39 & 0.12 & 100.7{\scriptsize$\pm$0.5} \\
Teacher ($R{=}5$) & 5.67 & 19.91 & 27.36 & 32.51 & 63.00 & 45.50 & 25.09 & 17.20 & 80.90 & 41.65 & 1.59 & 14.08 & 0.62 & 195.0{\scriptsize$\pm$1.4} \\
Three recurrences ($R{=}3$) & 5.68 & 19.96 & 25.79 & 32.45 & 63.11 & 45.03 & 25.43 & 17.60 & 80.20 & 41.37 & 1.82 & 13.85 & 0.38 & 153.1{\scriptsize$\pm$1.7} \\
Two recurrences ($R{=}2$) & 5.73 & 20.23 & 22.45 & 31.81 & 62.51 & 43.27 & 24.66 & 16.40 & 79.90 & 40.14 & 1.90 & 13.07 & 0.25 & 131.0{\scriptsize$\pm$0.6} \\
One recurrence ($R{=}1$) & 5.95 & 21.02 & 19.29 & 31.13 & 61.21 & 41.62 & 23.98 & 15.40 & 73.70 & 38.05 & 1.52 & 10.86 & 0.00 & 110.7{\scriptsize$\pm$0.1} \\
\midrule
Distilled, teacher init & 5.71 & 20.16 & 23.56 & 31.93 & 63.17 & 43.86 & 25.17 & 17.00 & 79.70 & 40.63 & 1.74 & 14.41 & 0.62 & 125.9{\scriptsize$\pm$1.0} \\
\midrule
\multicolumn{15}{l}{\textbf{M scale: (400M, 85.9B tokens)}} \\
Untied 4 & 4.68 & 15.73 & 31.19 & 37.43 & 63.71 & 50.72 & 27.65 & 15.80 & 82.10 & 44.08 & 2.35 & 13.12 & 2.50 & 229.8{\scriptsize$\pm$1.2} \\
Teacher ($R{=}5$) & 3.93 & 12.28 & 38.58 & 46.17 & 67.41 & 57.58 & 28.75 & 17.00 & 88.00 & 49.07 & 15.62 & 17.11 & 10.38 & 495.8{\scriptsize$\pm$3.5} \\
Three recurrences ($R{=}3$) & 3.94 & 12.34 & 34.76 & 45.06 & 67.57 & 56.06 & 27.65 & 17.80 & 86.70 & 47.94 & 12.89 & 16.55 & 11.62 & 370.1{\scriptsize$\pm$1.4} \\
Two recurrences ($R{=}2$) & 3.99 & 12.53 & 27.34 & 44.04 & 66.05 & 53.87 & 28.07 & 18.20 & 83.80 & 45.91 & 9.17 & 16.22 & 7.00 & 304.2{\scriptsize$\pm$2.3} \\
One recurrence ($R{=}1$) & 4.23 & 13.12 & 13.76 & 41.60 & 64.53 & 49.45 & 27.90 & 15.40 & 75.70 & 41.19 & 2.27 & 8.56 & 0.12 & 245.9{\scriptsize$\pm$2.6} \\
\midrule
Distilled, teacher init & 3.97 & 12.44 & 31.86 & 44.63 & 67.74 & 55.68 & 27.90 & 15.20 & 84.40 & 46.77 & 8.95 & 16.20 & 6.12 & 299.3{\scriptsize$\pm$1.6} \\
\midrule
\multicolumn{15}{l}{\textbf{L scale: (1.6B, 343.6B tokens)}} \\
Untied 4 & 3.66 & 11.14 & 40.35 & 48.63 & 70.29 & 62.37 & 30.63 & 16.40 & 88.90 & 51.08 & 7.20 & 18.62 & 10.50 & 614.2{\scriptsize$\pm$3.6} \\
Teacher ($R{=}5$) & 3.07 & 8.68 & 55.17 & 61.57 & 73.72 & 71.09 & 40.44 & 24.40 & 93.60 & 60.00 & 47.92 & 34.56 & 27.25 & 1482.1{\scriptsize$\pm$32.0} \\
Three recurrences ($R{=}3$) & 3.08 & 8.70 & 50.15 & 61.01 & 72.80 & 68.69 & 38.99 & 24.60 & 92.50 & 58.39 & 44.28 & 26.98 & 26.38 & 1056.4{\scriptsize$\pm$2.4} \\
Two recurrences ($R{=}2$) & 3.11 & 8.87 & 39.82 & 58.96 & 71.44 & 64.73 & 35.67 & 24.00 & 90.10 & 54.96 & 33.06 & 20.97 & 20.25 & 848.6{\scriptsize$\pm$1.8} \\
One recurrence ($R{=}1$) & 3.31 & 9.23 & 19.35 & 53.57 & 70.46 & 58.54 & 31.06 & 20.60 & 82.30 & 47.98 & 6.22 & 12.82 & 0.88 & 642.0{\scriptsize$\pm$4.0} \\
\midrule
 Distilled, teacher init & 3.12 & 8.81 & 43.97 & 58.92 & 72.69 & 66.12 & 35.24 & 21.40 & 90.40 & 55.53 & 32.90 & 23.49 & 16.50 & 830.3{\scriptsize$\pm$3.0} \\
\bottomrule
\CodeAfter
\begin{tikzpicture}
\foreach \sepcol in {2,4,12,15}{
\draw[draw=resultgrouprule,line width=0.35pt] (1-|\sepcol) -- (3-|\sepcol);
\draw[draw=resultgrouprule,line width=0.35pt] (4-|\sepcol) -- (10-|\sepcol);
\draw[draw=resultgrouprule,line width=0.35pt] (11-|\sepcol) -- (17-|\sepcol);
\draw[draw=resultgrouprule,line width=0.35pt] (18-|\sepcol) -- (24-|\sepcol);
}
\end{tikzpicture}
\end{NiceTabular}
\end{table}

\section{Related Work}
\label{sec:related}

\paragraph{Looped models and fixed points.}
Looped language models apply a shared block repeatedly~\citep{universal_transformer} {to scale test-time compute~\citep{anil2022path,yang2026stars} and latent reasoning~\citep{huginn,ouro}. As depth grows, their states approach fixed points}~\citep{huginn,blayney2026mechanistic}, which DEQs solve for directly~\citep{bai2019deep}. {Prior work uses these fixed points to change how many recurrences run: it stabilizes them~\citep{bai2021stabilizing,yang2026stars,deeploop}, trains supervised states as attractors~\citep{eqr}, or halts at convergence~\citep{fixed_point_reasoners}. We further identify and verify the efficiency gains that fixed points bring to training, inference, and post-training.}

\paragraph{Looped model efficiency.}
{Near a fixed point, further recurrences barely change the state. We show that this property explains two shortcuts looped models already take and enables two new ones. Truncated backpropagation~\citep{huginn,hrm} thus approximates the equilibrium gradient without the warm-up that implicit gradients need~\citep{neumann_rbp,phantom_grad,jfb}. Of the schemes that share KV across recurrences~\citep{huginn,ouro,mixture_of_recursions}, terminal sharing is likewise nearly lossless when training supervises near fixed points. Terminal sharing also opens a shortcut for prefill. {Existing looped-model} distillation keeps students recurrent (\appref{app:related-distill}), and prompt-KV predictors~\citep{horton2024kvprediction,qiao2024swiftkv} must predict each recurrence's KV; distilled prefill instead uses a non-recurrent student to predict only the endpoint, from which one teacher recurrence builds every KV bank. Furthermore, fixed-point reuse carries endpoints across computations: it warm-starts solvers on related inputs (\appref{app:related-reuse}), and we bring it to RL, scoring rollouts from saved states instead of replaying the recurrence~\citep{ouro}.}

\paragraph{Shaping fixed points.}
Two choices shape the fixed point: which depths training supervises, and how the input enters each recurrence. Randomized training depths promote path-independent equilibria~\citep{anil2022path}, {but existing depth distributions are fixed~\citep{huginn}, marginalized over at inference~\citep{dun}, or learned as exits~\citep{graves2016act,banino2021pondernet}. Closest to ours, Ouro~\citep{ouro} and RL-Halting~\citep{kuo2026stochastic} learn entropy-regularized depth distributions from the recurrent state, but they unroll every training step to the maximum depth and do not target fixed-point shaping (\appref{app:related-depth}).} Injecting the input into every recurrence makes the fixed point depend on it~\citep{bai2019deep,bansal2022algorithm,huginn}. Parcae~\citep{parcae} improves injection by forcing contractiveness to {stabilize} training{.} {We further observe that the carried-over state can have a component along the injected input; OrthoInj projects it out so every recurrence receives the input intact (\appref{app:related-injection}).}

\section{Discussion and Conclusion}
\label{sec:conclusion}

Near a fixed point, the endpoint of a looped language model can replace its recurrent trajectory in training, inference, and post-training. Whether training produces such fixed points depends on the depth prior and on input injection. A learned depth prior lowers perplexity over a fixed PLN prior at the same mean depth while keeping terminal KV sharing intact. OrthoInj projects the carried-over state off the injected input, giving lower perplexity than the injection schemes of Huginn and Parcae. On these models, terminal KV sharing and rollout-state reuse cut memory or time at a small cost in quality, and distilled prefill trades downstream accuracy for prefill latency, with a gap that grows with scale.

\paragraph{Limitations.}
{We train models of up to $1.6$B parameters with one seed per configuration, and our ablations do not exhaust the design space. Larger scales and other model configurations remain untested, and the feasibility of these techniques in real-world applications needs further study.}

\begingroup
\small
\setlength{\bibsep}{0pt}
\bibliographystyle{unsrtnat}
\bibliography{references}
\endgroup
\clearpage
\appendix
\etocdepthtag.toc{appendix}
\begingroup
\etocsettagdepth{main}{none}
\etocsettagdepth{appendix}{subsection}
\etocsettocstyle{\section*{Appendix Contents}\noindent
\appref{app:conditions} gives the theoretical analysis, including the proof of Lemma~\ref{lem:joint-frozen}.
\appref{app:training} specifies the methods and training configurations, including why depth sampling exposes the model to converged context.
\appref{app:additional-results} collects supplementary experiments, including rollout-state reuse in RL, and
\appref{app:protocols} defines the shared datasets and evaluation protocols.
\appref{app:extended-related} relates the learned depth prior, OrthoInj, terminal KV sharing, rollout-state reuse, and distilled prefill to prior work.
\par\medskip
}{}
\tableofcontents
\endgroup
\clearpage
\section{Theoretical Analysis and Proofs}
\label{app:conditions}
We prove Lemma~\ref{lem:joint-frozen} and extend it to token-by-token decoding and finite depth (\appref{app:joint-frozen}), and relate truncated backpropagation to implicit gradients (\appref{app:gradients}).
\subsection{Joint Iteration and Frozen-Prefix Decoding}
\label{app:joint-frozen}
Fix parameters, tokens, injection, positions, and masks. As in \secref{sec:inference}, $h$ includes all current-token recurrent variables and $C$ all prefix inputs, including per-layer KV and additional recurrent state. If $C$ contains inputs other than KV, sharing only the KV banks falls under Lemma~\ref{lem:joint-frozen} only when those other inputs also coincide. By the triangle inequality, the assumptions of Lemma~\ref{lem:joint-frozen} give, for all $h,\widetilde h\in\mathcal H$ and $C,\widetilde C\in\mathcal C$, the joint Lipschitz bound
\begin{equation}
\label{eq:joint-frozen-lipschitz}
\bigl\|f_\theta(h;C)
      -f_\theta(\widetilde h;\widetilde C)\bigr\|
\le
\kappa\|h-\widetilde h\|
+\beta\|C-\widetilde C\|.
\end{equation}
\begin{proof}[Proof of Lemma~\ref{lem:joint-frozen}]
By the contraction mapping theorem~\citep{ortega_rheinboldt}, $f_\theta(\cdot;C^\star_{<t})$ has a
unique fixed point $h^\star$ in the complete space $\mathcal H$, and
$\widetilde h^r\to h^\star$.
Let $\delta_r=\|h^r-h^\star\|$. Invariance keeps $h^r\in\mathcal H$, so equation~\eqref{eq:joint-frozen-lipschitz} and $h^\star=f_\theta(h^\star;C^\star_{<t})$ give
\begin{equation}
\label{eq:joint-frozen-gap}
\delta_{r+1}\le\kappa\,\delta_r+\beta\|C^r_{<t}-C^\star_{<t}\|.
\end{equation}
With $\eta_{\max}=\sup_r\|C^r_{<t}-C^\star_{<t}\|<\infty$, induction gives $\delta_r\le\max\{\delta_0,\beta\eta_{\max}/(1-\kappa)\}$, so $\bar\delta=\limsup_r \delta_r$ is finite, and taking $\limsup$ in equation~\eqref{eq:joint-frozen-gap} gives $\bar\delta\le\kappa\,\bar\delta$. Hence $\bar\delta=0$ and $h^r\to h^\star$. Both iterations thus reach $h^\star$ and $C^\star_{<t}$, so any readout continuous in both has the same limit.
\end{proof}

\paragraph{Extension to one-by-one decoding.}
For $N$ tokens and nonempty closed sets $\mathcal H_t\subseteq\mathbb R^{d_t}$, write
\begin{equation}
\label{eq:causal-joint-recurrence}
C^r_{<t}=\Phi_{\theta,t}(h_{<t}^r),
\qquad
h_t^{r+1}=f_{\theta,t}(h_t^r;C^r_{<t}),
\end{equation}
where $h_{<t}^r=(h_1^r,\ldots,h_{t-1}^r)$ and $C^r_{<1}=C_{<1}$. Each $\Phi_{\theta,t}$ is continuous and depth-shared, and for every context $C\in\Phi_{\theta,t}(\prod_{t'<t}\mathcal H_{t'})$, $f_{\theta,t}(\cdot;C)$ maps $\mathcal H_t$ into itself; on these states and contexts, $f_{\theta,t}$ satisfies equation~\eqref{eq:joint-frozen-lipschitz} with constants $\kappa_t<1$ and $\beta_t<\infty$. From initial states in these sets, joint iteration and sequential decoding, which solves and freezes each token in turn, converge to the same limiting states; with common continuous readouts, they also give the same predictions.
\begin{proof}
Induct on $t$. At $t=1$ both executions use the fixed context $C_{<1}$ and converge to the fixed point $h_1^\star$ of $f_{\theta,1}(\cdot;C_{<1})$. If $h_{t'}^r\to h_{t'}^\star$ for all $t'<t$, continuity gives $C^r_{<t}\to\Phi_{\theta,t}(h_{<t}^\star)=:C^\star_{<t}$, the context that sequential decoding freezes at position $t$. Invariance and closedness keep $C^r_{<t}$ and $C^\star_{<t}$ in $\Phi_{\theta,t}(\prod_{t'<t}\mathcal H_{t'})$, so Lemma~\ref{lem:joint-frozen} gives $h_t^r\to h_t^\star$. Continuity of the readouts completes the proof, whatever the order in which joint positions converge.
\end{proof}
\paragraph{Finite depth and approximate terminal context.}
Let $\eta=\|\bar C_{<t}-C^\star_{<t}\|$ and $\bar h^{r+1}=f_\theta(\bar h^r;\bar C_{<t})$, and assume that the invariance and Lipschitz conditions also hold for $\bar C_{<t}$. Equation~\eqref{eq:joint-frozen-gap} with $h^r,C^r_{<t}$ replaced by $\bar h^r,\bar C_{<t}$ unrolls to
\begin{equation}
\label{eq:frozen-prefix-finite-error}
\|\bar h^{R}-h^\star\|
\le
\kappa^{R}\|\bar h^0-h^\star\|
+\frac{\beta(1-\kappa^{R})}{1-\kappa}\,\eta.
\end{equation}
The two terms capture unfinished token refinement and imperfect context, respectively; the second bounds the fixed-point displacement by $\beta\eta/(1-\kappa)$. A readout or loss with state and context Lipschitz constants $L_h$ and $L_C$ differs by at most $L_h$ times the right-hand side of equation~\eqref{eq:frozen-prefix-finite-error} plus $L_C\eta$. Here $\eta$ includes errors in earlier decoded tokens' banks.

\paragraph{Empirical corroboration.}
\label{app:joint-frozen-empirical}
We compare joint iteration with current-token iteration after solving and freezing the prefix, using the matched Fixed $R=5$ and independently sampled PLN-5 Small Huginn checkpoints of Table~\ref{tab:prior-ablation}. On prefixes passing the convergence checks, the two paths reach nearly identical endpoints within each model (Figure~\ref{fig:joint-frozen-endpoint}), consistent with Lemma~\ref{lem:joint-frozen}. Freezing a depth-five prefix instead leaves a gap five to six orders of magnitude larger: the median relative endpoint gap is $0.27$ for Fixed $R=5$ and $0.033$ for PLN-5, against about $4\times10^{-7}$ with a converged prefix. This finite-prefix control shows that the lemma needs the converged context. Qualified prefixes need about a hundred or more recurrences to converge, far beyond the depth five used for sharing, so practical sharing rests on the finite-depth bound in equation~\eqref{eq:frozen-prefix-finite-error} and on the measured cached perplexity (Table~\ref{tab:prior-ablation}), not on exact convergence.
\begin{figure}[!htbp]
\centering
\includegraphics[width=3.1in]{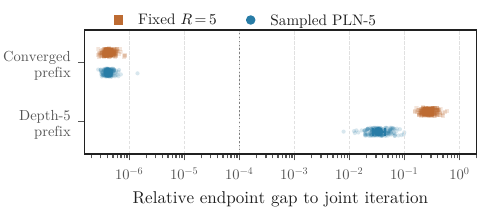}
\caption{\textbf{Freezing a converged prefix reproduces joint iteration (Lemma~\ref{lem:joint-frozen}).} Relative gap between the current-token endpoint of joint iteration and of iteration with the prefix frozen after convergence or after five recurrences. Dots are prefixes that pass the convergence checks ($206$ of $304$ for Fixed $R=5$, $191$ of $304$ for PLN-5); markers and bars give medians and interquartile ranges. The dotted line marks the $10^{-4}$ agreement tolerance.}
\label{fig:joint-frozen-endpoint}
\end{figure}

\FloatBarrier

We use 304 spans, $16$ from each validation source, with $1{,}534$ prefix tokens and one current token. Prefix-only and joint solves start independently from the model's evaluation-time state initialization. The frozen-prefix path updates only the current token, retaining each physical layer's historical KV and recomputing its own KV. All solves use FP32, PyTorch's reference (math) attention kernel, and at most $1{,}024$ recurrences. A solve qualifies when the maximum per-token relative state change and per-layer K/V change remain at most $10^{-5}$ for eight steps, followed by $16$ confirmation steps with state drift at most $5\times10^{-5}$ and a final one-step check. This criterion rules out cycles.

All three solves (prefix-only, joint, and frozen-prefix current-token) qualify on $206$ Fixed $R=5$ and $191$ PLN-5 spans; the other $98$ and 113 spans fail these checks and are dropped. Qualified prefix solves need median depths of $152.5$ (Fixed $R=5$) and $114$ (PLN-5), with maxima of $890$ and $961$. Every qualified pair agrees within relative state error $10^{-4}$ and prediction KL $10^{-7}$; the observed maximum state gaps are $8.24\times10^{-7}$ and $1.39\times10^{-6}$. The same checks pass from zero current-token initialization. The depth-five control uses the same qualified spans and requires its own current-token solve to pass. The two models' qualifying sets need not coincide.

\subsection{Truncated Backpropagation and Implicit Gradients}
\label{app:gradients}
For fixed input $x$, complete recurrent state $z=(\mathbf H,C)\in\R^{d_z}$ of \secref{sec:fixed-points}, and parameters $\theta\in\R^{d_\theta}$, we connect equation~\eqref{eq:fixed-point-sensitivity} to backpropagation through the final recurrent steps.

\paragraph{Finite-depth and equilibrium objectives.}
For a training-depth distribution $p$, the sampled-depth and equilibrium losses optimize different predictions:
\begin{equation}
 \mathcal L_{\rm depth}=\E_{R\sim p}[\ell_\theta(z^R,y)],\qquad
 \mathcal L_{\rm eq}=\E[\ell_\theta(z^\star,y)].
 \label{eq:two-objectives}
\end{equation}
Full BPTT differentiates the finite trajectory; implicit differentiation differentiates a local equilibrium branch. Equal states need not have equal derivatives. For $F_{\theta,a}(z;x)=az+(1-a)e$ with the prelude output $e=P_\theta(x)$, $z^0=0$, and $0\leq a<1$, the equilibrium $z^\star=e$ is independent of $a$, whereas $z^R=(1-a^R)e$ depends on it. Where $e=0$, both states are zero, but $\mathrm dz^R/\mathrm d\theta=(1-a^R)\,\mathrm de/\mathrm d\theta$ differs from $\mathrm dz^\star/\mathrm d\theta=\mathrm de/\mathrm d\theta$. TBPTT truncates the trajectory gradient; a convergent Neumann series approximates equilibrium sensitivity. Before convergence, the Neumann estimate is only a surrogate, so comparisons must state the objective and backward budget.

\paragraph{Equilibrium sensitivity and loss gradient.}
At $z^\star=F_\theta(z^\star;x)$~\citep{bai2019deep}, define $J_\star=\partial_zF_\theta(z^\star;x)\in\R^{d_z\times d_z}$ and $B_\star=\partial_\theta F_\theta(z^\star;x)\in\R^{d_z\times d_\theta}$, holding the state fixed in $B_\star$. Continuous differentiability near the equilibrium and invertibility of $I-J_\star$ give a locally differentiable equilibrium branch by the implicit function theorem:
\begin{equation}
 \underbrace{\frac{\mathrm d z^\star}{\mathrm d\theta}}_{D_\star}
 =\underbrace{\partial_zF_\theta}_{J_\star}
  \frac{\mathrm d z^\star}{\mathrm d\theta}
  +\underbrace{\partial_\theta F_\theta}_{B_\star},
 \qquad
 (I-J_\star)D_\star=B_\star.
 \label{eq:fixed-point-chain-rule}
\end{equation}
Thus $D_\star=(I-J_\star)^{-1}B_\star$. For a single example, $\mathcal L_{\rm eq}(\theta)=\ell_\theta(z^\star,y)$; define column gradients $g=\nabla_z\ell_\theta$ and $c=\nabla_\theta\ell_\theta$ at the equilibrium, holding the state fixed in $c$. The loss gradient is $G_{\rm imp}=c+D_\star^\top g$. Its adjoint form,
\begin{equation}
 (I-J_\star^\top)v=g,\qquad G_{\rm imp}=c+B_\star^\top v,
 \label{eq:implicit-gradient}
\end{equation}
uses the vector--Jacobian products (VJPs) $J_\star^\top v$ and $B_\star^\top v$ through one evaluation of $F_\theta$, with $v$ fixed, avoiding an explicit Jacobian or stored forward trajectory~\citep{bai2019deep}.

\paragraph{Neumann truncation and a detached suffix.}
The stronger condition that $J_\star$ has spectral radius below one ensures $J_\star^b\to0$, so the identity $(I-J_\star)\sum_{k=0}^{b-1}J_\star^k=I-J_\star^b$ yields
\begin{equation}
 D_\star=\sum_{k=0}^\infty J_\star^kB_\star,
 \qquad D_b=\sum_{k=0}^{b-1}J_\star^kB_\star,
 \qquad D_\star-D_b=J_\star^bD_\star.
 \label{eq:neumann-sensitivity-tail}
\end{equation}
Equivalently, the iteration $v_0=0$, $v_{k+1}=g+J_\star^\top v_k$ gives $G_b=c+B_\star^\top v_b$~\citep{neumann_rbp,phantom_grad}. Invertibility suffices for equation~\eqref{eq:implicit-gradient}; the spectral condition ensures convergence.

Starting from $\bar z^0=\sg(z^\star)$, each update $\bar z^{k+1}=F_\theta(\bar z^k;x)$ remains at the equilibrium and uses the same $J_\star,B_\star$. Its sensitivity after $k$ updates is $D_k$, since $D_0=0$ and $D_{k+1}=J_\star D_k+B_\star$; hence the loss gradient is $c+D_b^\top g=G_b$. A Euclidean contraction bound controls the omitted tail:
\begin{lemma}[Truncated backpropagation at equilibrium; adapted from \citet{neumann_rbp}]
\label{lem:suffix}
Suppose $F_\theta$ is continuously differentiable near the current equilibrium, $\ell_\theta$ is differentiable there, and $\norm{J_\star}_2\leq \kappa<1$. Start from $\sg(z^\star)$ and differentiate $b$ applications of the same deterministic recurrence, retaining all direct parameter paths. Then
\begin{equation}
 G_b=c+B_\star^\top\sum_{k=0}^{b-1}(J_\star^\top)^k g,\qquad
 \norm{G_{\rm imp}-G_b}_2\leq\norm{B_\star}_2\norm{g}_2\frac{\kappa^b}{1-\kappa}.
 \label{eq:gradient-bound}
\end{equation}
\end{lemma}
\begin{proof}[Proof of Lemma~\ref{lem:suffix}]
The sensitivity recursion above establishes the expression for $G_b$. Subtracting it from the infinite series and using $\norm{J_\star}_2\leq \kappa<1$ gives
\[
 \norm{G_{\rm imp}-G_b}_2
 =\left\lVert B_\star^\top\sum_{k=b}^\infty(J_\star^\top)^k g\right\rVert_2
 \leq\norm{B_\star}_2\norm{g}_2\sum_{k=b}^\infty \kappa^k
 =\norm{B_\star}_2\norm{g}_2\frac{\kappa^b}{1-\kappa}.
\]
\end{proof}
This specializes the TBPTT/Neumann connection~\citep{neumann_rbp,phantom_grad}. If only the spectral radius of $J_\star$ is below one, then for any $\gamma$ between the spectral radius and one, $\norm{J_\star^k}_2=O(\gamma^k)$, and the same argument gives $\norm{G_{\rm imp}-G_b}_2=O(\norm{B_\star}_2\norm{g}_2\gamma^b)$. Direct paths remain essential: $B_\star$ includes input/prelude injection, and $c$ includes the output head. Detaching prefix KV removes further dependencies and therefore changes the estimator.

\paragraph{Finite-trajectory BPTT and TBPTT.}
Along $z^{r+1}=F_\theta(z^r;x)$, let $J_r=\partial_zF_\theta(z^r;x)$ and $B_r=\partial_\theta F_\theta(z^r;x)$. The recursion $\partial z^{r+1}/\partial\theta=J_r\,\partial z^r/\partial\theta+B_r$ from a detached window start $z^{R-b}$ gives $\partial z^R/\partial\theta=\tilde D_b$, the TBPTT sensitivity of equation~\eqref{eq:fixed-point-sensitivity} with $i=R-1-k$:
\begin{equation}
 \tilde D_b
 =\sum_{i=R-b}^{R-1}
   \bigl(J_{R-1}J_{R-2}\cdots J_{i+1}\bigr)B_i.
 \label{eq:tbptt-trajectory-sensitivity}
\end{equation}
The empty product is $I$. Full BPTT includes all terms $i=0,\ldots,R-1$, which give $\tilde D_R$, and $(J_{R-1}\cdots J_0)\,\mathrm d z^0/\mathrm d\theta$ for parameter-dependent initialization. Multiplying the transposed sensitivity by the endpoint loss gradient and adding the loss's direct parameter gradient gives the loss gradient.

Equation~\eqref{eq:tbptt-trajectory-sensitivity} reduces to $D_b$ when all suffix Jacobians equal $J_\star,B_\star$ and approaches the endpoint Neumann estimate as they converge. Because early derivatives come from transient states and truncation adds its own error, neither estimator is uniformly more accurate. Controlling the error at approximate states additionally requires locally Lipschitz state and parameter Jacobians, a uniform state-Jacobian bound below one, regular loss derivatives, and small state errors throughout the window. A small endpoint residual alone is not enough.

\FloatBarrier
\section{Additional Method Details}
\label{app:training}
We specify the shared model and training configuration (\appref{app:setup-details}), explain why sampled training depths expose the model to converged context (\appref{app:depth-exposure}), and detail the learned depth prior of \secref{sec:fixed-point-depth-prior} (\appref{app:prior-algorithm}), and the input-injection recipes of \secref{sec:architecture} (\appref{app:architecture-variants}).

\subsection{Models and Training Contract}
\label{app:setup-details}

The model layout follows \secref{sec:setup}, and Table~\ref{tab:setup-dense-ladder-config} gives the full configurations.
All models use separate input and output embeddings. They use the Jais64k
tokenizer~\citep{sengupta2023jais} with a $64{,}256$-token vocabulary and bfloat16 arithmetic. The weighted pretraining mixture covers web
text, QA, code, mathematics, reasoning, planning, multilingual text, and
synthetic data. Its QA component includes TxT360~\citep{txt360data2024}.
Architectural variants and Untied~12 differ in size, so each comparison states its scale and configuration.

\begin{table}[ht]
  \centering\papertablestyle
  \caption{\textbf{Model configurations at three scales.}
  Blocks count prelude, recurrent-core, and coda blocks of the looped models; Untied~4 and Untied~12 stack $4$ and $12$ blocks.
  Width and FFN give hidden and feedforward dimensions, and Q and KV give query and key--value head counts.
  Non-emb. and Total give Untied~4 parameter counts without and with embeddings; the looped rows add Huginn's width$\times$width input adapter.
  The main text rounds the scales to $100$M, $400$M, and $1.6$B.
  Parameter and token counts are in billions; the peak learning rate is tuned per run (\appref{app:setup-lr}).}
  \label{tab:setup-dense-ladder-config}
  \renewcommand{\appendixtablefontsize}{9}\setlength{\tabcolsep}{18pt}%
  \newcommand{\grayband}[1]{\multicolumn{4}{@{}>{\columncolor{papergray}[0pt][0pt]\hspace{6pt}}l@{}}{\rule{0pt}{2.00ex}\tablehead{#1}}\\}%
  \begin{tabular}{@{\hspace{6pt}}lrrr@{\hspace{6pt}}}
    \toprule
     & S & M & L \\
    \midrule
    \grayband{A: Architecture and attention geometry}
    Blocks (prelude/core/coda) & \multicolumn{3}{c}{1 / 2 / 1} \\
    Width & 1,536 & 3,072 & 6,144 \\
    FFN & 4,352 & 8,704 & 17,408 \\
    Q heads & 24 & 48 & 96 \\
    KV heads & 6 & 12 & 24 \\
    Head dim. & 64 & 64 & 64 \\
    \midrule
    \grayband{B: Parameters (B)}
    Non-emb. & 0.104 & 0.415 & 1.661 \\
    Total & 0.301 & 0.810 & 2.451 \\
    Looped non-emb. & 0.106 & 0.425 & 1.699 \\
    Looped total & 0.304 & 0.819 & 2.488 \\
    \midrule
    \grayband{C: Batch and token budget}
    Sequence length & \multicolumn{3}{c}{8,192} \\
    Global batch (sequences) & \multicolumn{3}{c}{512} \\
    Tokens per update & \multicolumn{3}{c}{4,194,304} \\
    Updates & 5,120 & 20,480 & 81,920 \\
    Tokens (B) & 21.47 & 85.90 & 343.60 \\
    \midrule
    \grayband{D: Optimizer}
    Optimizer & \multicolumn{3}{c}{AdamW} \\
    $(\beta_1,\beta_2)$ & \multicolumn{3}{c}{$(0.90,\,0.98)$} \\
    $\epsilon$ & \multicolumn{3}{c}{$10^{-8}$} \\
    Weight decay & \multicolumn{3}{c}{0.10} \\
    Gradient clipping & \multicolumn{3}{c}{1.00} \\
    LR schedule & \multicolumn{3}{c}{WSD} \\
    Linear warmup & \multicolumn{3}{c}{first 5\% of updates} \\
    Cosine decay & \multicolumn{3}{c}{last 10\%, to $0.1\times$ peak} \\
    \bottomrule
  \end{tabular}
\end{table}
Each microbatch draws one training depth (\secref{sec:fixed-point-depth-prior}), so a sampled-depth run
draws one depth per gradient-accumulation step. Compute depends on the realized depths and backward truncation.

\paragraph{Learning-rate selection.}
\label{app:setup-lr}
Each scale has one learning-rate range, and every run tunes its peak rate within it: $[0.5,4]\times10^{-3}$ at S, $[0.25,4]\times10^{-3}$ at M, and $[0.125,2]\times10^{-3}$ at L. Screens keep the final WSD horizon but stop after $512$--$1{,}024$ updates at S, $2{,}048$--$12{,}288$ at M, and $8{,}192$--$32{,}768$ at L, and select by native validation CE at $R=5$ on $304$ frozen $1{,}536$-token sequences, distinct from \textsc{MixVal}. Every selected rate is interior to its run's grid. A variant that changes only $\lambda_H$ or the depth-draw correlation reuses the rate of its screened configuration. M OrthoInj in Table~\ref{tab:architecture} instead trains at $10^{-3}$, Parcae-Decay's selected M rate; its own screen selects $1.5\times10^{-3}$ (\appref{app:injection-lr-m} compares both rates).

\subsection{Why Sample Training Depths?}
\label{app:depth-exposure}
Let $p(R)$ be the training-depth distribution and $\ell_t$ the next-token CE at position $t$. Sampling one sequence-wide horizon optimizes $\mathcal L_{\rm depth}$ of equation~\eqref{eq:two-objectives} with $\ell_\theta(z^R,y)=N^{-1}\sum_t\ell_t$~\citep{huginn}. Fix parameters, input, initialization, and a trajectory whose prefix KV is converged from depth $r_{\rm c}$. Backpropagating through the final $1\leq b(R)\leq R$ steps gives the expected number of backpropagated updates that read the converged prefix KV, which we call its exposure,
\begin{equation}
 w_{r_{\rm c}}=\sum_R p(R)\min\!\left\{b(R),(R-r_{\rm c})_+\right\},
 \qquad (\cdot)_+=\max(\cdot,0).
 \label{eq:random-depth-readers}
\end{equation}
The window $R-b(R)+1,\ldots,R$ intersects updates after $r_{\rm c}$ in $\min\{b(R),(R-r_{\rm c})_+\}$ steps. Equation~\eqref{eq:random-depth-readers} gives $w_{r_{\rm c}}=0$ for fixed $R=5$ when $r_{\rm c}\geq5$, but positive PLN-5 exposure whenever $p(R>r_{\rm c})>0$; sufficiently deep fixed training can also provide exposure. With our rule $b(R)=\min(R,10)$ (\appref{app:prior-algorithm}) and the depths drawn by the Small PLN-5 run of Table~\ref{tab:prior-ablation}, PLN-5 gives $w_5=1.10$ ($p(R>5)=35.4\%$) and $w_{10}=0.15$, so exposure concentrates on prefixes that converge within about ten recurrences. The count ignores attention weights and gradient magnitudes, and Figure~\ref{fig:prefix-dynamics} does not give per-bank convergence depths or training-time read traces.

Random-depth loss alone does not imply KV stationarity: $z^r=(z_{\rm a},(-1)^r z_{\rm b})$ has constant loss if the head reads only $z_{\rm a}$, while any KV that reads the second component alternates. Supervising late converged states also does not ensure that early, transient states respond well to terminal KV. Exposure is thus necessary but not sufficient for sharing, so Table~\ref{tab:prior-ablation} measures sharing directly.

\subsection{Learned Depth Prior}
\label{app:prior-algorithm}
Every learned prior uses equation~\eqref{eq:depth-prior-objective}. Table~\ref{tab:prior-reward-ablation}'s four reward arms (CE, log CE, exp CE, and inverse perplexity, $-\exp(-\mathrm{CE})$) use $\lambda_H=0.01$, and Table~\ref{tab:prior-ablation} reports inverse perplexity at $\lambda_H\in\{0,0.01\}$.
\paragraph{Sampling and controller updates.}
The fixed PLN-5 baseline caps equation~\eqref{eq:huginn-pln} at $R\le64$ and recalibrates $\mu$ so that $\mathbb E[\min(R,64)]=5$. The learned prior initializes from this capped PLN: $\phi_0=\log p_{\rm PLN}$ uses $64$ logits; tail mass accumulates at depth $64$ and all masses are positive. Sampling and update order follow Algorithm~\ref{alg:learned-prior}.

The backpropagation cap is $b_{\max}=10$; earlier recurrent states are detached, but gradients still reach the prelude through the input injected at each backpropagated recurrence, and the output head. Equation~\eqref{eq:depth-prior-objective} uses $\bar R=5$ and $\lambda_m=1$; global token weighting is specified below. The prior's separate Adam optimizer uses LR $10^{-3}$, betas $(0.90,0.999)$, and zero weight decay. During a prior warmup of $32$ updates, the controller collects its baseline without changing the logits. The fixed arm freezes the logits in the same sampler.

\paragraph{Loss aggregation and baseline.}
For microbatch $i$, detached loss sums and unmasked target counts are reduced across ranks, and dividing by the global count $N_i$ gives $\ell_i$. The within-step average of a per-microbatch statistic $\chi_i$ is
\begin{equation}
 \widehat{\mathbb E}_{n}[\chi]=\frac{\sum_i N_i\chi_i}{\sum_i N_i}.
 \label{eq:depth-prior-weighted-average}
\end{equation}
For raw-CE control, with $\ell$ the globally reduced mean CE at sampled depth $r$, $\hat\ell_{n}$ the pre-step baseline, and $\bar R$ the target mean, the controller minimizes
\begin{equation}
\mathcal J_{\mathrm{CE}}(\phi)
=\widehat{\mathbb E}_{n}\left[
\bigl(\operatorname{stopgrad}(\ell)-\hat\ell_{n}\bigr)
\log p_\phi(r)\right]
-\lambda_H H(p_\phi)+\lambda_m\bigl(\mathbb E_{p_\phi}[R]-\bar R\bigr)^2.
\label{eq:depth-prior-controller}
\end{equation}
Both regularizers are evaluated analytically over active depths and differentiated only through the logits.
The objective and baseline share this weighting. The first weighted CE initializes the baseline, and its EMA also runs during the prior warmup:
\begin{equation}
 \hat\ell_{n+1}=\nu \hat\ell_{n}+(1-\nu)\widehat{\mathbb E}_{n}\!\left[\operatorname{stopgrad}(\ell)\right],
 \label{eq:depth-prior-baseline}
\end{equation}
where $\nu\in[0,1)$ is the EMA decay, set to $0.95$.

\paragraph{Alternative prior signals.}
Three PLN-initialized arms replace raw CE in equation~\eqref{eq:depth-prior-objective} by one of three rewards, computed after global token-weighted reduction: $\omega(\ell)=-\log(\max(\ell,10^{-12}))$, $-\exp(\ell)$, or $\exp(-\ell)$. Their advantage is $A_i=s_{n-1}(u_i-\hat u_{n-1})$ with $u_i=\omega(\ell_i)$, where $s$ is the ratio of previous CE and reward standard deviations, each floored at $10^{-12}$; it starts at one. These arms retain the model CE objective and mean penalty, adding no forward passes.

\paragraph{Reward form and entropy.}
Table~\ref{tab:prior-reward-ablation} names each arm by its cost $-\omega(\ell)$. Centering, scaling, and optimization order follow Algorithm~\ref{alg:learned-prior}. The entropy in \secref{sec:fixed-point-depth-prior} has logit derivative $\partial H/\partial\phi(r)=-p_\phi(r)[\log p_\phi(r)+H(p_\phi)]$, requiring no model forward pass. It does not constrain tail mass at a particular depth; \appref{app:m-entropy-matched} evaluates its coefficient.

\paragraph{Uniform controls.}
$U_{1:64}$ initializes all logits equally (mean $32.50$) and retains full support. $U_{1:10}$ restricts uniform initialization and normalization to depths $1$--$10$ (mean $5.50$), keeping depths $11$--64 inactive. Both target a mean depth of five and use raw-CE feedback with the PLN arm's model and prior optimizers. The restricted control changes both initialization and support. Within the training budget, the mean penalty does not pull $U_{1:64}$ down to five: its final prior mean is $6.15$ (Table~\ref{tab:prior-reward-ablation}).

\paragraph{Compute accounting.}
Backward-window visits add backward and recomputation work. Model-FLOP totals count each logical update once and exclude recomputation, controller work, and repeated work after restarts. Learning rates, schedules, and batch layouts follow \appref{app:setup-details}.

\subsection{Architecture Variants and Initialization}
\label{app:architecture-variants}
Table~\ref{tab:input-injection} defines the recipes of \secref{sec:architecture}; OrthoInj's anchor is the vector it projects against.
\begin{table}[!htbp]
\centering\papertablestyle\renewcommand{\appendixtablefontsize}{9}%
\caption{\textbf{Input-injection recipes.} $e=P_\theta(x)$ is the prelude output, $\tilde e$ is the injected input, and RMSNorm uses learned gains. DEQ-QKV injects at internal sites (superscript $(j)$ indexes a recurrent layer); the others inject at the core entrance. $\Lambda=\operatorname{diag}(\exp(-\Delta\odot \alpha))$ with $\Delta=\operatorname{softplus}(\Delta_{\rm raw})$ and $\alpha=\exp(\alpha_{\rm raw})$, and OrthoInj applies $Q_q$ of equation~\eqref{eq:orthoinj} with $q=\Delta\odot W\tilde e$ and $\epsilon=10^{-6}$.}
\label{tab:input-injection}
\begin{tabular}{@{}lll@{}}
\toprule
Injection recipe & Injected input $\tilde e$ & Input to recurrent core \\
\midrule
DEQ-QKV & $e$ & $W_{\rm QKV}^{(j)}h^{r,(j)}+W_e^{(j)}\tilde e$ \\
Huginn-Linear~\citep{huginn} & $e$ & $W_{\rm lin}[h^r;\tilde e;1]$ \\
Parcae-Decay~\citep{parcae} & $\operatorname{RMSNorm}(e)$ & $\Lambda h^r+\Delta\odot W\tilde e$ \\
Parcae w/o norm & $e$ & $\Lambda h^r+\Delta\odot W\tilde e$ \\
OrthoInj (ours) & $e$ & $Q_q\Lambda h^r+q$ \\
\bottomrule
\end{tabular}
\end{table}

\paragraph{Input source and retention.}
DEQ-QKV injects the prelude output, which stays fixed across recurrences, into Q/K/V through separate bias-free, Gaussian-initialized projections in the two physical recurrent layers, shared across visits. It replaces the injection at the core entrance and the diagonal retention of earlier states, and keeps the prelude and the Gaussian initialization of the recurrent state.

Huginn-Linear uses a learned concatenation projection (the appended $1$ gives its bias) followed by a parameter-free RMSNorm after the full core:
\begin{equation}
 \zeta_t^r=W_{\rm lin}[h_t^r;e_t;1],\qquad
 \mathbf H^{r+1}=\operatorname{RMSNorm}_{\rm exit}\!\left[\Gamma_\theta(\zeta_1^r,\ldots,\zeta_N^r)\right].
 \label{eq:linear-injection}
\end{equation}
The diagonal injection recipes use
\begin{equation}
 \begin{gathered}
 \Delta=\operatorname{softplus}(\Delta_{\rm raw}),\qquad \alpha=\exp(\alpha_{\rm raw}),\qquad \Lambda=\operatorname{diag}(\exp(-\Delta\odot\alpha)),\\
 q_t=\Delta\odot W\tilde e_t,\qquad \zeta_t^r=\Lambda h_t^r+q_t.
 \end{gathered}
 \label{eq:input-injection}
\end{equation}
Here $\Delta_{\rm raw},\alpha_{\rm raw}\in\mathbb R^d$ are shared across visits, $\odot$ denotes channelwise multiplication, $W\in\mathbb R^{d\times d}$ is a full, bias-free input projection, and $q_t$ is the injected vector of \secref{sec:architecture}; the core update $\mathbf H^{r+1}=\Gamma_\theta(\zeta_1^r,\ldots,\zeta_N^r)$ follows equation~\eqref{eq:additive-input-injection}.
With the prelude output $e=P_\theta(x)$ of \secref{sec:fixed-points}, Parcae-Decay uses $\tilde e=\operatorname{RMSNorm}(e)$ with learned gains, $W_{\rm init}=I$, $\alpha_{\rm init}=1$, and $\Delta_{\rm init}=\tfrac12\log5$ (implemented by inverse softplus). This gives retention $1/\sqrt5$, matching the intended linear-injection scale~\citep{parcae_implementation}; prelude normalization addresses large-output instability~\citep[Appendix~J]{parcae}. Parcae w/o norm uses $\tilde e=e$ and $\Delta_{\rm init}=1$, retaining $W_{\rm init}=I$ and $\alpha_{\rm init}=1$; it changes both normalization and initialization. OrthoInj uses $\tilde e=e$ without prelude RMSNorm and the Parcae-Decay initialization ($W_{\rm init}=I$, $\alpha_{\rm init}=1$, $\Delta_{\rm init}=\tfrac12\log5$). It projects the carryover against the injected vector $q$ as in \secref{sec:architecture}, with FP32 projection inner products and $\epsilon=10^{-6}$, and retains all block RMSNorms. An earlier OrthoInj variant, used only for the M learning-rate screen (\appref{app:setup-lr}), uses $\Delta_{\rm init}=1$ and the raw input $e$ as its anchor. {{Unlike the released Parcae implementation, which applies a learned linear readout $C$ to the final recurrent state before the coda~\citep{parcae_implementation}, these diagonal-injection recipes pass the final state to the coda directly ($C=I$); the readout lies outside the recurrence and does not enter the fixed-point iteration.}}

\paragraph{Linear injection and recurrent normalization.}
\label{app:linear-normalization}
The original Huginn normalizes before attention and FFN, after residual additions, and at the core exit~\citep{huginn}. We use PreNorm blocks; Huginn-Linear adds the parameter-free core-exit RMSNorm in equation~\eqref{eq:linear-injection}.

Without exit normalization, S PLN-25 training reached a nonfinite loss at all five tested learning rates. Replaying a fixed checkpoint traced the amplification to the injection interface; adding the exit RMSNorm kept the same inputs finite through recurrence $64$. This replay only diagnoses forward stability; every evaluated Huginn-Linear model is trained from scratch with the exit norm.

\FloatBarrier
\section{Additional Experimental Results}
\label{app:additional-results}
We report supplementary results and their settings; \appref{app:protocols} gives the shared protocols.

\subsection{Depth-Prior Learning and Test-Time Scaling}
\label{app:prior-results}

\paragraph{Training and evaluation.}
Every depth-prior run uses Parcae w/o norm injection (Table~\ref{tab:input-injection}). Every sampled-depth row of Table~\ref{tab:prior-ablation} draws one depth per microbatch, independently across microbatches and shared across ranks. Runs use one training seed and the setup of \appref{app:setup-details}.

Native MixVal and CacheEval follow \appref{app:setup-mixture} and \appref{app:setup-cache} at depth five. Sharing penalties are within-checkpoint ratios $100(\mathrm{PPL}_{\rm terminal}/\mathrm{PPL}_{\rm FullKV}-1)$, not differences from native MixVal. Model-FLOP accounting follows \appref{app:prior-algorithm}. For Fixed PLN-5, the trainer logs a constant nominal cost per update, so Tables~\ref{tab:prior-ablation} and~\ref{tab:prior-reward-ablation} instead apply the learned priors' per-depth cost to the depths each run drew. Every column of Table~\ref{tab:prior-ablation} except FLOPs follows the terminal-KV protocol of \appref{app:setup-terminal}, and Table~\ref{tab:prior-ablation-fullkv} evaluates the same checkpoints without KV reuse. The task AVG excludes the generation scores (GSM8K, DROP, and MBPP+).

\begin{table}[!ht]
\centering\resulttablestyle
\caption{\textbf{Table~\ref{tab:prior-ablation} without KV reuse.} Same checkpoints as Table~\ref{tab:prior-ablation}, evaluated at $R=5$. Val.\ and Wiki are native PPL over full sequences, choice tasks are scored without a cache, and generation keeps all twelve visit-specific KV banks. Table~\ref{tab:prior-ablation} instead scores Val.\ and Wiki only on the tokens after a $4{,}096$-token prompt that continue a document begun inside it (\appref{app:setup-terminal}), so only the task columns compare directly across the two tables. Untied~12 stacks twelve distinct blocks ($0.31$B, $1.25$B and $4.98$B non-embedding parameters at S, M and L, three times the looped models) and has twelve KV banks; it is scored the same way and not bolded.}
\label{tab:prior-ablation-fullkv}
\fontsize{9}{10.8}\selectfont\renewcommand{\arraystretch}{\resultrowstretch}
\begin{NiceTabular}{@{\hspace{6pt}}lccccccccccccc@{\hspace{6pt}}}
\CodeBefore
\rowcolor{papergray}{3,9,15}
\Body
\toprule
Training & \multicolumn{2}{c}{PPL $\downarrow$} & \multicolumn{8}{c}{Downstream (\%) $\uparrow$} & \multicolumn{3}{c}{Generation (\%) $\uparrow$} \\
\cmidrule(lr){1-1}\cmidrule(lr){2-3}\cmidrule(lr){4-11}\cmidrule(lr){12-14}
Prior & Val. & Wiki & LAMB. & HS & PIQA & ARC-E & ARC-C & OBQA & SciQ & AVG & GSM8K & DROP & MBPP+ \\
\midrule
\multicolumn{14}{l}{\textbf{S scale: (100M, 21.5B tokens)}} \\
Fixed $R{=}5$ & \textbf{6.50} & \textbf{20.62} & \textbf{27.79} & \textbf{32.57} & 62.19 & 47.56 & 25.77 & 15.20 & 78.80 & 41.41 & 2.05 & \textbf{15.28} & 0.00 \\
Fixed PLN-5 & 6.69 & 21.47 & 26.04 & 32.14 & 62.13 & \textbf{47.94} & 24.23 & 15.60 & 81.20 & 41.33 & \textbf{2.27} & 12.39 & \textbf{0.62} \\
Learned, $\lambda_H{=}0$ & 6.60 & 21.14 & 26.94 & 32.09 & 62.19 & 45.92 & 24.66 & \textbf{18.00} & \textbf{81.60} & 41.63 & 1.90 & 12.72 & 0.25 \\
Learned, $\lambda_H{=}0.01$ & 6.60 & 21.12 & 27.05 & 32.11 & \textbf{62.84} & 45.45 & \textbf{26.02} & 17.40 & 81.30 & \textbf{41.74} & 2.12 & 13.32 & \textbf{0.62} \\
\cmidrule(lr){1-14}
Untied 12 & 5.80 & 18.20 & 31.67 & 34.78 & 64.80 & 50.93 & 26.02 & 15.80 & 82.30 & 43.76 & 1.59 & 12.50 & 2.75 \\
\midrule
\multicolumn{14}{l}{\textbf{M scale: (400M, 85.9B tokens)}} \\
Fixed $R{=}5$ & \textbf{4.51} & \textbf{12.74} & \textbf{40.19} & \textbf{46.04} & 67.41 & 57.24 & \textbf{30.55} & \textbf{19.00} & \textbf{90.20} & \textbf{50.09} & \textbf{15.54} & \textbf{21.85} & 13.12 \\
Fixed PLN-5 & 4.66 & 13.31 & 39.22 & 44.01 & 67.36 & 55.51 & 29.61 & 17.20 & 88.60 & 48.79 & 12.89 & 17.15 & 9.12 \\
Learned, $\lambda_H{=}0$ & 4.55 & 12.89 & 39.08 & 45.30 & \textbf{67.68} & \textbf{57.32} & 28.75 & 18.20 & 88.40 & 49.25 & 14.40 & 19.35 & \textbf{13.62} \\
Learned, $\lambda_H{=}0.01$ & 4.57 & 12.96 & 38.48 & 45.30 & 67.63 & 56.86 & 28.58 & 16.60 & 87.40 & 48.69 & 14.10 & 17.72 & 11.75 \\
\cmidrule(lr){1-14}
Untied 12 & 4.18 & 11.67 & 42.11 & 49.25 & 70.13 & 60.14 & 30.80 & 19.20 & 90.50 & 51.73 & 15.92 & 21.14 & 16.00 \\
\midrule
\multicolumn{14}{l}{\textbf{L scale: (1.6B, 343.6B tokens)}} \\
Fixed $R{=}5$ & \textbf{3.51} & \textbf{8.99} & 54.55 & 61.68 & 73.39 & 69.53 & \textbf{40.27} & \textbf{25.80} & \textbf{94.10} & 59.90 & \textbf{50.64} & \textbf{36.10} & 27.12 \\
Fixed PLN-5 & 3.58 & 9.23 & 54.65 & 60.29 & 73.56 & 69.02 & 37.71 & 24.40 & 93.00 & 58.95 & 44.12 & 33.89 & 24.50 \\
Learned, $\lambda_H{=}0$ & 3.53 & 9.05 & \textbf{55.42} & \textbf{61.85} & 72.91 & 70.08 & \textbf{40.27} & 23.00 & 93.10 & 59.52 & 47.23 & 35.20 & 27.38 \\
Learned, $\lambda_H{=}0.01$ & 3.54 & 9.11 & 55.27 & 61.71 & \textbf{73.72} & \textbf{71.55} & 40.02 & 25.40 & 93.40 & \textbf{60.15} & 50.27 & 35.01 & \textbf{28.88} \\
\cmidrule(lr){1-14}
Untied 12 & 3.34 & 8.36 & 57.73 & 64.59 & 74.37 & 72.31 & 42.32 & 25.00 & 94.50 & 61.55 & 49.58 & 39.18 & 30.88 \\
\bottomrule
\CodeAfter
\begin{tikzpicture}
\foreach \sepcol in {2,4,12}{
\draw[draw=resultgrouprule,line width=0.35pt] (1-|\sepcol) -- (3-|\sepcol);
\draw[draw=resultgrouprule,line width=0.35pt] (4-|\sepcol) -- (9-|\sepcol);
\draw[draw=resultgrouprule,line width=0.35pt] (10-|\sepcol) -- (15-|\sepcol);
\draw[draw=resultgrouprule,line width=0.35pt] (16-|\sepcol) -- (21-|\sepcol);
}
\end{tikzpicture}
\end{NiceTabular}
\end{table}

\paragraph{Matched M entropy-coefficient ablation.}
\label{app:m-entropy-matched}
Figure~\ref{fig:m-entropy-ablation} varies $\lambda_H\in\{0,0.01,0.05,0.1\}$ at M scale. As $\lambda_H$ rises, the mean depth stays near five while $p_\phi(5)$ falls from $67.4\%$ to $14.0\%$ (Table~\ref{tab:m-entropy-distribution}). WikiText PPL at $R=5$ rises from $12.89$ to $13.20$, and at $R=32$, $\lambda_H=0.01$ and $0.05$ tie for the lowest PPL ($12.93$; panel (a)). Panel (b) isolates the cost of KV sharing by scoring each checkpoint with per-visit and shared KV. Sharing raises the continuation PPL on WikiText (after a $1{,}024$-token prefill, \appref{app:setup-wiki}) of fixed-depth training from $12.19$ to $16.71$, whereas PLN-5 and all learned priors retain or slightly improve it. On WikiText and LAMBADA, higher entropy does not monotonically improve sharing. Without KV reuse (Table~\ref{tab:prior-ablation-fullkv}), $\lambda_H=0$ gives slightly lower validation PPL than $\lambda_H=0.01$ at M and L ($4.55$ vs.\ $4.57$ and $3.53$ vs.\ $3.54$) and ties it at S, but with terminal KV sharing it falls behind at L (Table~\ref{tab:prior-ablation}): GSM8K $36.47$ vs.\ $47.92$.

\begin{figure}[!htbp]
\centering
\input{figures/appendix/m_entropy_table_geometry.tex}
\begin{tikzpicture}
  \node[anchor=south west,inner sep=0] (plots) at (0,0)
    {\includegraphics[width=\linewidth]{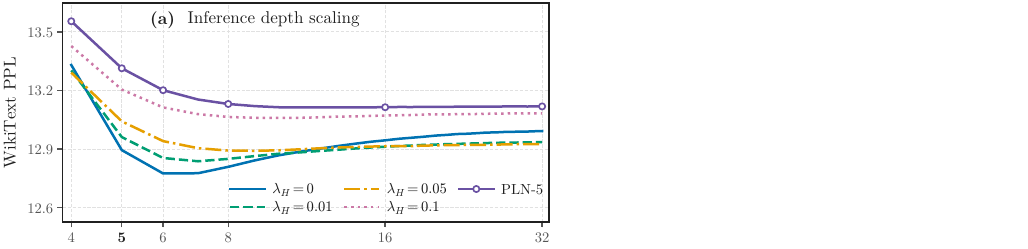}};
  \begin{scope}[x={(plots.south east)},y={(plots.north west)}]
    \node[anchor=north,inner sep=0] at (\mentropytablex,\mentropytabley) {%
      \begin{minipage}{\mentropytablewidth\linewidth}
        \centering\resulttablestyle
        {\fontsize{8.2}{10}\selectfont\textbf{(b)}\quad KV cache sharing (PPL $\downarrow$)\par}
        \vspace{3pt}
        \fontsize{7.4}{8.5}\selectfont
        \setlength{\tabcolsep}{1pt}
        \renewcommand{\arraystretch}{1.0}
        \begin{NiceTabular}{@{\hspace{3pt}}l*{6}{>{\centering\arraybackslash}m{26pt}}@{\hspace{3pt}}}
\CodeBefore
\rowcolor{papergray}{3}
\Body
\toprule
Training & \multicolumn{2}{c}{WikiText} & \multicolumn{2}{c}{MixVal} & \multicolumn{2}{c}{LAMBADA} \\
\cmidrule(lr){1-1}\cmidrule(lr){2-3}\cmidrule(lr){4-5}\cmidrule(lr){6-7}
Prior & Per-visit & Shared & Per-visit & Shared & Per-visit & Shared \\
\midrule
\multicolumn{7}{l}{\textbf{M scale: (400M, 85.9B tokens)}} \\
Fixed $R=5$ & 12.19 & 16.71 & 4.40 & 5.76 & 19.82 & 20.68\\
PLN-5 & 12.75 & 12.60 & 4.57 & 4.56 & 21.89 & 21.82\\
\midrule
$\lambda_H=0$ & 12.34 & 12.34 & 4.45 & 4.49 & 21.36 & 21.33\\
$\lambda_H=0.01$ & 12.42 & 12.37 & 4.46 & 4.48 & 22.92 & 22.93\\
$\lambda_H=0.05$ & 12.48 & 12.37 & 4.50 & 4.49 & 22.06 & 22.05\\
$\lambda_H=0.1$ & 12.64 & 12.53 & 4.54 & 4.53 & 22.27 & 22.25\\
\bottomrule
\CodeAfter
\begin{tikzpicture}
\foreach \sepcol in {2,4,6}{
\draw[draw=resultgrouprule,line width=0.35pt] (1-|\sepcol) -- (3-|\sepcol);
\draw[draw=resultgrouprule,line width=0.35pt] (4-|\sepcol) -- (10-|\sepcol);
}
\end{tikzpicture}
\end{NiceTabular}

      \end{minipage}};
  \end{scope}
\end{tikzpicture}
\input{figures/appendix/m_entropy_accuracy_note.tex}
\caption{\textbf{Entropy trades PPL at the training depth for PPL at larger depths; learned priors keep KV sharing.} M scale. (a) WikiText-103 PPL across test-time depths; the bold tick marks the training mean $R=5$. (b) PPL with per-visit KV ($12$ banks) and shared KV ($4$ banks). Panel (b) scores continuations or last words, so its values differ from panel (a).}
\label{fig:m-entropy-ablation}
\end{figure}

These M runs differ only in $\lambda_H$: they share model, data, seed, learning-rate schedule, and the controller of \appref{app:prior-algorithm} (prior LR $0.001$, $\lambda_m=1$).

Paired readouts follow \appref{app:setup-wiki} and \appref{app:setup-cache}; paired LAMBADA last-word PPL averages over the $5{,}153$ words and includes the prefill prediction. Native seven-task averages, with WinoGrande instead of SciQ (\appref{app:setup-wiki}), stay within $43.7$--$44.0$ across $\lambda_H$.

\paragraph{Paired task accuracy with and without KV sharing.}
Table~\ref{tab:m-entropy-paired-tasks} reports paired FullKV and terminal task scoring for the learned priors at $R_{\rm pre}=R_{\rm dec}=5$. Sharing does not lower the average of any learned prior (AVG $\Delta$ from $+0.25$ to $+0.70$ points), whereas Fixed $R=5$ loses $0.61$ points (FullKV and shared AVG $44.73$ and $44.12$); PLN-5 gives $44.30$ and $44.44$. Both layouts score identical requests in BF16 without FA3, and each answer option starts from a fresh cache.

Accuracy intervals use $2{,}000$ paired bootstrap draws per task, keeping task weights equal.
\begin{table}[!htbp]
\centering\resulttablestyle
\caption{\textbf{Terminal sharing does not lower the task accuracy of any learned prior.} M entropy sweep at $R_{\rm pre}=R_{\rm dec}=5$. Per-task columns give Shared minus FullKV accuracy in points; the interval is a paired, task-stratified $95\%$ bootstrap interval.}
\label{tab:m-entropy-paired-tasks}
\renewcommand{\arraystretch}{\resultrowstretch}
\begin{NiceTabular}{@{\hspace{6pt}}cccccccccccc@{\hspace{6pt}}}
\CodeBefore
\rowcolor{papergray}{3}
\Body
\toprule
Prior & \multicolumn{7}{c}{Shared $-$ FullKV (points)} & \multicolumn{4}{c}{AVG (\%)} \\
\cmidrule(lr){1-1}\cmidrule(lr){2-8}\cmidrule(lr){9-12}
$\lambda_H$ & LAMB. & HS & PIQA & ARC-E & ARC-C & WG & OBQA & FullKV & Shared & $\Delta$ & 95\% interval \\
\midrule
\multicolumn{12}{l}{\textbf{M scale: (400M, 85.9B tokens)}} \\
0 & +0.06 & +0.36 & +1.64 & +0.76 & 0.00 & +1.90 & +0.20 & 43.97 & 44.67 & +0.70 & [+0.31, +1.11] \\
0.01 & +0.02 & +0.75 & $-$0.33 & +0.93 & $-$0.51 & +0.47 & +0.40 & 43.68 & 43.93 & +0.25 & [$-$0.14, +0.62] \\
0.05 & +0.02 & +0.69 & +0.65 & +0.84 & +0.09 & +0.87 & 0.00 & 43.77 & 44.22 & +0.45 & [+0.10, +0.80] \\
0.1 & +0.04 & +0.50 & +0.82 & +0.71 & +0.26 & +0.32 & +0.60 & 43.96 & 44.42 & +0.46 & [+0.13, +0.80] \\
\bottomrule
\CodeAfter
\begin{tikzpicture}
\foreach \sepcol in {2,9}{
\draw[draw=resultgrouprule,line width=0.35pt] (1-|\sepcol) -- (3-|\sepcol);
\draw[draw=resultgrouprule,line width=0.35pt] (4-|\sepcol) -- (8-|\sepcol);
}
\end{tikzpicture}
\end{NiceTabular}
\end{table}

Table~\ref{tab:m-entropy-distribution} records final priors and realized means. The CacheEval sharing penalty decreases monotonically with entropy, from $+0.96\%$ to $-0.22\%$, and absolute shared MixVal PPL is lowest at $\lambda_H=0.01$ ($4.48$, within $0.01$ of $\lambda_H=0$ and $0.05$; Figure~\ref{fig:m-entropy-ablation}b). Intervals use $2{,}000$ paired sequence-CE bootstrap draws within each source, transformed by $100(\exp(\Delta\mathrm{CE})-1)$; they quantify evaluation uncertainty, not training-seed variability or between-coefficient differences.
\begin{table}[!htbp]
\centering\resulttablestyle
\caption{\textbf{Higher entropy spreads the learned prior and lowers the sharing penalty.} Final priors of the M entropy sweep. $\widehat{\E}[R]$ is the mean realized draw and $H(p_\phi)$ is in nats; probabilities and penalties are percentages, with paired $95\%$ bootstrap intervals.}
\label{tab:m-entropy-distribution}
\renewcommand{\arraystretch}{\resultrowstretch}
\begin{NiceTabular}{@{\hspace{6pt}}ccccccccc@{\hspace{6pt}}}
\CodeBefore
\rowcolor{papergray}{3}
\Body
\toprule
Prior & \multicolumn{6}{c}{Final depth distribution} & \multicolumn{2}{c}{Sharing penalty (\%)} \\
\cmidrule(lr){1-1}\cmidrule(lr){2-7}\cmidrule(lr){8-9}
$\lambda_H$ & $\E_{p_\phi}[R]$ & $\widehat{\E}[R]$ & $H(p_\phi)$ & $p_\phi(5)$ & $p_\phi(R>10)$ & $p_\phi(R>16)$ & Penalty & 95\% interval \\
\midrule
\multicolumn{9}{l}{\textbf{M scale: (400M, 85.9B tokens)}} \\
0 & 5.007 & 5.014 & 1.067 & 67.40 & 0.32 & 0.010 & +0.955 & [+0.786, +1.136] \\
0.01 & 5.008 & 5.033 & 1.592 & 37.79 & 0.79 & 0.042 & +0.340 & [+0.189, +0.503] \\
0.05 & 5.009 & 5.040 & 2.091 & 19.32 & 2.69 & 0.156 & $-$0.030 & [$-$0.159, +0.101] \\
0.1 & 5.016 & 5.036 & 2.293 & 13.98 & 5.02 & 0.488 & $-$0.224 & [$-$0.342, $-$0.108] \\
\bottomrule
\CodeAfter
\begin{tikzpicture}
\foreach \sepcol in {2,8}{
\draw[draw=resultgrouprule,line width=0.35pt] (1-|\sepcol) -- (3-|\sepcol);
\draw[draw=resultgrouprule,line width=0.35pt] (4-|\sepcol) -- (8-|\sepcol);
}
\end{tikzpicture}
\end{NiceTabular}
\end{table}

\paragraph{Training trajectories and depth distributions.}
Figure~\ref{fig:prior-training-loss} plots the training loss of the S and M checkpoints in Table~\ref{tab:prior-ablation}. Figures~\ref{fig:fixed-vs-pln-depth-scaling} and~\ref{fig:fixed-pln-learned-distributions} use the S/M checkpoints of Table~\ref{tab:prior-ablation}: fixed PLN-5 with independent draws and learned priors with $\lambda_H=0.01$ (LR $0.002$ at S, $0.0015$ at M). The learned M checkpoint is the $\lambda_H=0.01$ curve of Figure~\ref{fig:m-entropy-ablation}, whose PLN-5 reference is also the Table~\ref{tab:prior-ablation} run.
\begin{figure}[!htbp]
\centering
\includegraphics[width=\linewidth]{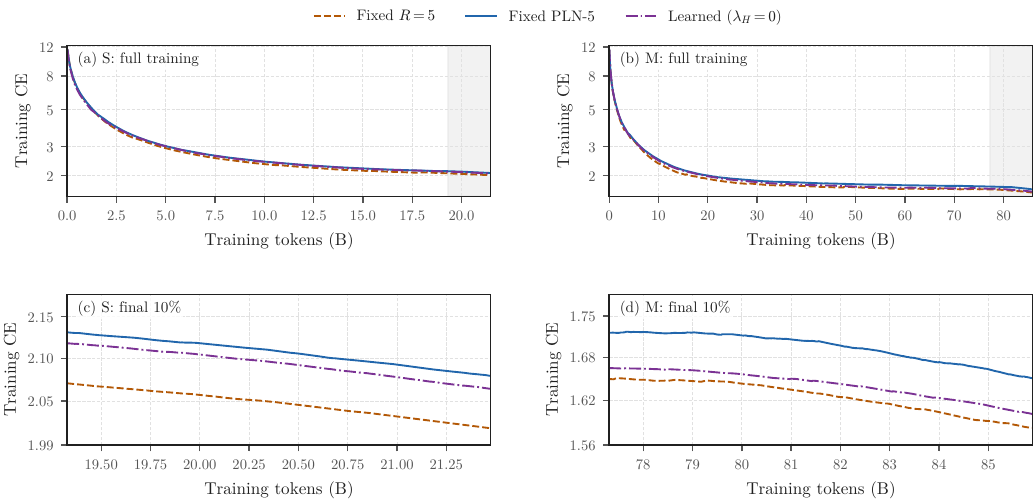}
\caption{\textbf{Training loss for the S/M checkpoints in Table~\ref{tab:prior-ablation}.}
Fixed PLN-5 draws depths independently.
S and M complete $21.5$B and $85.9$B tokens, respectively; learned priors use $\lambda_H=0$.
Curves show debiased exponential moving averages of optimizer-step CE (decay $0.999$), with linear token axes and logarithmic CE axes; the lower panels expand the shaded final $10\%$.
CE is measured at $R=5$ for fixed-depth training and averaged over sampled depths for the priors.
Learning rates and batch layouts follow \appref{app:setup-details}.}
\label{fig:prior-training-loss}
\end{figure}

\begin{figure}[!htbp]
\centering
\includegraphics[width=\linewidth]{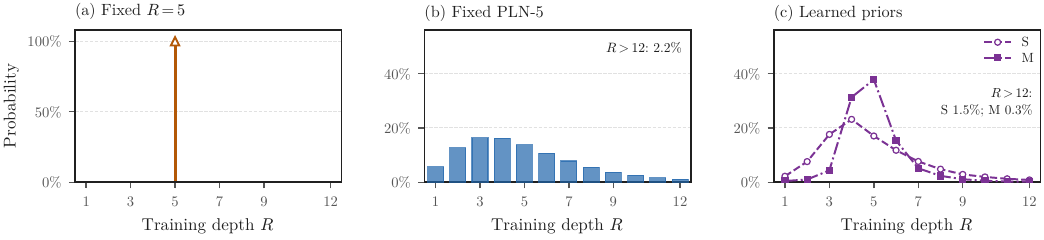}
\caption{\textbf{Depth distributions for the scaling comparison.}
Fixed $R=5$, fixed PLN-5, and learned S/M priors ($\lambda_H=0.01$) of Figure~\ref{fig:fixed-vs-pln-depth-scaling}.
Depths $1$--$12$ are shown without renormalization; annotations give the remaining mass through depth $64$.
Panels (b,c) share a probability scale.}
\label{fig:fixed-pln-learned-distributions}
\end{figure}

\paragraph{Reward and initialization ablations.}
Table~\ref{tab:prior-reward-ablation} differs from Table~\ref{tab:prior-ablation} in two ways: its S control is an earlier Fixed PLN run with independent draws, and Val.\ is native PPL on the LR-selection set (\appref{app:setup-lr}) rather than terminal-KV PPL on \textsc{MixVal}. Four-task scores use summed answer-token likelihood on full splits (\appref{app:setup-wiki}).

Cached PPL uses four terminal banks and the common CacheEval protocol at $R_{\rm pre}=R_{\rm dec}=5$. At M, the fixed PLN arm uses antithetic draws and the learned priors use independent draws; the arms otherwise share budget, model, data, seed, and optimizer.
\begin{table}[!ht]
\centering\resulttablestyle
\caption{\textbf{Reward and initialization ablations for the learned prior.} Rows name the reward signal. Learned rows use $\lambda_H=0.01$ and start from capped PLN-5 ($U_{1:k}$: uniform over depths $1,\ldots,k$); the M Fixed PLN row uses antithetic draws. FLOPs are in $10^{19}$ and $\E_{p_\phi}[R]$ is the final prior mean. Val.\ is PPL on the LR-selection set.}
\label{tab:prior-reward-ablation}
\renewcommand{\arraystretch}{\resultrowstretch}
\begin{NiceTabular}{@{\hspace{6pt}}lcccccccccc@{\hspace{6pt}}}
\CodeBefore
\rowcolor{papergray}{3,11}
\Body
\toprule
\multicolumn{3}{c}{Training} & \multicolumn{3}{c}{PPL $\downarrow$} & \multicolumn{5}{c}{Downstream (\%) $\uparrow$} \\
\cmidrule(lr){1-3}\cmidrule(lr){4-6}\cmidrule(lr){7-11}
Prior & $\E_{p_\phi}[R]$ & FLOPs & Val. & Wiki & Cached & HS & ARC-C & ARC-E & PIQA & AVG \\
\midrule
\multicolumn{11}{l}{\textbf{S scale: (100M, 21.5B tokens)}} \\
Fixed PLN & 5.00 & 7.31 & 6.86 & 21.34 & 6.52 & 29.60 & 22.15 & 47.06 & 61.15 & 39.99 \\
\midrule
CE & 5.00 & 7.33 & \textbf{6.79} & \textbf{21.17} & 6.47 & 29.46 & 21.80 & 46.89 & 61.81 & 39.99 \\
$\log(\mathrm{CE})$ & 5.01 & 7.34 & 6.85 & 21.60 & 6.51 & \textbf{29.69} & 21.20 & 46.81 & 61.75 & 39.86 \\
$\exp(\mathrm{CE})$ & 5.01 & 7.33 & 6.85 & 21.59 & 6.50 & 29.67 & 21.03 & 46.51 & 61.86 & 39.77 \\
$-\exp(-\mathrm{CE})$ & 5.00 & 7.33 & 6.80 & 21.19 & \textbf{6.46} & 29.50 & 21.72 & 47.44 & \textbf{62.35} & 40.25 \\
\midrule
CE ($U_{1:10}$) & 5.01 & 7.37 & 6.87 & 21.58 & 6.52 & 29.67 & \textbf{22.40} & \textbf{47.65} & 61.92 & \textbf{40.41} \\
CE ($U_{1:64}$) & 6.15 & 11.53 & 6.94 & 21.40 & 6.57 & 29.42 & 21.55 & 46.98 & 62.24 & 40.04 \\
\midrule
\multicolumn{11}{l}{\textbf{M scale: (400M, 85.9B tokens)}} \\
Fixed PLN & 5.00 & 91.01 & 4.72 & 13.16 & 4.49 & 36.15 & 25.49 & 58.86 & 68.66 & 47.29 \\
\midrule
CE & 5.01 & 91.99 & 4.64 & 12.85 & 4.47 & \textbf{36.98} & \textbf{28.24} & 58.22 & \textbf{69.64} & \textbf{48.27} \\
$-\exp(-\mathrm{CE})$ & 5.01 & 91.95 & \textbf{4.63} & \textbf{12.77} & \textbf{4.46} & 36.76 & 26.87 & 58.14 & 69.37 & 47.78 \\
CE ($U_{1:10}$) & 5.01 & 92.31 & 4.64 & 12.91 & 4.48 & 36.89 & 27.73 & \textbf{59.28} & 68.72 & 48.15 \\
\bottomrule
\CodeAfter
\begin{tikzpicture}
\foreach \sepcol in {4,7}{
\draw[draw=resultgrouprule,line width=0.35pt] (1-|\sepcol) -- (3-|\sepcol);
\draw[draw=resultgrouprule,line width=0.35pt] (4-|\sepcol) -- (11-|\sepcol);
\draw[draw=resultgrouprule,line width=0.35pt] (12-|\sepcol) -- (16-|\sepcol);
}
\end{tikzpicture}
\end{NiceTabular}
\end{table}

With model CE training unchanged, inverse perplexity gives the best cached PPL and task mean among the S learned-PLN arms, while raw CE and inverse perplexity tie on native and WikiText PPL (within $0.02$). At M, raw CE leads task accuracy and inverse perplexity leads WikiText and cached PPL, each costing about $1\%$ more FLOPs than fixed PLN. Log and exponential CE have higher S perplexity than either. $U_{1:10}$ leads S task accuracy but has higher PPL; it also changes both initialization and support. $U_{1:64}$ increases compute and worsens validation and cached PPL. These controls match tokens, not compute.

\paragraph{Fixed priors with the learned shape.}
\label{app:static-prior}
Table~\ref{tab:static-prior} separates the distribution the learned prior reaches from its adaptation during training. At S, the final learned prior moves mass from depths $1$--$2$ to depths $3$--$5$; the capped PLN closest to it in KL has $\sigma=0.35$. Sampling either this PLN or the frozen final distribution from the first update, at the learned prior's LR $0.002$, matches or beats the learned prior's validation PPL ($5.67$ and $5.63$ against $5.67$), with a downstream average within $1.1$ points. Fixed PLN-5 at the same LR draws the same depths as its Table~\ref{tab:prior-ablation} run and reaches $5.69$. At either LR, the learned shape lowers S validation PPL by $0.3$--$1.1\%$ relative to Fixed PLN-5. The learned prior finds a better fixed shape without a search over distributions; at S, adapting it during training adds no measurable gain. At M, the two rows of Table~\ref{tab:prior-ablation} share one LR; at L, each uses its own selected LR.

\begin{table}[!htbp]
\centering\resulttablestyle
\caption{\textbf{At S, fixed priors with the learned shape match the learned prior.} S models evaluated as in Table~\ref{tab:prior-ablation} (terminal KV sharing at $R=5$). Fixed PLN ($\sigma{=}0.35$) is the capped PLN closest to the final learned prior in KL; the frozen learned prior samples that final distribution from the first update. The first and last rows are the Table~\ref{tab:prior-ablation} runs; the others match the learned run except for the prior and LR. FLOPs are in units of $10^{19}$.}
\label{tab:static-prior}
\renewcommand{\arraystretch}{\resultrowstretch}
\begin{NiceTabular}{@{\hspace{6pt}}lccccccccccccccc@{\hspace{6pt}}}
\CodeBefore
\rowcolor{papergray}{3}
\Body
\toprule
\multicolumn{3}{c}{Training} & \multicolumn{2}{c}{PPL $\downarrow$} & \multicolumn{8}{c}{Downstream (\%) $\uparrow$} & \multicolumn{3}{c}{Generation (\%) $\uparrow$} \\
\cmidrule(lr){1-3}\cmidrule(lr){4-5}\cmidrule(lr){6-13}\cmidrule(lr){14-16}
Prior & LR & FLOPs & Val. & Wiki & LAMB. & HS & PIQA & ARC-E & ARC-C & OBQA & SciQ & AVG & GSM & DROP & MBPP+ \\
\midrule
\multicolumn{16}{l}{\textbf{S scale: (100M, 21.5B tokens)}} \\
Fixed PLN-5 ($\sigma{=}0.5$) & 0.0015 & 7.292 & 5.73 & 20.19 & 26.41 & 32.42 & 61.70 & \textbf{47.56} & 23.98 & 16.20 & 81.00 & 41.32 & 1.52 & 13.29 & 0.75 \\
Fixed PLN-5 ($\sigma{=}0.5$) & 0.002 & 7.292 & 5.69 & 20.12 & \textbf{27.52} & 32.40 & 61.21 & 46.93 & 23.98 & 16.20 & 79.90 & 41.16 & 2.43 & 14.02 & \textbf{1.00} \\
Fixed PLN ($\sigma{=}0.35$) & 0.0015 & 7.348 & 5.71 & 19.87 & 26.10 & 32.29 & 61.64 & 45.29 & 24.83 & 14.40 & 82.20 & 40.96 & 1.90 & 10.96 & 0.38 \\
Fixed PLN ($\sigma{=}0.35$) & 0.002 & 7.348 & 5.67 & \textbf{19.75} & 26.57 & 32.68 & 61.70 & 46.09 & 24.74 & 16.60 & 80.30 & 41.24 & 2.20 & 11.03 & 0.50 \\
Frozen learned prior & 0.0015 & 7.334 & 5.71 & 20.25 & 26.64 & \textbf{32.84} & 61.32 & 46.25 & 24.83 & 16.40 & \textbf{82.30} & 41.51 & 2.05 & 11.13 & 0.38 \\
Frozen learned prior & 0.002 & 7.334 & \textbf{5.63} & 19.92 & 27.21 & 32.54 & 61.53 & 44.23 & 24.74 & 15.60 & 78.40 & 40.61 & \textbf{2.65} & 13.15 & 0.12 \\
\midrule
Learned, $\lambda_H{=}0.01$ & 0.002 & 7.328 & 5.67 & 19.91 & 27.36 & 32.51 & \textbf{63.00} & 45.50 & \textbf{25.09} & \textbf{17.20} & 80.90 & \textbf{41.65} & 1.59 & \textbf{14.08} & 0.62 \\
\bottomrule
\end{NiceTabular}
\end{table}

\subsection{Input Injection}
\label{app:architecture-results}
Models, optimization, batch layouts, and token budgets follow \appref{app:setup-details}, and \appref{app:setup-lr} gives the learning-rate grids. Equal token budgets need not mean equal compute.

\paragraph{Injection site and KV sharing.}
DEQ-QKV injects the input into queries, keys, and values. It breaks terminal KV sharing: decoding CacheEval one token at a time (\appref{app:setup-cache}), its PPL is $16.52$, against $6.43$--$6.69$ for the other S schemes and $7.13$ on \textsc{MixVal} without KV reuse (Table~\ref{tab:architecture-fullkv}). At S, generation is near the floor for every scheme (GSM8K at most $2.05\%$, MBPP+ at most $1.62\%$; Table~\ref{tab:architecture}).

\begin{table}[!ht]
\centering\resulttablestyle
\caption{\textbf{Table~\ref{tab:architecture} without KV reuse.} Same checkpoints as Table~\ref{tab:architecture}, evaluated at $R=5$. Val.\ and Wiki are native PPL over full sequences, choice tasks are scored without a cache, and generation keeps all twelve visit-specific KV banks.}
\label{tab:architecture-fullkv}
\fontsize{9}{10.8}\selectfont\renewcommand{\arraystretch}{\resultrowstretch}
\begin{NiceTabular}{@{\hspace{6pt}}lccccccccccccc@{\hspace{6pt}}}
\CodeBefore
\rowcolor{papergray}{3,8,12}
\Body
\toprule
Model & \multicolumn{2}{c}{PPL $\downarrow$} & \multicolumn{8}{c}{Downstream (\%) $\uparrow$} & \multicolumn{3}{c}{Generation (\%) $\uparrow$} \\
\cmidrule(lr){1-1}\cmidrule(lr){2-3}\cmidrule(lr){4-11}\cmidrule(lr){12-14}
Injection & Val. & Wiki & LAMB. & HS & PIQA & ARC-E & ARC-C & OBQA & SciQ & AVG & GSM8K & DROP & MBPP+ \\
\midrule
\multicolumn{14}{l}{\textbf{S scale: (100M, 21.5B tokens)}} \\
DEQ-QKV & 7.13 & 23.83 & 25.29 & 30.65 & 60.99 & 44.74 & 24.32 & 14.20 & 78.60 & 39.83 & 1.36 & \textbf{14.61} & 0.50 \\
Huginn-Linear & 6.89 & 22.23 & 25.91 & 31.40 & \textbf{62.19} & 43.35 & 24.06 & \textbf{17.00} & 79.10 & 40.43 & 1.67 & 11.84 & 0.50 \\
Parcae-Decay & 6.71 & 21.24 & 26.02 & 32.20 & 61.75 & \textbf{47.26} & 24.40 & 16.40 & 81.40 & 41.35 & \textbf{2.35} & 10.94 & 0.62 \\
\midrule
OrthoInj & \textbf{6.60} & \textbf{21.16} & \textbf{27.30} & \textbf{32.38} & 61.92 & 45.45 & \textbf{25.34} & 16.60 & \textbf{82.10} & \textbf{41.59} & 1.36 & 13.34 & \textbf{1.00} \\
\midrule
\multicolumn{14}{l}{\textbf{M scale: (400M, 85.9B tokens)}} \\
Huginn-Linear & 4.65 & 13.33 & 36.33 & 43.90 & 67.85 & 56.94 & \textbf{29.18} & 17.40 & \textbf{89.10} & 48.67 & \textbf{15.09} & 16.82 & 8.38 \\
Parcae-Decay & 4.62 & 13.20 & 37.28 & 44.37 & \textbf{68.06} & \textbf{57.62} & 27.82 & \textbf{19.20} & 88.10 & 48.92 & 12.59 & \textbf{19.44} & 9.00 \\
\midrule
OrthoInj & \textbf{4.58} & \textbf{13.06} & \textbf{37.49} & \textbf{44.58} & 67.57 & 56.82 & \textbf{29.18} & 18.20 & 88.70 & \textbf{48.94} & 13.57 & 19.38 & \textbf{10.75} \\
\midrule
\multicolumn{14}{l}{\textbf{L scale: (1.6B, 343.6B tokens)}} \\
Parcae-Decay & 3.58 & 9.22 & 53.19 & 60.32 & \textbf{73.12} & \textbf{68.77} & \textbf{38.82} & 23.40 & \textbf{93.60} & 58.75 & \textbf{47.38} & \textbf{33.72} & \textbf{31.25} \\
\midrule
 OrthoInj & \textbf{3.57} & \textbf{9.21} & \textbf{54.22} & \textbf{60.76} & 72.58 & 68.52 & 37.20 & \textbf{24.60} & 93.50 & \textbf{58.77} & 47.01 & 31.73 & 27.25 \\
\bottomrule
\CodeAfter
\begin{tikzpicture}
\foreach \sepcol in {2,4,12}{
\draw[draw=resultgrouprule,line width=0.35pt] (1-|\sepcol) -- (3-|\sepcol);
\draw[draw=resultgrouprule,line width=0.35pt] (4-|\sepcol) -- (8-|\sepcol);
\draw[draw=resultgrouprule,line width=0.35pt] (9-|\sepcol) -- (12-|\sepcol);
\draw[draw=resultgrouprule,line width=0.35pt] (13-|\sepcol) -- (15-|\sepcol);
}
\end{tikzpicture}
\end{NiceTabular}
\end{table}

\paragraph{Initialization and evaluation.}
Native validation CE at S selects the standard OrthoInj initialization of \appref{app:architecture-variants}. M OrthoInj uses the same initialization and one complete $20{,}480$-update checkpoint for every score.
Table~\ref{tab:architecture} follows the terminal-KV protocol of \appref{app:setup-terminal}, and Table~\ref{tab:architecture-fullkv} its no-reuse protocol; generation is excluded from AVG.
We keep weight decay at $0.1$: in a single-seed S Parcae-Decay control (LR $0.0015$, $5{,}120$ updates), zero decay worsens validation and cached PPL from $6.63$ and $6.45$ to $6.76$ and $6.58$.

\paragraph{From Parcae-Decay to OrthoInj.}
OrthoInj differs from Parcae-Decay by two switches: it projects the carryover against the injected vector $q$, and it drops prelude RMSNorm. Table~\ref{tab:orthoinj-normalization-init} crosses the two switches at S. Each switch lowers validation PPL on its own: the projection from $6.71$ to $6.67$, and dropping prelude RMSNorm from $6.71$ to $6.63$. Their gains add up: the projection lowers PPL by a similar amount with and without prelude RMSNorm ($0.04$ and $0.03$), and OrthoInj, which has both switches, reaches $6.60$. OrthoInj also has the highest downstream average; the other three cells lie within $0.03$ points of each other.

All four cells train with PLN-5 and independent depth draws, seed $1$, common data, optimizer, and batch settings, $\Delta_{\rm init}=\tfrac12\log5$, and $5{,}120$ updates; transformer-block RMSNorms are retained throughout. Parcae-Decay and OrthoInj are the S models of Table~\ref{tab:architecture}. Each cell uses the LR selected by its own staged screen. The cell without either switch was screened with independent draws over $\{1,1.5,2,3\}\times10^{-3}$ and selects $2\times10^{-3}$; the other three screens used antithetic draws. Setting $\Delta_{\rm init}=1$ instead changes selection PPL by at most $0.03$ ($6.87$ vs.\ $6.84$ with prelude RMSNorm and the projection; $6.77$ vs.\ $6.78$ for OrthoInj).

\begin{table}[!htbp]
\sidefloat[22.9pt]{0.59\linewidth}{table}{tables/appendix/orthoinj_switch_2x2.tex}\hfill
\sidefloat{0.38\linewidth}{figure}{figures/appendix/injection_switch_loss_figure.tex}
\end{table}

Figure~\ref{fig:injection-switch-loss} shows the late training loss of the four cells. Over the final $256$ updates, dropping prelude RMSNorm lowers mean training CE by $0.010$ with or without the projection, and the projection lowers it by a further $0.002$--$0.003$. These single-seed comparisons do not estimate seed variance.

\paragraph{Training loss.}
\label{app:injection-training-loss}
OrthoInj ends training with the lowest training CE at every scale (Figure~\ref{fig:injection-training-loss}): over the final $512$ updates, its mean is $0.011$ (S), $0.005$ (M), and $0.003$ (L) below Parcae-Decay's. At S, each run uses the peak learning rate selected by its own screen (\appref{app:setup-lr}): $2\times10^{-3}$ for OrthoInj and $1.5\times10^{-3}$ for Parcae-Decay. At M, both use $10^{-3}$, Parcae-Decay's selected rate, and OrthoInj's loss stays below that of both baselines over the whole plotted range. At L, both runs use $5\times10^{-4}$, and OrthoInj stays below Parcae-Decay over the whole plotted range.

\paragraph{At either M learning rate, OrthoInj has the lower training loss and PPL.}
\label{app:injection-lr-m}
OrthoInj's own M screen selects $1.5\times10^{-3}$, so we also train both methods at this rate (Table~\ref{tab:injection-lr-m}). At $1.5\times10^{-3}$, OrthoInj ends training $0.014$ below Parcae-Decay, and Parcae-Decay ends $0.010$ above its $10^{-3}$ run. OrthoInj keeps the lower validation and WikiText PPL at both rates. The downstream average depends on the rate: OrthoInj leads by $0.13$ points at $10^{-3}$ and trails by $0.51$ at $1.5\times10^{-3}$, where it still leads GSM8K and DROP.

\begin{table}[!htbp]
\centering\resulttablestyle
\caption{\textbf{At either M learning rate, OrthoInj has the lower training loss and PPL.} Train CE is the mean training loss over the final $512$ updates; the other columns follow Table~\ref{tab:architecture} (terminal KV sharing, $R=5$). The Parcae-Decay run at $1.5\times10^{-3}$ uses a different data-parallel layout: its global batch and sampled depths match the other runs, its data order does not.}
\label{tab:injection-lr-m}
\renewcommand{\arraystretch}{\resultrowstretch}
\begin{NiceTabular}{@{\hspace{6pt}}lccccccc@{\hspace{6pt}}}
\CodeBefore
\rowcolor{papergray}{2,5}
\Body
\toprule
Injection & Train CE $\downarrow$ & Val.\ $\downarrow$ & Wiki $\downarrow$ & AVG $\uparrow$ & GSM $\uparrow$ & DROP $\uparrow$ & MBPP+ $\uparrow$ \\
\midrule
\multicolumn{8}{l}{\textbf{Peak learning rate $10^{-3}$ (Table~\ref{tab:architecture})}} \\
Parcae-Decay & 1.602 & 3.96 & 12.45 & 49.03 & \textbf{15.01} & \textbf{20.17} & \textbf{8.25} \\
OrthoInj & \textbf{1.596} & \textbf{3.94} & \textbf{12.32} & \textbf{49.16} & 13.95 & 19.38 & 7.88 \\
\midrule
\multicolumn{8}{l}{\textbf{Peak learning rate $1.5\times10^{-3}$ (OrthoInj's selected rate)}} \\
Parcae-Decay & 1.611 & 3.96 & 12.41 & \textbf{49.80} & 14.48 & 18.63 & \textbf{11.62} \\
OrthoInj & \textbf{1.598} & \textbf{3.94} & \textbf{12.32} & 49.29 & \textbf{17.51} & \textbf{21.04} & 10.25 \\
\bottomrule
\end{NiceTabular}
\end{table}

\paragraph{Carryover along the injected input.}
Table~\ref{tab:injection-alignment} measures the component of the carryover $\Lambda h_t^r$ along $q_t$ that motivates OrthoInj (\secref{sec:architecture}). In Parcae-Decay this component is large and positive: at the fixed point, the carryover has cosine $0.81$ (S), $0.87$ (M), and $0.91$ (L) with $q_t$, so the input enters the core at $2.2\times$ (S), $2.8\times$ (M), and $3.4\times$ (L) its injected size. The factor varies across tokens, from $1.9$ to $2.5$ at S, from $2.3$ to $3.3$ at M, and from $2.6$ to $4.2$ at L between the $5$th and $95$th percentiles, and no token's carryover cancels half of the input. Before its projection, the OrthoInj carryover also leans toward $q_t$ (cosine $0.42$--$0.48$); the projection removes this component and fixes the factor at one.

\begin{table}[!htbp]
\centering\resulttablestyle
\caption{\textbf{In Parcae-Decay, the carryover amplifies the injected input.} Carryover $\Lambda h_t^r$ along the injected input $q_t$ at the fixed point of the Table~\ref{tab:architecture} checkpoints: recurrence $32$ of a depth-$32$ forward on $16$ \textsc{MixVal} sequences ($131{,}072$ tokens) with attention kept within documents; recurrence $16$ agrees to two decimals. Along $q$ is the mean of $\langle\Lambda h_t^r,q_t\rangle/\norm{q_t}_2^2$, so the input enters the core as $(1+\text{along }q)\,q_t$; gain gives its median and $5$--$95\%$ range over tokens. OrthoInj values precede its projection, which fixes the gain at one.}
\label{tab:injection-alignment}
\renewcommand{\arraystretch}{\resultrowstretch}
\begin{NiceTabular}{@{\hspace{6pt}}lcccc@{\hspace{6pt}}}
\CodeBefore
\rowcolor{papergray}{2,5,8}
\Body
\toprule
Injection & $\cos(\Lambda h_t^r,q_t)$ & Along $q$ & Gain & Gain $5$--$95\%$ \\
\midrule
\multicolumn{5}{l}{\textbf{S scale: (100M, 21.5B tokens)}} \\
Parcae-Decay & 0.81 & 1.19 & 2.18 & 1.91--2.49 \\
OrthoInj (before projection) & 0.44 & 0.59 & 1 & 1 \\
\midrule
\multicolumn{5}{l}{\textbf{M scale: (400M, 85.9B tokens)}} \\
Parcae-Decay & 0.87 & 1.76 & 2.75 & 2.29--3.26 \\
OrthoInj (before projection) & 0.42 & 0.62 & 1 & 1 \\
\midrule
\multicolumn{5}{l}{\textbf{L scale: (1.6B, 343.6B tokens)}} \\
Parcae-Decay & 0.91 & 2.43 & 3.43 & 2.60--4.24 \\
OrthoInj (before projection) & 0.48 & 0.72 & 1 & 1 \\
\bottomrule
\end{NiceTabular}
\end{table}

\subsection{Truncation, Depth Sampling, and Fixed Points}
\label{app:tbptt-study}
\label{app:gradient-results}\label{app:gradient-measurements}
\paragraph{Cost.} Training memory grows linearly with the window: peak allocation is $48$, $86$ and $124$\,GB for $b=5,10,15$ at local batch $8$. Full BPTT under PLN-5 must also hold the deepest draw; a depth-$30$ draw at local batch 8 exceeds device memory, and per-layer recomputation brings it to $28$\,GiB at about $1.3\times$ the step time of $b=10$. At a fixed depth $R=5$, $b=2$ trains $1.32\times$ faster than full BPTT. In a Small PLN-25 model at $R=16$, windows $b=5$ and $b=10$ lower saved activations from $3.95$ to $1.52$ and $2.63$\,GiB, which stay constant for fixed $b$ once $R\geq b$, while the core gradient at $b=10$ keeps a median cosine similarity of $0.978$ with full BPTT.

\paragraph{With depth sampling, windows of four or more recurrences match full BPTT.} Table~\ref{tab:tbptt-pln5-windows} varies only the backward window under IID PLN-5 draws. At M, every window $b\geq4$ lies within $1\%$ of full BPTT in WikiText perplexity at $R=5$, on either side of it, while shorter windows lose $1.6\%$ ($b=3$), $5.5\%$ ($b=2$) and $15\%$ ($b=1$). At S, $b=5$ and $b=10$ lie $0.8\%$ and $1.5\%$ above full BPTT. For every window, including $b=1$, perplexity at $R=64$ is below that at $R=5$ and terminal-KV sharing is lossless, so truncation lowers quality without preventing fixed points.

\paragraph{Truncation as a fixed-point shaper.} Table~\ref{tab:tbptt-fixed-r5} fixes the training depth, so depth sampling cannot shape fixed points. Full BPTT then degrades beyond the training depth and loses $21\%$ perplexity under terminal-KV sharing. Backpropagating through only the last two recurrences removes both, with a random or a zero initial state, at a cost of about $0.07$--$0.10$ CE at the training depth. The stop-gradient removes credit for iteration-specific computation in the detached prefix: the window is trained only to map whatever state the model reaches to a good prediction, which favors states that further recurrences leave unchanged. Without input injection, the prefix state is the only carrier of the input, so this pressure erases it: truncation then neither fits the training depth nor yields a fixed point. A two-term endpoint Neumann estimate, which uses only the Jacobian at the endpoint, collapses from initialization at every LR, because no fixed point exists early in training.

\paragraph{Relation to prior uses of truncation.} Earlier looped models that extrapolate stably also truncate backpropagation, but in each case together with another source of depth variation. Progressive training detaches a prefix of random length, so the window lands at a random depth, and mixes in a full-depth loss~\citep{bansal2022algorithm}; its motivation, discouraging iteration-specific behavior, matches the mechanism above. The looped transformers of \citet{yang2024looped} detach all but the last $b$ loops, but grow the loop count by a curriculum and supervise every loop in the window; they find that larger windows help reach a fixed point. Huginn backpropagates through the last eight recurrences of a sampled depth~\citep{huginn}. Both \citet{bansal2022algorithm} and \citet{yang2024looped} also find input injection necessary. Table~\ref{tab:tbptt-fixed-r5} removes these sources of depth variation: the depth is fixed, the window is the last two recurrences, and full BPTT is the control.

\begin{table}[!t]
\centering\resulttablestyle
\caption{\textbf{With depth sampling, windows of four or more recurrences match full BPTT.} Models trained with IID PLN-5 draws. S rows share one LR. Each M window uses the LR selected from $\{0.001,0.0015,0.002,0.003\}$ by a staged sweep ($2{,}048$, $4{,}096$ and $20{,}480$ steps); full BPTT is shown at two LRs. Terminal KV gives the change in cached PPL from full to terminal KV at $R_{\rm pre}=R_{\rm dec}=5$.}
\label{tab:tbptt-pln5-windows}
\renewcommand{\arraystretch}{\resultrowstretch}
\begin{NiceTabular}{@{\hspace{6pt}}lcccc@{\hspace{6pt}}}
\CodeBefore
\rowcolor{papergray}{3,7}
\Body
\toprule
\multicolumn{2}{c}{Training} & \multicolumn{2}{c}{WikiText PPL $\downarrow$} & Terminal KV \\
\cmidrule(lr){1-2}\cmidrule(lr){3-4}\cmidrule(lr){5-5}
Window $b$ & LR & $R{=}5$ & $R{=}64$ & $\Delta$ (\%) \\
\midrule
\multicolumn{5}{l}{\textbf{S scale: (100M, 21.5B tokens)}} \\
5 & 0.0015 & 21.32 & 21.09 & $-0.15$ \\
10 & 0.0015 & 21.47 & 21.21 & $-0.16$ \\
full (64) & 0.0015 & 21.15 & 20.94 & $-0.09$ \\
\midrule
\multicolumn{5}{l}{\textbf{M scale: (400M, 85.9B tokens)}} \\
1 & 0.001 & 15.23 & 15.12 & $-0.15$ \\
2 & 0.001 & 13.97 & 13.85 & $-0.06$ \\
3 & 0.001 & 13.46 & 13.30 & $-0.29$ \\
4 & 0.001 & 13.33 & 13.19 & $-0.15$ \\
5 & 0.001 & 13.37 & 13.18 & $-0.24$ \\
6 & 0.0015 & 13.17 & 13.00 & $-0.25$ \\
7 & 0.001 & 13.32 & 13.16 & $-0.28$ \\
8 & 0.001 & 13.32 & 13.15 & $-0.31$ \\
9 & 0.001 & 13.34 & 13.17 & $-0.33$ \\
10 & 0.001 & 13.21 & 13.05 & $-0.22$ \\
full (64) & 0.001 & 13.25 & 13.09 & $-0.21$ \\
full (64) & 0.0015 & 13.27 & 13.10 & $-0.34$ \\
\bottomrule
\CodeAfter
\begin{tikzpicture}
\foreach \sepcol in {3,5}{
\draw[draw=resultgrouprule,line width=0.35pt] (1-|\sepcol) -- (3-|\sepcol);
\draw[draw=resultgrouprule,line width=0.35pt] (4-|\sepcol) -- (7-|\sepcol);
\draw[draw=resultgrouprule,line width=0.35pt] (8-|\sepcol) -- (20-|\sepcol);
}
\end{tikzpicture}
\end{NiceTabular}
\end{table}

\begin{table}[!t]
\centering\resulttablestyle
\caption{\textbf{At a fixed training depth, truncation shapes fixed points when the input is injected.} S models trained at fixed $R=5$ with diagonal (Diag.) or no injection, each with its own screened LR. CE is native validation CE at $R=5$. Without injection, the loop starts from the prelude output, with embeddings and prelude frozen. The two-term endpoint Neumann estimate collapsed at all five LRs within $1{,}024$ steps.}
\label{tab:tbptt-fixed-r5}
\renewcommand{\arraystretch}{\resultrowstretch}
\begin{NiceTabular}{@{\hspace{6pt}}lllcccc@{\hspace{6pt}}}
\CodeBefore
\rowcolor{papergray}{3}
\Body
\toprule
\multicolumn{3}{c}{Training} & Val. & \multicolumn{2}{c}{WikiText PPL $\downarrow$} & Terminal KV \\
\cmidrule(lr){1-3}\cmidrule(lr){4-4}\cmidrule(lr){5-6}\cmidrule(lr){7-7}
Init. & Injection & Backward & CE $\downarrow$ & $R{=}5$ & $R{=}64$ & $\Delta$ (\%) \\
\midrule
\multicolumn{7}{l}{\textbf{S scale: (100M, 21.5B tokens)}} \\
Random & Diag.\ & Full ($b{=}5$) & 1.894 & 20.62 & 26.09 & $+21.9$ \\
Random & Diag.\ & TBPTT $b{=}2$ & 1.961 & 22.29 & 22.71 & $+2.9$ \\
Zero & Diag.\ & Full ($b{=}5$) & 1.886 & 20.76 & 26.06 & $+20.7$ \\
Zero & Diag.\ & TBPTT $b{=}2$ & 1.984 & 23.01 & 23.26 & $+1.3$ \\
Random & Diag.\ & Neumann, $b{=}2$ & \multicolumn{4}{c}{collapsed at every LR} \\
\midrule
Prelude & None & Full ($b{=}5$) & 2.345 & 39.85 & 9349 & $+232$ \\
Prelude & None & TBPTT $b{=}2$ & 4.971 & 1046 & 37860 & $+1890$ \\
\bottomrule
\CodeAfter
\begin{tikzpicture}
\foreach \sepcol in {4,5,7}{
\draw[draw=resultgrouprule,line width=0.35pt] (1-|\sepcol) -- (3-|\sepcol);
\draw[draw=resultgrouprule,line width=0.35pt] (4-|\sepcol) -- (11-|\sepcol);
}
\end{tikzpicture}
\end{NiceTabular}
\end{table}

\subsection{Recurrent-State and Cache Dynamics}
\label{app:dynamics-results}
This section gives the protocol behind the tokenwise convergence of \secref{sec:fixed-points} and measures how far hidden states and KV caches settle at larger depth.
\paragraph{Tokenwise dynamics of trained checkpoints.}
\label{app:huginn-prefix-probe}

Figure~\ref{fig:huginn-prefix-stabilization} uses the Small Fixed PLN-5 checkpoint of Table~\ref{tab:prior-ablation} at step $5{,}120$. We take eight frozen samples, one from each of the first eight validation sources in alphabetical order; each sample is the first $2{,}048$ tokens of the source's first packed sequence, with document boundaries and causal masks preserved. The unsmoothed heatmap in panel (a) shows the first document, which ends at position $1{,}003$; panel (c) pools all eight samples at document-local positions below $1{,}024$.

\paragraph{State and convergence.}
Each continuous $64$-update trajectory uses native normal initialization, executes the prelude and coda once, and generates no tokens. The relative-change diagnostic in \secref{sec:fixed-points} uses states after both core blocks and exit normalization, before the coda, with nonzero reference states and FP32 measurements on BF16 states. Under this criterion, Figure~\ref{fig:huginn-prefix-stabilization} sets $R_{\rm obs}=32$ and $\tau=2\%$, requiring a sub-threshold suffix of at least four updates; Table~\ref{tab:huginn-convergence} sets $R_{\rm obs}=64$. Tokens that do not meet the criterion have no observed convergence depth and leave display gaps. The logarithmic color scale is clipped to $[10^{-4},1]$, and the gray initialization row is unclipped.

\paragraph{Positional association.}
The zoomed panel (b) shows the latest-converging token. Panel (c) includes all qualifying positions, with medians in nonoverlapping $128$-position bins: $13{,}707$ positions qualify and two do not. For the correlation in \secref{sec:fixed-points}, ties use average ranks; per-sequence values range from $-0.10$ to $0.47$. No independent-token significance test is applied.

\paragraph{Huginn convergence measurements.}
\label{app:huginn-convergence}
\begin{table}[!htbp]
\centering\resulttablestyle\appendixtablesfalse
\caption{\textbf{Hidden states and KV caches settle with depth.} S models trained with PLN-5 and PLN-25. \textbf{A:} a token is converged if every later update through recurrence $64$ stays below the threshold for at least four updates; onset is the median over converged tokens. \textbf{B:} relative pre-RoPE KV change between recurrences $r-5$ and $r$ for PLN-25, with sequence-bootstrap intervals. Values are percentages except onset, which counts recurrences.}
\label{tab:huginn-convergence}
\begin{NiceTabular}{@{}rrrrr@{}}
\toprule
\tableband{5}{A: Adjacent hidden-state changes and converged-token coverage}
Mean depth & Median $\rho_{64,t}$ & Converged at 1\% (\%) & Converged at 2\% (\%) & Onset at 2\% \\
\midrule
5 & 0.33 & 99.74 & 99.93 & 7 \\
25 & 0.23 & 97.39 & 99.09 & 11 \\
\bottomrule
\end{NiceTabular}
\par\vspace{8pt}
\begin{NiceTabular}{@{}rrrrr@{}}
\toprule
\tableband{5}{B: KV changes over five recurrences; PLN-25}
Recurrence $r$ & Reference $r-5$ & Median KV change & 95\% interval & KV norm / norm at $r=5$ \\
\midrule
25 & 20 & 1.23 & [1.22, 1.24] & 97.48 \\
30 & 25 & 0.85 & [0.84, 0.86] & 97.48 \\
160 & 155 & 0.59 & [0.58, 0.60] & 97.48 \\
\bottomrule
\end{NiceTabular}
\end{table}

\paragraph{Adjacent hidden-state changes.}
Table~\ref{tab:huginn-convergence}A applies the initialization, precision and convergence protocol in \appref{app:huginn-prefix-probe} through update $64$ over all $16{,}384$ positions; onset medians are computed over qualifying tokens only. The mean-$5$ and mean-$25$ models use LRs of $1.5\times10^{-3}$ and $2\times10^{-3}$, respectively, and also differ in training recipe, so the two rows are not a controlled comparison of depth. Within $64$ updates, more than $97\%$ of tokens meet the $1\%$ criterion in both models.

\paragraph{Five-recurrence KV changes.}
Panel B stacks the pre-RoPE keys and values of the banks $C^r$ from both physical core layers into one matrix every fifth recurrence. On each of $1{,}216$ \textsc{MixVal} sequences (\appref{app:setup-mixture}), we compute
\begin{equation}
 \rho^{C}_r=\frac{\norm{C^r-C^{r-5}}_F}
 {\sqrt{\tfrac12(\norm{C^r}_F^2+\norm{C^{r-5}}_F^2)}}.
 \label{eq:huginn-kv-five-step}
\end{equation}
The table reports sequence medians with pointwise $95\%$ intervals from $10{,}000$ sequence-bootstrap resamples; norm ratios use each sequence's $\norm{C^5}_F$ as the reference. The five-recurrence KV change falls from $1.23\%$ at $r=25$ to $0.59\%$ at $r=160$, below the $2\%$ single-update threshold of \secref{sec:fixed-points} even though it spans five recurrences, while the KV norm stays at $97.48\%$ of its $r=5$ value, so the remaining change occurs at nearly constant norm. Each continuous $160$-recurrence trajectory uses two core-block applications per recurrence without token generation. These measurements characterize state and cache changes; they do not establish global contraction, a causal training-prior effect, or the quality of residual-based stopping.

\subsection{Terminal KV Sharing at Deeper Recurrence}
\label{app:deep-r-kv}
Table~\ref{tab:deep-r-kv} compares terminal KV sharing with the full per-visit cache on the Table~\ref{tab:prior-ablation} checkpoints at $R=5$, $16$, and $32$, where the full cache holds $12$, $34$, and $66$ banks and sharing keeps $4$. Sharing raises \textsc{MixVal} PPL by at most $0.85\%$ at $R=5$ and by at most $0.11\%$ at $R=16$ and $32$, so its cost shrinks as recurrence deepens, consistent with Lemma~\ref{lem:joint-frozen}. With sharing, GSM8K accuracy lies within $3.2$ points of the full cache, with no consistent sign. Both caches show the same depth trend: Fixed PLN-5 improves from $R=5$ to $R=16$ at every scale, whereas the PPL of the L learned prior rises beyond $R=5$ while its GSM8K accuracy grows.

\begin{table}[!htbp]
\centering\resulttablestyle
\caption{\textbf{Terminal KV sharing stays close to the full cache as recurrence deepens.} Table~\ref{tab:prior-ablation} checkpoints evaluated at depth $R$ with the full per-visit cache ($2+2R$ banks) or terminal KV sharing ($4$ banks), under the protocol of Table~\ref{tab:prior-ablation}: Val.\ is \textsc{MixVal} PPL on tokens that continue a document begun in the $4{,}096$-token prefill, and GSM8K decodes with the given cache. $\Delta$: relative change in Val.\ PPL from full to terminal KV.}
\label{tab:deep-r-kv}
\renewcommand{\arraystretch}{\resultrowstretch}
\begin{NiceTabular}{@{\hspace{6pt}}lccccccc@{\hspace{6pt}}}
\CodeBefore
\rowcolor{papergray}{3,10,17}
\Body
\toprule
 & & Full & \multicolumn{3}{c}{Val.\ PPL $\downarrow$} & \multicolumn{2}{c}{GSM8K (\%) $\uparrow$} \\
\cmidrule(lr){4-6}\cmidrule(lr){7-8}
Prior & $R$ & banks & Full & Terminal & $\Delta$ (\%) & Full & Terminal \\
\midrule
\multicolumn{8}{l}{\textbf{S scale: (100M, 21.5B tokens)}} \\
Fixed PLN-5 & 5 & 12 & 5.733 & 5.734 & $+0.01$ & 2.20 & 1.52 \\
 & 16 & 34 & 5.719 & 5.719 & $+0.01$ & 1.52 & 1.59 \\
 & 32 & 66 & 5.720 & 5.720 & $+0.01$ & 1.29 & 1.36 \\
Learned, $\lambda_H{=}0.01$ & 5 & 12 & 5.666 & 5.670 & $+0.08$ & 1.90 & 1.59 \\
 & 16 & 34 & 5.661 & 5.661 & $+0.01$ & 2.20 & 2.27 \\
 & 32 & 66 & 5.663 & 5.664 & $+0.01$ & 1.97 & 2.20 \\
\midrule
\multicolumn{8}{l}{\textbf{M scale: (400M, 85.9B tokens)}} \\
Fixed PLN-5 & 5 & 12 & 3.996 & 3.997 & $+0.04$ & 12.89 & 13.72 \\
 & 16 & 34 & 3.979 & 3.979 & $+0.01$ & 14.78 & 14.78 \\
 & 32 & 66 & 3.981 & 3.981 & $+0.00$ & 14.40 & 14.71 \\
Learned, $\lambda_H{=}0.01$ & 5 & 12 & 3.917 & 3.926 & $+0.23$ & 14.86 & 15.62 \\
 & 16 & 34 & 3.929 & 3.930 & $+0.02$ & 16.68 & 16.98 \\
 & 32 & 66 & 3.935 & 3.936 & $+0.02$ & 16.22 & 17.13 \\
\midrule
\multicolumn{8}{l}{\textbf{L scale: (1.6B, 343.6B tokens)}} \\
Fixed PLN-5 & 5 & 12 & 3.092 & 3.095 & $+0.10$ & 44.43 & 47.61 \\
 & 16 & 34 & 3.078 & 3.078 & $+0.01$ & 50.42 & 50.11 \\
 & 32 & 66 & 3.080 & 3.080 & $+0.00$ & 50.42 & 51.18 \\
Learned, $\lambda_H{=}0.01$ & 5 & 12 & 3.047 & 3.073 & $+0.85$ & 50.42 & 47.92 \\
 & 16 & 34 & 3.137 & 3.141 & $+0.11$ & 48.75 & 50.42 \\
 & 32 & 66 & 3.158 & 3.159 & $+0.04$ & 52.24 & 51.63 \\
\bottomrule
\end{NiceTabular}
\end{table}

{At M scale, Table~\ref{tab:deep-r-kv-m} adds the full downstream battery at $R=5$ and $16$: sharing moves the seven-task average by at most $0.34$ points at $R=5$ and $0.04$ points at $R=16$.}

\begin{table}[!htbp]
\centering\resulttablestyle
\caption{\textbf{Terminal KV sharing matches the full cache at $R=5$ and $R=16$.} M-scale ($400$M) checkpoints of Table~\ref{tab:prior-ablation} evaluated at depth $R$ with the full per-visit cache ($2+2R$ banks) or terminal KV sharing ($4$ banks), under the protocol of Table~\ref{tab:prior-ablation}. Bold: the better cache within each pair. Table~\ref{tab:deep-r-kv} adds $R=32$ and the S and L scales.}
\label{tab:deep-r-kv-m}
\fontsize{9}{10.8}\selectfont\renewcommand{\arraystretch}{\resultrowstretch}
\begin{NiceTabular}{@{\hspace{6pt}}lclccccccccccccc@{\hspace{6pt}}}
\toprule
 & & Cache & \multicolumn{2}{c}{PPL $\downarrow$} & \multicolumn{8}{c}{Downstream (\%) $\uparrow$} & \multicolumn{3}{c}{Generation (\%) $\uparrow$} \\
\cmidrule(lr){4-5}\cmidrule(lr){6-13}\cmidrule(lr){14-16}
Prior & $R$ & (banks) & Val. & Wiki & LAMB. & HS & PIQA & ARC-E & ARC-C & OBQA & SciQ & AVG & GSM & DROP & MBPP+ \\
\midrule
Fixed PLN-5 & 5 & Full (12) & 4.00 & 12.57 & 39.18 & 44.05 & 67.25 & 55.47 & 29.69 & \textbf{17.20} & 88.50 & 48.76 & 12.89 & 17.14 & 8.75 \\
 &  & Terminal (4) & 4.00 & \textbf{12.55} & \textbf{39.69} & \textbf{44.64} & \textbf{67.68} & \textbf{56.61} & \textbf{29.78} & 16.40 & \textbf{88.60} & \textbf{49.06} & \textbf{13.72} & \textbf{17.75} & \textbf{10.38} \\
\cmidrule(lr){2-16}
 & 16 & Full (34) & 3.98 & 12.37 & 40.21 & \textbf{44.70} & 67.63 & \textbf{56.48} & \textbf{29.52} & 17.60 & 88.50 & 49.23 & 14.78 & 18.88 & 7.88 \\
 &  & Terminal (4) & 3.98 & 12.37 & \textbf{40.31} & 44.63 & \textbf{67.79} & 56.44 & 29.35 & \textbf{17.80} & 88.50 & \textbf{49.26} & 14.78 & \textbf{19.56} & 7.88 \\
\midrule
Learned, $\lambda_H{=}0.01$ & 5 & Full (12) & \textbf{3.92} & \textbf{12.24} & 38.40 & 45.31 & \textbf{67.74} & 56.69 & \textbf{28.84} & 16.80 & 87.30 & 48.73 & 14.86 & \textbf{17.92} & \textbf{12.62} \\
 &  & Terminal (4) & 3.93 & 12.28 & \textbf{38.58} & \textbf{46.17} & 67.41 & \textbf{57.58} & 28.75 & \textbf{17.00} & \textbf{88.00} & \textbf{49.07} & \textbf{15.62} & 17.11 & 10.38 \\
\cmidrule(lr){2-16}
 & 16 & Full (34) & 3.93 & 12.18 & \textbf{37.88} & 45.91 & \textbf{67.57} & \textbf{57.15} & 28.84 & \textbf{18.40} & 87.90 & \textbf{49.09} & 16.68 & \textbf{19.41} & \textbf{10.38} \\
 &  & Terminal (4) & 3.93 & 12.18 & 37.78 & \textbf{45.94} & 67.46 & 57.11 & 28.84 & 18.20 & \textbf{88.00} & 49.05 & \textbf{16.98} & 19.35 & 8.88 \\
\bottomrule
\end{NiceTabular}
\end{table}

\subsection{Independent Prefill and Decode Budgets}
\label{app:cache-results}\label{app:cache-budgets}\label{app:pd-results}
Terminal KV sharing lets prefill depth $R_{\rm pre}$ and decode depth $R_{\rm dec}$ vary independently: every decoding visit to a physical attention layer reads the same terminal bank, so no decoding visit needs a matching prefill visit. $R_{\rm pre}$ controls only how refined this bank is, while $R_{\rm dec}$ controls how many recurrences each new token receives. Under the convergence conditions of Lemma~\ref{lem:joint-frozen}, deeper prefill moves the bank toward its fixed point, so decoding should improve with $R_{\rm pre}$ and then saturate. Figure~\ref{fig:pd-scaling} confirms this for Small checkpoints trained at mean recurrence $5$ and $25$ once the decode depth is at least $5$ (at least $2$ for the mean-$5$ model): beyond $R_{\rm pre}=5$, cached log-perplexity changes by at most $0.01$ at every decode depth. At the shallowest decode depths, a matched shallow prefill does better; for the mean-$25$ model at $R_{\rm dec}=2$, $R_{\rm pre}=1$ gives $2.89$ against $3.13$ for $R_{\rm pre}\geq5$. Decode depth matters more, especially for the mean-$25$ model, where raising $R_{\rm dec}$ from $1$ to $10$ lowers log-perplexity from $3.47$ to $1.82$ before it plateaus.

\begin{figure}[!htbp]
\centering
\includegraphics[width=\linewidth]{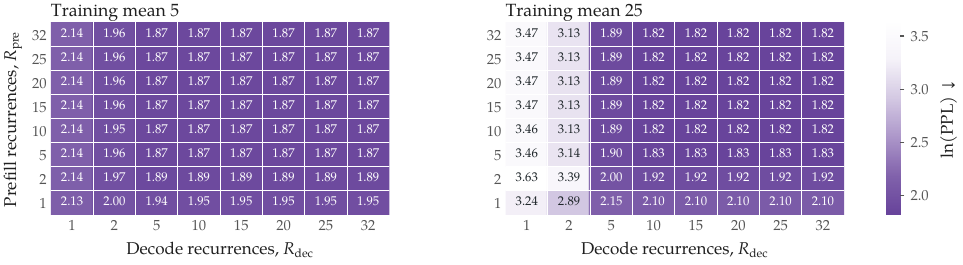}
\caption{\textbf{Prefill gains saturate by $R_{\rm pre}=5$, and decode depth matters more.} Cells give the cached log-perplexity of Small checkpoints trained at mean recurrence $5$ (left) and $25$ (right), on CacheEval with terminal KV sharing, $1{,}024$ prefix tokens, and continuation positions $2$--$512$.}
\label{fig:pd-scaling}
\end{figure}

Both grids use checkpoints with Parcae w/o norm injection, reuse the terminal history, and keep the current token's KV local to each visit, with $2+2R_{\rm pre}$ block applications per prefill token and $2+2R_{\rm dec}$ per incremental input. Figure~\ref{fig:pd-scaling} evaluates $R_{\rm pre},R_{\rm dec}\in\{1,2,5,10,15,20,25,32\}$ on CacheEval. KV allocation stays constant across each grid's $64$ settings.

\subsection{Distilled Prefill}
\label{app:distillation-results}
We first evaluate teacher-initialized students on the full battery of \secref{sec:exp-distill} (\appref{app:distill-ablation}) and then distill a deeper $R=25$ teacher (\appref{app:distill-r25}).
\subsubsection{Teacher-Initialized Students under the Full Battery}
\label{app:distill-ablation}
\paragraph{Evaluation.}
Table~\ref{tab:distilled-prefill-main} and Tables~\ref{tab:distilled-prefill-ablation} and~\ref{tab:distilled-prefill-r25} use the data, requests, and scorers of Table~\ref{tab:prior-ablation}. The prefill path processes the prompt, and the teacher at $R=5$ processes everything after it while reading the four terminal KV banks the path produced. For Val.\ and Wiki, the path prefills the prompt part of the document that spans the first $4{,}096$ tokens of each $8{,}192$-token sequence ($1{,}216$ and $38$ sequences), and the teacher scores that document\'s remaining tokens (\appref{app:setup-terminal}). For the seven choice tasks, the path prefills every context token but the last, and the teacher scores the last context token and the candidate continuation. For GSM8K, DROP, and MBPP+, the path prefills the prompt and the teacher decodes with the settings of \appref{app:generation-protocol}. The teacher rows at $R=3$ and $R=1$ prefill with the teacher stopped after that many recurrences and keep each physical layer's final-recurrence bank, as along the prefill-depth axis of Figure~\ref{fig:pd-scaling}. Untied~4 and the standalone student decode themselves with their own KV. The teacher row therefore uses terminal KV throughout and equals the $\lambda_H=0.01$ rows of Table~\ref{tab:prior-ablation}; with the full cache, its downstream averages are $41.74$ on Small and $48.69$ on Medium (Table~\ref{tab:prior-ablation-fullkv}). Latencies are paired means of six warmed repetitions at batch eight, timed in the teacher's process; Untied~4 is timed the same way in a separate job.

\paragraph{Teacher and student.}
The Small and Medium teachers in \secref{sec:exp-distill} are the learned-prior models of Table~\ref{tab:prior-ablation} ($\lambda_H=0.01$) with Parcae input injection and no prelude RMSNorm. They were trained on $21.47$B and $85.90$B tokens with learning rates $2\times10^{-3}$ and $1.5\times10^{-3}$, respectively. The Large teacher is the $1.6$B model of the same table, trained on $343.60$B tokens at $10^{-3}$. Targets are pre-coda states after five recurrent visits. The two-layer students match their teachers' widths ($51.9$M parameters on Small, $207.6$M on Medium, and $830.5$M on Large). Neither student shares weights between its two layers. Equation~\eqref{eq:hidden-distillation} uses energy normalization calibrated on $64$ training rows.

\paragraph{Student training.}
At each scale, held-out hidden-state MSE on $128$ prompts selects among learning rates $3\times10^{-4}$, $10^{-3}$, $3\times10^{-3}$, and $6\times10^{-3}$. All run $128$ updates, the top two continue to $512$, and the selected $10^{-3}$ reaches $2{,}560$ updates on Small and $10{,}240$ on Medium. The cosine schedule uses eight warmup updates and reaches $10^{-4}$ at update $2{,}048$, retaining that rate thereafter. All stages use GBS $256$ at sequence length $8{,}192$, so each final student sees one quarter of its teacher's training tokens ($5.369$B on Small and $21.475$B on Medium); training excludes held-out rows.

\paragraph{Teacher initialization.}
Teacher-initialized students add a $d\times d$ input map to the two-block student ($2.36$M parameters on Small, $9.44$M on Medium, $37.75$M on Large). It is set to $\operatorname{diag}(\Delta)W$, so it outputs the injected input $q$ of equation~\eqref{eq:additive-input-injection}, and the two blocks copy the recurrent core. Because the teacher's initial state is small (standard deviation $d^{-1/2}$ per coordinate), the untrained student approximately reproduces the teacher's state after one recurrence. They train for the same $2{,}560$ and $10{,}240$ updates at GBS $256$ as the randomly initialized students, directly at learning rate $10^{-3}$, the rate selected for those students. A control with the same input map, objective, and rate but random initialization (random, +KL) isolates the initialization. The Large student has no such control; it trains for $40{,}960$ updates, one quarter of its teacher's tokens ($85.90$B), with the same rate and schedule. The KL loss is the second term of equation~\eqref{eq:hidden-distillation} and has weight one.

\begin{table}[htbp]
\centering\resulttablestyle
\caption{\textbf{Teacher initialization helps the prefill path, and the KL loss raises its generation scores.} Same battery as Table~\ref{tab:distilled-prefill-main}, $R=5$; AVG is the mean over its seven choice tasks. \emph{+1 rec.}\ applies one teacher recurrence to the student state (our path; the teacher-init +KL row is the main recipe), \emph{KV direct} feeds the student's own KV to the teacher, and \emph{Standalone} decodes with the student alone. +KL adds the second term of equation~\eqref{eq:hidden-distillation}.}
\label{tab:distilled-prefill-ablation}
\renewcommand{\arraystretch}{\resultrowstretch}
\begin{NiceTabular}{@{\hspace{6pt}}lllcccccccc@{\hspace{6pt}}}
\CodeBefore
\rowcolor{papergray}{3,14}
\Body
\toprule
\multicolumn{3}{c}{Prefill path} & \multicolumn{2}{c}{PPL $\downarrow$} & \multicolumn{2}{c}{Downstream (\%) $\uparrow$} & \multicolumn{3}{c}{Generation (\%) $\uparrow$} & Latency $\downarrow$ \\
\cmidrule(lr){1-3}\cmidrule(lr){4-5}\cmidrule(lr){6-7}\cmidrule(lr){8-10}\cmidrule(lr){11-11}
Use & Init & Loss & Val. & Wiki & LAMB. & AVG & GSM & DROP & MBPP+ & 8K (ms) \\
\midrule
\multicolumn{11}{l}{\textbf{S scale: (100M, 21.5B tokens)}} \\
Teacher &  &  & 5.67 & 19.91 & 27.36 & 41.65 & 1.59 & 14.08 & 0.62 & 195.0 \\
Untied 4 &  &  & 6.57 & 25.19 & 23.95 & 39.45 & 1.97 & 11.39 & 0.12 & 100.7 \\
\midrule
+1 rec. & random & hidden & 5.77 & 20.42 & 24.24 & 40.44 & 1.90 & 11.51 & 0.12 & 125.4 \\
 & random & +KL & 5.76 & 20.32 & 20.88 & 39.96 & 1.52 & 12.44 & 0.00 & 125.3 \\
 & teacher & hidden & 5.72 & 20.24 & 24.63 & 40.73 & 0.83 & 14.25 & 0.38 & 126.3 \\
 & teacher & +KL & 5.71 & 20.16 & 23.56 & 40.63 & 1.74 & 14.41 & 0.62 & 125.9 \\
\midrule
KV direct & teacher & hidden & 6.08 & 21.36 & 18.86 & 38.68 & 1.29 & 10.44 & 0.00 & 105.1 \\
 & teacher & +KL & 5.99 & 21.08 & 18.96 & 38.71 & 1.14 & 11.28 & 0.00 & 104.9 \\
\midrule
Standalone & teacher & hidden & 7.81 & 30.64 & 19.41 & 38.00 & 2.27 & 9.41 & 0.00 & 100.1 \\
 & teacher & +KL & 6.97 & 26.73 & 21.35 & 38.22 & 2.12 & 9.92 & 0.00 & 99.9 \\
\midrule
\multicolumn{11}{l}{\textbf{M scale: (400M, 85.9B tokens)}} \\
Teacher &  &  & 3.93 & 12.28 & 38.58 & 49.07 & 15.62 & 17.11 & 10.38 & 495.8 \\
Untied 4 &  &  & 4.68 & 15.73 & 31.19 & 44.08 & 2.35 & 13.12 & 2.50 & 229.8 \\
\midrule
+1 rec. & random & hidden & 4.04 & 12.69 & 29.15 & 45.45 & 5.53 & 13.20 & 0.25 & 295.5 \\
 & random & +KL & 4.02 & 12.57 & 29.52 & 45.66 & 7.05 & 15.36 & 2.50 & 295.3 \\
 & teacher & hidden & 3.98 & 12.51 & 33.22 & 46.88 & 8.42 & 15.21 & 5.38 & 299.1 \\
 & teacher & +KL & 3.97 & 12.44 & 31.86 & 46.77 & 8.95 & 16.20 & 6.12 & 299.3 \\
\midrule
KV direct & teacher & hidden & 4.35 & 13.66 & 9.99 & 39.65 & 2.12 & 7.21 & 0.00 & 235.6 \\
 & teacher & +KL & 4.60 & 14.27 & 15.85 & 40.81 & 1.29 & 9.39 & 0.00 & 235.2 \\
\midrule
Standalone & teacher & hidden & 6.24 & 22.14 & 22.67 & 40.64 & 1.97 & 11.22 & 0.38 & 226.6 \\
 & teacher & +KL & 5.14 & 17.56 & 29.40 & 43.75 & 2.58 & 12.47 & 1.00 & 226.0 \\
\bottomrule
\CodeAfter
\begin{tikzpicture}
\foreach \sepcol in {4,6,8,11}{
\draw[draw=resultgrouprule,line width=0.35pt] (1-|\sepcol) -- (3-|\sepcol);
\draw[draw=resultgrouprule,line width=0.35pt] (4-|\sepcol) -- (14-|\sepcol);
\draw[draw=resultgrouprule,line width=0.35pt] (15-|\sepcol) -- (25-|\sepcol);
}
\end{tikzpicture}
\end{NiceTabular}
\end{table}

\paragraph{Results.}
Teacher initialization improves the prefill path (Table~\ref{tab:distilled-prefill-ablation}). With the hidden loss alone, it raises the Medium downstream average from $45.45$ to $46.88$, GSM8K from $5.53$ to $8.42$, and MBPP+ from $0.25$ to $5.38$. On Small, where generation is near zero for every model, it raises DROP from $11.51$ to $14.25$. With the KL term in both, it raises the downstream average from $45.66$ to $46.77$ on Medium and from $39.96$ to $40.63$ on Small, and Medium MBPP+ from $2.50$ to $6.12$. Adding the KL term, as in the main recipe, keeps the downstream average within $0.1$ point and raises all three generation scores on both scales; on Medium, GSM8K rises from $8.42$ to $8.95$, DROP from $15.21$ to $16.20$, and MBPP+ from $5.38$ to $6.12$. Feeding the student's own KV to the teacher (KV direct) costs $2$--$7$ downstream points relative to +1 rec. Running the student without the teacher (standalone) depends most on the KL loss: it raises the Medium average from $40.64$ to $43.75$ and lowers Val.\ PPL from $7.08$ to $5.88$. Even then, neither use beats Untied~4, which has the same physical depth.

\subsubsection{A Deeper Teacher at \texorpdfstring{$R=25$}{R=25}}
\label{app:distill-r25}
Distilling a deeper endpoint should save more prefill time. We trained Small and Medium learned-prior teachers ($\lambda_H=0.01$) whose PLN prior has mean $25$, on the same $21.47$B and $85.90$B tokens at learning rates $2\times10^{-3}$ and $1.5\times10^{-3}$, and evaluate them at $R=25$. Their students follow the $R=5$ recipe and budget. Table~\ref{tab:distilled-prefill-r25} places them next to the $R=5$ results.

\begin{table}[htbp]
\centering\resulttablestyle
\caption{\textbf{Distilling a deeper endpoint saves more prefill time but loses more quality.} Mean-$25$ learned-prior teachers at $R=25$ next to the $R=5$ teachers of Table~\ref{tab:distilled-prefill-main}, same battery; AVG is the mean over its seven choice tasks. All students use the hidden loss alone at both depths, unlike the main recipe; standalone students decode themselves. The M $R=25$ teacher-initialized rows reuse the timing of the randomly initialized student.}
\label{tab:distilled-prefill-r25}
\renewcommand{\arraystretch}{\resultrowstretch}
\begin{NiceTabular}{@{\hspace{6pt}}llcccccccc@{\hspace{6pt}}}
\CodeBefore
\rowcolor{papergray}{3,14}
\Body
\toprule
\multicolumn{2}{c}{Prefill} & \multicolumn{2}{c}{PPL $\downarrow$} & \multicolumn{2}{c}{Downstream (\%) $\uparrow$} & \multicolumn{3}{c}{Generation (\%) $\uparrow$} & Latency $\downarrow$ \\
\cmidrule(lr){1-2}\cmidrule(lr){3-4}\cmidrule(lr){5-6}\cmidrule(lr){7-9}\cmidrule(lr){10-10}
Teacher $R$ & Path & Val. & Wiki & LAMB. & AVG & GSM & DROP & MBPP+ & 8K (ms) \\
\midrule
\multicolumn{10}{l}{\textbf{S scale: (100M, 21.5B tokens)}} \\
5 & Teacher & 5.67 & 19.91 & 27.36 & 41.65 & 1.59 & 14.08 & 0.62 & 195.0 \\
 & One recurrence & 5.95 & 21.02 & 19.29 & 38.05 & 1.52 & 10.86 & 0.00 & 110.7 \\
 & Distilled, random init & 5.77 & 20.42 & 24.24 & 40.44 & 1.90 & 11.51 & 0.12 & 125.4 \\
 & Distilled, teacher init & 5.72 & 20.24 & 24.63 & 40.73 & 0.83 & 14.25 & 0.38 & 126.3 \\
 & Standalone & 7.81 & 30.64 & 19.41 & 38.00 & 2.27 & 9.41 & 0.00 & 100.1 \\
\midrule
25 & Teacher & 5.41 & 18.59 & 28.82 & 41.66 & 2.73 & 13.99 & 1.00 & 648.4 \\
 & One recurrence & 6.36 & 22.45 & 10.44 & 33.84 & 1.52 & 3.44 & 0.00 & 132.7 \\
 & Distilled, random init & 5.67 & 19.46 & 20.20 & 38.64 & 1.67 & 11.80 & 0.00 & 156.1 \\
 & Distilled, teacher init & 5.57 & 19.20 & 22.65 & 39.25 & 1.52 & 11.33 & 0.12 & 151.0 \\
 & Standalone & 10.17 & 37.43 & 18.82 & 36.19 & 2.73 & 7.52 & 0.00 & 99.2 \\
\midrule
\multicolumn{10}{l}{\textbf{M scale: (400M, 85.9B tokens)}} \\
5 & Teacher & 3.93 & 12.28 & 38.58 & 49.07 & 15.62 & 17.11 & 10.38 & 495.8 \\
 & One recurrence & 4.23 & 13.12 & 13.76 & 41.19 & 2.27 & 8.56 & 0.12 & 245.9 \\
 & Distilled, random init & 4.04 & 12.69 & 29.15 & 45.45 & 5.53 & 13.20 & 0.25 & 295.5 \\
 & Distilled, teacher init & 3.98 & 12.51 & 33.22 & 46.88 & 8.42 & 15.21 & 5.38 & 299.1 \\
 & Standalone & 6.24 & 22.14 & 22.67 & 40.64 & 1.97 & 11.22 & 0.38 & 226.6 \\
\midrule
25 & Teacher & 3.91 & 12.17 & 41.61 & 49.45 & 21.38 & 20.11 & 11.25 & 1847.7 \\
 & One recurrence & 8.95 & 43.04 & 0.62 & 32.95 & 0.00 & 3.82 & 0.00 & 289.9 \\
 & Distilled, random init & 4.20 & 12.92 & 19.87 & 42.09 & 1.74 & 10.30 & 0.00 & 342.3 \\
 & Distilled, teacher init & 4.05 & 12.55 & 30.04 & 45.54 & 4.32 & 14.99 & 1.38 & 342.3 \\
 & Standalone & 8.30 & 27.84 & 18.40 & 39.43 & 1.82 & 6.39 & 0.00 & 227.1 \\
\bottomrule
\CodeAfter
\begin{tikzpicture}
\foreach \sepcol in {3,5,7,10}{
\draw[draw=resultgrouprule,line width=0.35pt] (1-|\sepcol) -- (3-|\sepcol);
\draw[draw=resultgrouprule,line width=0.35pt] (4-|\sepcol) -- (14-|\sepcol);
\draw[draw=resultgrouprule,line width=0.35pt] (15-|\sepcol) -- (25-|\sepcol);
}
\end{tikzpicture}
\end{NiceTabular}
\end{table}

The savings grow with depth: at $8$K, distilled prefill is $4.3\times$ faster than the $R=25$ teacher on Small and $5.4\times$ on Medium, compared with $1.54\times$ and $1.66\times$ at $R=5$. The quality cost grows faster. The teacher-initialized student trails its $R=25$ teacher by $2.4$ and $3.9$ downstream points on Small and Medium, compared with $0.9$ and $2.2$ at $R=5$, and it loses most of the deeper Medium teacher's generation gain (GSM8K $21.38$ vs.\ $4.32$). Teacher initialization matters more at this depth: on Medium it raises the downstream average from $42.09$ to $45.54$. A single recurrence no longer suffices: its downstream average is $33.84$ on Small and $32.95$ on Medium.

As a result, the $R=25$ student is both slower and less accurate than the $R=5$ student at each scale ($151.0$ vs.\ $126.3$\,ms and $39.25$ vs.\ $40.73$ on Small; $342.3$ vs.\ $299.1$\,ms and $45.54$ vs.\ $46.88$ on Medium). Its lower Small PPL comes from the stronger $R=25$ decoder, which scores every continuation token itself; relative to its own teacher, the $R=25$ student's PPL gap is $1.7$--$3.5$$\times$ the $R=5$ gap on both scales. It is faster than the $R=5$ teacher but falls well below it. Hidden-state error does not explain the gap. The depth-$25$ endpoint lies farther from the prelude output on both scales (copying the input leaves a relative error of $0.945$ vs.\ $0.856$ on Small and $0.997$ vs.\ $0.942$ on Medium). The Small students fit it less well (held-out relative MSE $0.070$--$0.100$ vs.\ $0.058$), but the Medium students fit it better ($0.011$--$0.017$ vs.\ $0.036$) and still lose more quality. The deeper teacher's decoder appears more sensitive to small handoff errors. A stronger teacher need not give a better handoff, consistent with evidence that a larger teacher--student gap can weaken distillation~\citep{cho2019efficacy,mirzadeh2020improved}. We therefore report $R=5$ in the main text.

\subsection{Reusing States during Post-Training}
\label{app:posttraining}
This section specifies the rollout-state reuse estimator of \secref{sec:why-rl}, gives the training protocol, and reports the training cost, uncertainty, and diagnostics behind \secref{sec:exp-rl}, and a code RL comparison on MBPP+. Saved states are used only before the optimizer step changes the parameters, and missing, stale, or token-mismatched states are rejected.

\paragraph{Forward state and gradient estimator.}
For a saved endpoint $z^R$, token sequence $x$, and input-conditioned recurrent map $F_\theta(z;x)$ of \secref{sec:fixed-points}, Algorithm~\ref{alg:rl-reuse} gives the update. Its only recorded graph is one application of $F_\theta$ at the saved endpoint, which includes the prelude through input injection, and $\widetilde z$ keeps the forward value $z^R$. With $J_R=\partial_{z}F_\theta(z^R;x)$ as in \appref{app:gradients}, the hook replaces the readout gradient $g$ with $v=\sum_{j=0}^{b-1}(J_R^{\top})^{j}g$. Neumann-4 ($b=5$) applies four products, so its sum has five terms, including the identity; Neumann-3 stops after three products (four terms). The number in each Neumann label thus counts vector--Jacobian products. Neumann-4 runs one recurrence forward, four backward passes to the state only, and one full backward pass, all through the same graph. The parameter gradient combines $(\partial_\theta F_\theta)^\top v$ with direct readout derivatives; differentiable attention over saved states preserves cross-token derivatives. Recompute + Neumann-4 reconstructs $z^R$ without gradients before applying this estimator; full BPTT differentiates all six recurrent visits. This endpoint Neumann estimate is approximate and carries no convergence guarantee. Reuse and recomputation differ only in the no-gradient forward pass that recomputation adds. At $R=6$, the four state products of Neumann-4 are close to the six recurrent visits of full BPTT, so this experiment does not test how the saving scales with depth.

\paragraph{Model, data, and reward.}
The shared checkpoint in \secref{sec:exp-rl} is the Large learned-prior model of Table~\ref{tab:prior-ablation} ($\lambda_H=0.01$, $343.6$B pretraining tokens). Its input injection is Parcae w/o norm, the injection used by all Table~\ref{tab:prior-ablation} rows. Rollout and evaluation use visit-specific KV. Prompts use the eight-shot chain-of-thought format of \appref{app:generation-protocol}; generation stops at the next question marker, and the binary reward requires the extracted numeric answer to match the reference. Training draws questions from the $7{,}473$-question GSM8K training split in a fixed shuffled order, excluding the eight demonstrations and two probe questions. Evaluation uses a fixed random sample of $500$ test questions. We did not verify pretraining decontamination; all arms share the checkpoint, so contamination would shift absolute accuracy but not the comparison between arms.

\paragraph{Policy updates and sampling budget.}
Each arm samples $16$ responses per training question once (temperature $0.8$, $512$-token cap), forming $200$ two-question batches. Dr.~GRPO centers each question's rewards within its group of $16$ without standard-deviation scaling and divides the summed token loss by the batch size times the completion cap. Updates use a clipping range of $0.2$, no reference KL, the recorded rollout log-probabilities as the importance-ratio denominator, and the rollout temperature during replay. Zero-advantage batches skip the optimizer step, and the others receive one on-policy update, so arms with equal sampling budgets can differ in update and token counts.

Training uses replay microbatch two and gradient accumulation $16$, so each update uses $32$ responses. AdamW uses LR $10^{-6}$, moments $(0.9,0.95)$, $\epsilon=10^{-8}$, zero weight decay, gradient-norm clipping at one, and a four-update linear warmup. Evaluation uses the same seed for every endpoint. It draws eight samples per question with the rollout sampler; pass@1 averages the eight outcomes, and pass@8 counts a question as solved if any sample is correct.

\paragraph{Recurrence and task choice.}
Two measurements favor $R=6$ for the starting checkpoint. With $64$ samples per question at temperature $0.8$ on $500$ GSM8K training questions, $R=6$ reaches pass@1/8/64 of $38.1$/$73.8$/$91.4\%$, against $36.2$/$72.0$/$89.8\%$ at $R=5$. With greedy decoding on the test split (Table~\ref{tab:rl-depth-scan}), $R=6$ also gives the best GSM8K accuracy, while code and reading scores decline slowly through $R=7$ and fall from $R=8$. This greedy accuracy ($53.68$ at $R=6$) is higher than the sampled pass@1 at temperature $0.8$ that Table~\ref{tab:rl-state-reuse} reports before RL ($40.63$). Groups of $16$ samples contain both a correct and an incorrect response for $78\%$ of the training questions, compared with $65\%$ of the MBPP+ training tasks, so GSM8K gives more learning signal per rollout.

\begin{table}[!htbp]
\centering\resulttablestyle
\caption{\textbf{$R=6$ gives the best GSM8K accuracy for the RL starting checkpoint.} L learned-prior model ($\lambda_H=0.01$) with visit-specific KV; GSM8K decodes greedily on the full test split.}
\label{tab:rl-depth-scan}
\renewcommand{\arraystretch}{\resultrowstretch}
\begin{NiceTabular}{@{\hspace{6pt}}lccccccc@{\hspace{6pt}}}
\CodeBefore
\rowcolor{papergray}{2}
\Body
\toprule
$R$ & 5 & 6 & 7 & 8 & 10 & 15 & 20 \\
\midrule
\multicolumn{8}{l}{\textbf{L scale: (1.6B, 343.6B tokens)}} \\
GSM8K (greedy) & 50.27 & \textbf{53.68} & 52.31 & 52.01 & 49.96 & 48.90 & 49.73 \\
MBPP+ pass@1 & \textbf{28.88} & 28.75 & 27.88 & 25.88 & 22.50 & 20.88 & 20.75 \\
DROP F1 & \textbf{35.01} & 34.67 & 33.93 & 32.39 & 30.45 & 30.47 & 29.15 \\
\bottomrule
\end{NiceTabular}
\end{table}

\begin{table}[!htbp]
\centering\resulttablestyle
\caption{\textbf{Rollout generation dominates RL training time.} Realized counts and times for the GSM8K runs of Table~\ref{tab:rl-state-reuse}. Tokens counts sampled response tokens. Train (min) includes rollout, answer checking, and updates; Update (s) sums scoring, backward, and optimizer time. Prompts whose responses all receive the same reward skip replay, so per-update means fall below the medians in Table~\ref{tab:rl-state-reuse}.}
\label{tab:rl-training-cost}
\renewcommand{\arraystretch}{\resultrowstretch}
\begin{NiceTabular}{@{\hspace{6pt}}lccccc@{\hspace{6pt}}}
\CodeBefore
\rowcolor{papergray}{3}
\Body
\toprule
Training & \multicolumn{2}{c}{Realized} & \multicolumn{3}{c}{Time} \\
\cmidrule(lr){1-1}\cmidrule(lr){2-3}\cmidrule(lr){4-6}
Update method & Tokens & Updates & Train (min) & Rollout (min) & Update (s) \\
\midrule
\multicolumn{6}{l}{\textbf{L scale: (1.6B, 343.6B tokens)}} \\
Full BPTT & 1,009,546 & 188 & 65.45 & 57.42 & 463.5 \\
Recompute + Neumann-4 & 880,930 & 186 & 61.71 & 55.50 & 356.5 \\
Reuse + Neumann-4 & 886,381 & 191 & 57.30 & 52.97 & 244.1 \\
Reuse + Neumann-3 & 858,002 & 192 & 58.98 & 54.95 & 225.9 \\
\bottomrule
\CodeAfter
\begin{tikzpicture}
\foreach \sepcol in {2,4}{
\draw[draw=resultgrouprule,line width=0.35pt] (1-|\sepcol) -- (3-|\sepcol);
\draw[draw=resultgrouprule,line width=0.35pt] (4-|\sepcol) -- (8-|\sepcol);
}
\end{tikzpicture}
\end{NiceTabular}
\end{table}

\paragraph{Quality and uncertainty.}
Pass@1 gains over initialization in Table~\ref{tab:rl-state-reuse} are $22.58$ ($[20.28,24.90]$) points for full BPTT, $19.43$ ($[17.28,21.58]$) for recomputation, $21.03$ ($[18.85,23.23]$) for Reuse + Neumann-4, and $20.68$ ($[18.48,22.93]$) for Reuse + Neumann-3. Reuse + Neumann-4 differs from full BPTT by $-1.55$ ($[-3.13,0.05]$) points and from recomputation by $+1.60$ ($[0.08,3.18]$); Neumann-3 differs from Neumann-4 by $-0.35$ ($[-2.05,1.33]$). All $95\%$ intervals use $5{,}000$ paired question-bootstrap draws, resampling questions together with all eight of their responses. They condition on one training seed and so exclude seed variance. Reuse and recomputation compute nearly identical gradients (cosine $0.997$--$0.999$, below), so their $+1.60$ difference indicates the size of run-to-run variation; differences of this size between arms do not show an estimator effect.

\paragraph{Paired update cost and diagnostics.}
Timings in Table~\ref{tab:rl-state-reuse} are per-update medians from the training runs. A paired probe with fixed actions, masks, actor probabilities, and synthetic advantages on two probe questions (two warmups, five repetitions) gives $169.7$, $131.6$, and $89.2$\,ms for full BPTT, recomputation, and Reuse + Neumann-4. Whole-model gradient cosines are $0.997$--$0.999$ between reuse and recomputation and $0.76$--$0.87$ between reuse and full BPTT ($0.74$--$0.86$ with Neumann-3).

At $R=6$, the relative state residual averages $0.034$ (maximum 0.041 in the Neumann-4 run), larger than near a converged fixed point, so the finite Neumann series truncates more of the full gradient. Reuse's maximum replay log-probability RMSE is $0.021$, and its minimum normalized importance-weight effective sample size is $0.9995$. Detached BF16 states stay on the GPU for one batch and use about $2.0$\,GiB without host transfer; the time to capture them is counted as rollout time. The longest full-BPTT probe peaks at $49.2$\,GiB after Adam-state allocation.

\paragraph{Total cost and scope.}
Table~\ref{tab:rl-training-cost} reports total costs, which also depend on sampled lengths and nonzero-update counts. Multiple seeds, the full test split, and reuse across updates remain untested.

\paragraph{Transfer to MATH500.}
We evaluate the starting checkpoint and the four GSM8K endpoints on the $500$ MATH500 problems with the Minerva four-shot prompt, a $1{,}024$-token cap, generation stopped at the next problem marker, and Minerva-normalized exact match. Sampling matches the GSM8K evaluation ($R=6$, temperature $0.8$, eight samples per problem), and each problem uses the same seed for every model. RL raises pass@1 by $4.88$ ($[3.55,6.22]$) points for full BPTT, $3.92$ ($[2.73,5.12]$) for recomputation, $4.42$ ($[3.17,5.67]$) for Reuse + Neumann-4, and $5.38$ ($[3.95,6.78]$) for Reuse + Neumann-3. Reuse + Neumann-4 differs from full BPTT by $-0.45$ ($[-1.50,0.68]$) points and from recomputation by $+0.50$ ($[-0.55,1.50]$), using the paired question bootstrap above.

\paragraph{Code RL on MBPP+.}
We repeat the comparison on MBPP+ with the same checkpoint, optimizer, and update methods, using $R=6$, group size $16$, and two tasks per update. Prompts use the three-shot MBPP+ format of \appref{app:generation-protocol}, and generation stops at the first test, docstring, or top-level call after the function. The reward is one only if the sanitized solution passes both the EvalPlus base and plus tests, executed in a network-isolated sandbox. The training pool has $273$ tasks in a fixed shuffled order, and one pass of $136$ two-task batches uses the first $272$. These tasks are disjoint from the $100$ held-out evaluation tasks and the three demonstrations, with exact prompt and solution duplicates removed.

\begin{table}[!htbp]
\centering\resulttablestyle
\caption[Rollout-state reuse during code RL (MBPP+).]{\textbf{Reusing rollout states also cuts the update time of code RL.} Same checkpoint, $R=6$, optimizer, and update methods as Table~\ref{tab:rl-state-reuse}. Training covers $272$ MBPP+ tasks once; evaluation draws eight samples per task on $100$ held-out tasks. Update is the median scoring plus backward time per update; Train includes rollout, code verification, and updates. Before RL resamples the starting checkpoint with the RL sampler; its pass@1 differs from Table~\ref{tab:rl-depth-scan} ($28.75$, independent samples) by $2.1$ points ($95\%$ interval $[-1.4,5.8]$).}
\label{tab:rl-mbpp}
\renewcommand{\arraystretch}{\resultrowstretch}
\begin{NiceTabular}{@{\hspace{6pt}}lccccc@{\hspace{6pt}}}
\CodeBefore
\rowcolor{papergray}{3}
\Body
\toprule
Training & \multicolumn{2}{c}{MBPP+ (\%) $\uparrow$} & \multicolumn{3}{c}{Time $\downarrow$} \\
\cmidrule(lr){1-1}\cmidrule(lr){2-3}\cmidrule(lr){4-6}
Update method & Pass@1 & Pass@8 & Update (s) & Train (min) & Rollout (min) \\
\midrule
\multicolumn{6}{l}{\textbf{L scale: (1.6B, 343.6B tokens)}} \\
Before RL & 30.88 & 63.00 & \resultmissing & \resultmissing & \resultmissing \\
Full BPTT & 40.88 & 69.00 & 1.03 & 20.0 & 14.4 \\
Recompute + Neumann-4 & 37.12 & 69.00 & 0.65 & 18.9 & 14.9 \\
\midrule
Reuse + Neumann-4 & 39.38 & 64.00 & 0.51 & 21.8 & 17.8 \\
Reuse + Neumann-3 & 38.00 & 68.00 & 0.48 & 21.2 & 17.0 \\
\bottomrule
\CodeAfter
\begin{tikzpicture}
\foreach \sepcol in {2,4}{
\draw[draw=resultgrouprule,line width=0.35pt] (1-|\sepcol) -- (3-|\sepcol);
\draw[draw=resultgrouprule,line width=0.35pt] (4-|\sepcol) -- (9-|\sepcol);
}
\end{tikzpicture}
\end{NiceTabular}
\end{table}

All methods raise pass@1 from $30.9\%$ to $37$--$41\%$ (Table~\ref{tab:rl-mbpp}). Pass@8 moves by at most six points, and each held-out task is worth one point of it; Reuse + Neumann-4 ends five points below full BPTT ($64.0$ vs.\ $69.0$). Reuse + Neumann-4 differs from full BPTT by $-1.50$ ($[-4.38,1.38]$) points of pass@1 and from recomputation by $+2.25$ ($[-1.25,5.75]$); Neumann-3 differs from Neumann-4 by $-1.38$ ($[-5.12,2.25]$). Intervals use the same paired question bootstrap as on GSM8K. Reuse speeds up scoring and backward computation by $1.27\times$ over recomputation and $2.02\times$ over full BPTT. Rollout generation takes $72$--$82\%$ of training time and code verification another $12$--$15\%$, so sampling dominates end-to-end cost, as on GSM8K. The reuse arms are not faster end to end: their rollouts take $17.0$--$17.8$ minutes against $14.4$ for full BPTT, which outweighs their $1.5$--$1.6$-minute saving outside rollout.

\FloatBarrier

\FloatBarrier
\section{Datasets and Evaluation Protocols}
\label{app:protocols}
We define the shared datasets, target populations, and scoring rules. \appref{app:additional-results} gives each experiment's checkpoint, execution budget, and any departures from these common protocols. Each cached, split, or reused-state path for evaluation or gradients was checked against a reference (uncached scoring, full autograd, or an analytic oracle) on a subset of inputs before use; the experiment descriptions omit these checks. Unless stated otherwise, results use one training seed, intervals quantify evaluation-sample uncertainty, and ``batch eight'' means a batch size of eight.

\subsection{Mixture-Matched Validation}

\label{app:setup-mixture}

\textsc{MixVal} is a fixed mixture-matched split of $64$ packed sequences of
length $8{,}192$ from each of $19$ TxT360 validation sources, totaling $9{,}961{,}472$
next-token targets. We report $\mathrm{PPL}=\exp(\mathrm{CE})$ from
token-weighted teacher-forced cross-entropy, evaluated in BF16 with
multi-segment semantics and no generation cache. Document-level separation
from the training corpus remains unverified.

\subsection{WikiText-103 and Zero-Shot Tasks}
\label{app:setup-wiki}
Table~\ref{tab:task-inventory} lists the evaluation tasks.
\begin{table}[t]
\centering\papertablestyle\renewcommand{\appendixtablefontsize}{9}%
\caption{\textbf{Evaluation tasks.} Choice tasks are zero-shot and score candidates by log likelihood (length-normalized for HellaSwag and ARC-C, summed otherwise). Generation tasks decode greedily except MBPP+, which averages eight samples at temperature $0.8$.}
\label{tab:task-inventory}
\begin{tabular}{@{}lllrrl@{}}
\toprule
Task & Split & Metric & Examples & Shots & Used in \\
\midrule
LAMBADA (OpenAI) & test & accuracy, last-word PPL & 5,153 & 0 & main tables \\
HellaSwag & validation & accuracy & 10,042 & 0 & main tables \\
PIQA & validation & accuracy & 1,838 & 0 & main tables \\
ARC-Easy & test & accuracy & 2,376 & 0 & main tables \\
ARC-Challenge & test & accuracy & 1,172 & 0 & main tables \\
OpenBookQA & test & accuracy & 500 & 0 & main tables \\
SciQ & test & accuracy & 1,000 & 0 & main tables \\
WinoGrande & validation & accuracy & 1,267 & 0 & \appref{app:prior-results}, \appref{app:architecture-results} \\
\midrule
DROP & frozen subset & F1 & 500 & 3 & main tables \\
GSM8K & test & accuracy & 1,319 & 8 & main tables, \appref{app:posttraining} \\
MBPP+ v0.2.0 & frozen subset & pass@1 & 100 & 3 & main tables, \appref{app:posttraining} \\
\bottomrule
\end{tabular}
\end{table}

We concatenate the $2{,}891$ nonempty records of WikiText-103~\citep{merity2016pointer}
with newlines. Native curves in Figures~\ref{fig:fixed-vs-pln-depth-scaling}
and~\ref{fig:m-entropy-ablation} score $310{,}315$ next-token targets in $38$
windows of $8{,}192$ tokens at every depth $1$--$64$, using batch two, BF16,
FlashAttention-3 (FA3), and state seed $42$. PPL exponentiates token-averaged
cross-entropy; document-level pretraining decontamination was not performed.

Paired WikiText PPL in Figure~\ref{fig:m-entropy-ablation} uses the same
windows and example IDs but scores $271{,}403$ continuation targets after a
$1{,}024$-token prefill, excluding the prefill prediction and padding. Both
layouts start from an identical per-visit prefill at $R_{\rm pre}=R_{\rm dec}=5$, with batch eight
and FA3 disabled; terminal sharing retains each physical layer's last bank,
reducing twelve banks to four.

The main tables share six zero-shot tasks from the Mamba-3
suite~\citep{lahoti2026mamba3}, using complete labeled splits (LAMBADA test, HellaSwag and PIQA validation, and ARC and OpenBookQA test): LAMBADA OpenAI
($5{,}153$), HellaSwag ($10{,}042$), PIQA ($1{,}838$), ARC-Easy ($2{,}376$),
ARC-Challenge ($1{,}172$), and OpenBookQA ($500$). HellaSwag and ARC-Challenge
normalize continuation log likelihood by character length; PIQA, ARC-Easy,
and OpenBookQA sum it. LAMBADA accuracy requires every final-word subtoken
to match its greedy prediction; last-word PPL exponentiates the negative mean final-word log probability given the context, summing over subtokens.

Tables~\ref{tab:prior-ablation} and~\ref{tab:architecture} and their no-reuse counterparts add all $1{,}000$ SciQ
test questions at $R=5$, with the support paragraph followed by
\texttt{Question: \{question\}} and then \texttt{Answer:}. Four answers are scored
by summed log likelihood. Table~\ref{tab:m-entropy-paired-tasks} and the entropy-sweep averages in \appref{app:m-entropy-matched} instead add the WinoGrande~\citep{winogrande} validation split (1,267 examples),
scoring the common suffix after each candidate-filled prefix by summed log likelihood. Each
seven-task AVG equally weights its accuracies, excluding PPL and generation.

The main distilled-prefill results (Table~\ref{tab:distilled-prefill-main})
use this battery, with each prompt split between the prefill path and the
teacher decoder (\appref{app:distill-ablation}).

Request formatting and pair tokenization use
\texttt{lm-evaluation-harness v0.4.2}~\citep{eval_harness}, commit \texttt{4600d6bf}, pinned
for this implementation; Mamba-3 does not specify a harness revision.
Native scoring uses BF16 with FA3, candidate batch eight, and no BOS, EOS, chat
template, or generation cache. State seed $42$ and fixed example IDs are
shared across candidates and checkpoints. Requests are right-padded into
token-determined power-of-two buckets of minimum width $128$; batches and
bucket shapes stay fixed across models. We keep all ARC questions, including those with three or five choices, and reject rather than truncate inputs over $8{,}192$ tokens.

The reward ablations (Table~\ref{tab:prior-reward-ablation}) instead use raw-likelihood scoring and a four-task mean.

\subsection{Terminal-KV Evaluation}
\label{app:setup-terminal}
Tables~\ref{tab:prior-ablation} and~\ref{tab:architecture} and the Teacher row of Table~\ref{tab:distilled-prefill-main} score every checkpoint at $R=5$ with terminal KV sharing, using the battery's requests and scorers. At $R=5$, a token visits the four physical attention layers twelve times: the prelude and coda once each, and each of the two core blocks five times. The cache keeps one bank per physical layer. Because the prelude and coda run once per token, their banks are exact; each core bank keeps only the final-recurrence KV of earlier tokens. A forward pass over a chunk of new tokens runs all five recurrences; at every core visit, the chunk attends to that layer's bank and to its own KV from the same visit, and after the final recurrence its KV is appended to the bank. From an empty cache this pass equals the native forward pass and differs only in what it stores; every checkpoint passes this check on \textsc{MixVal} prefixes.

Prefill passes every prompt token but the last as one chunk, then the last prompt token alone. For the seven choice tasks, the context is the prompt, and each candidate token is then passed alone, so every scored token reads earlier tokens only through terminal banks. Generation prefills the prompt the same way and decodes one token at a time with the settings of \appref{app:generation-protocol}. For Val.\ and Wiki, the prompt is the first $4{,}096$ tokens of each $8{,}192$-token sequence ($1{,}216$ \textsc{MixVal} and $38$ WikiText sequences), and attention stays within documents, as in training. PPL covers the tokens after the prompt that continue a document begun inside it, and the prediction from the last prompt token: $42\%$ of the remaining \textsc{MixVal} tokens, and every WikiText token, since WikiText sequences contain no document boundary. The prompt part of each such document is prefilled, and its later tokens read it through the terminal banks while attending to one another as in training, so Val.\ and Wiki test sharing of the prefix only. Tables~\ref{tab:prior-ablation-fullkv} and~\ref{tab:architecture-fullkv} evaluate the same checkpoints without KV reuse: Val.\ and Wiki are native PPL over full sequences (\appref{app:setup-mixture}), choice tasks are scored without a cache, and generation uses $R=5$ prefill and decoding with twelve visit-specific KV banks and a BOS token.

\subsection{Cached-Continuation Perplexity}

\label{app:setup-cache}
Cached PPL is teacher-forced next-token perplexity from the native incremental
decoder. Prefill's final-position logits score the first continuation target; appending each reference token once scores the next target. For the scored positions $\mathcal T_i$ of sequence $i$,
\begin{equation}
  \mathrm{CE}_{\mathrm{cache}}
  =-\frac{1}{\sum_i|\mathcal T_i|}\sum_i\sum_{t\in\mathcal T_i}
    \log p_\theta^{\mathrm{inc}}(x_{i,t}\mid x_{i,<t}),
  \qquad
  \mathrm{PPL}_{\mathrm{cache}}
  =\exp\!\left(\mathrm{CE}_{\mathrm{cache}}\right).
  \label{eq:setup-cached-ppl}
\end{equation}
Predicted EOS tokens do not stop scoring.

\textsc{CacheEval} contains $304$ frozen single-document \textsc{MixVal} spans
($16$ per source), each with $1{,}024$ prompt tokens and $512$ continuation targets.
The cache grows from $1{,}024$ to $1{,}535$ tokens over $511$ teacher-forced inputs.
Prior, input-injection, and training-depth comparisons score targets $2$--$512$
($155{,}344$ total), using five recurrences for both prefill and decoding;
\appref{app:pd-results} varies these budgets independently. Each experiment specifies its cache layout.

\subsection{Generation Benchmarks}
\label{app:generation-protocol}
All generation scores are percentages. DROP~\citep{drop} uses official F1 on $500$ frozen
examples (three-shot, greedy), GSM8K~\citep{gsm8k} uses numeric-answer accuracy
on all $1{,}319$ test examples (eight-shot, greedy), and MBPP+ uses
estimated pass@1 on $100$ frozen v0.2.0 tasks (three-shot, eight samples
per task, temperature $0.80$). MBPP+ augments MBPP~\citep{mbpp} with EvalPlus
tests~\citep{evalplus}. Generation caps are $64$ tokens for DROP and $512$
for GSM8K and MBPP+. RL uses the same eight GSM8K demonstrations;
\appref{app:posttraining} gives its training split and rollout protocol.
Each experiment specifies its model and recurrence budgets.

\FloatBarrier
\section{Extended Related Work}
\label{app:extended-related}
We position the learned depth prior, OrthoInj, terminal KV sharing, rollout-state reuse, and distilled prefill against related work that \secref{sec:related} covers only briefly.
\subsection{Learned Depth Distributions and Exits}
\label{app:related-depth}

Ouro, RL-Halting, and our prior all optimize an entropy-regularized expected objective over depths~\citep{ouro,kuo2026stochastic}. Ouro takes the exact expectation with per-token exit gates and reuses the gates for early exit.

\paragraph{Closest related work.} The closest related work is RL-Halting~\citep{kuo2026stochastic}, which also supervises one sampled depth per example and learns the depth distribution with REINFORCE. Its stopping head reads the detached, average-pooled hidden state after each iteration and outputs a hazard $\pi_r$, the probability of stopping at recurrence $r$ given no earlier stop, so $p(R=r\mid x)=\pi_r\prod_{r'<r}(1-\pi_{r'})$. The reward is the negative task loss with a moving-average baseline and entropy coefficient $0.01$, and inference stops at the most probable depth. Because the hazards depend on the states along the loop, training unrolls every example to its maximum depth of $20$ to form the distribution; the sampled depth selects the supervision but not the compute. A hazard that ignores the state is constant and gives a geometric distribution, so the head cannot start from a prescribed prior such as capped PLN. RL-Halting has no budget term and targets length extrapolation on synthetic tasks; it does not evaluate fixed points or KV sharing.

Our prior is unconditional and used only in training, giving it three advantages over Ouro and RL-Halting. First, it saves compute: it samples one depth per microbatch before the loop and runs only that depth, and the budget term in equation~\eqref{eq:depth-prior-objective} fixes the expected forward compute at $\bar R$. Ouro instead runs every step to maximum depth for the exact expectation, and RL-Halting unrolls to that maximum depth to form its state-dependent distribution. Second, it starts from a strong prior: its logits start at the capped PLN prior, so each learned run begins at the fixed-prior baseline, whereas a state-dependent hazard has no setting that reproduces a prescribed prior. Third, it deliberately promotes KV sharing: because the prior does not decide where inference stops, its entropy term extends supervision above $\bar R$, where settled prefix context supports terminal KV sharing (\appref{app:m-entropy-matched}). Table~\ref{tab:prior-reward-ablation} compares its inverse-perplexity reward to RL-Halting's raw CE.

Learned exits have two limits for fixed-point training. A single exit distribution both stops inference and weights the per-depth training losses; \citet{adaptive_depth} trace gate failures mainly to the trajectory this joint training induces, and find that fixed-prior depth supervision yields useful trajectories. An exit also stops once prediction suffices, so training need not reach depths beyond prefix convergence, where supervision prepares the model for terminal KV sharing (\secref{sec:inference}). Depth uncertainty also learns a global depth distribution, but within a Bayesian objective that marginalizes over depths at inference~\citep{dun}.

\subsection{Input Injection and Orthogonal Updates}
\label{app:related-injection}

DEQ-Transformer injects embeddings into Q/K/V; DEQ-TrellisNet injects them before convolutional LSTM gates~\citep[Section~4]{bai2019deep}. TrellisNet combines causal convolution, shared depth, and persistent injection~\citep{bai2019trellis}. Huginn and Parcae instead inject at the recurrent-core entrance~\citep{huginn,parcae}.

Orthogonal residual updates~\citep{orthogonal_residual} also apply a projection inside a residual block, but to the other operand. They keep only the component of a module's output that is orthogonal to the residual stream, so an update adds new directions without rescaling the stream. OrthoInj (equation~\eqref{eq:orthoinj}) instead projects the carried-over state off the injected input $q_t$ and then adds $q_t$, so the component along $q_t$ is fixed at every recurrence while the orthogonal part evolves freely. The two constraints act on different quantities: what a block adds, and what the recurrence retains.

\subsection{KV Sharing Across Recurrent Depth}
\label{app:related-kv-sharing}
Huginn reports zero-shot cyclic reuse of a fixed historical-KV budget, with
later recurrent visits overwriting earlier entries~\citep[Section~6.2]{huginn}.
Ouro's last-loop reuse during decoding largely preserves GSM8K accuracy
with a modest MATH-500~\citep{lightman2024verify} decrease, while prefill sharing degrades
performance~\citep[Section~5.4.2, Table~14]{ouro}. Even with visit-specific prompt banks, decode-time sharing represents the generated tokens by shared banks.

MELT instead finds zero scores on its Table~5 reasoning benchmarks for all four combinations of
first-/last-loop reuse and preserved/shared prompt caches on
Ouro-1.4B-Thinking~\citep[Appendix~B]{melt2026}. Its LightEval setup permits
$32$K-token completions. Differences from Ouro's few-shot CoT evaluation,
alongside incomplete reproduction of the original results, leave the
discrepancy unresolved~\citep[Section~4.1 and Appendix~F.1]{melt2026}.
MELT adapts the model to a shared cache with gated updates, chunk-wise training,
a gradual transition, and attention-aligned distillation; it does not
directly evaluate Huginn.

Under a common GSM8K setup without retraining,
\citet[Appendix~C.1, Table~1]{schwethelm2026cdb} find that one shared KV slot
per physical layer leaves Huginn accuracy nearly unchanged at $32$ recurrences, while
Ouro accuracy collapses at four. A separate first-visit slot recovers most of
Ouro's accuracy. These results suggest that cache compatibility depends on the model;
Huginn under MELT's long-reasoning protocol remains unmeasured. These studies test sharing zero-shot or after adaptation. We instead change only the training depth policy, which decides whether terminal sharing is nearly lossless (Table~\ref{tab:prior-ablation}); like them, we do not measure long-form reasoning.

\subsection{Fixed-Point Reuse}
\label{app:related-reuse}
Implicit models reuse a fixed point to warm-start the solver on a related input. DEQ-Flow and StreamDEQ start each video frame from the previous frame's fixed point~\citep{bai2022deqflow,ertenli2022streamdeq}, and fixed point diffusion models reuse solutions between denoising timesteps~\citep{bai2024fpdm}. In each case, the reused state initializes a new forward solve. Rollout-state reuse (\secref{sec:why-rl}) instead carries states between the two passes that RL makes over the same tokens, rollout and the policy update. The update scores the sampled tokens from the endpoints saved during rollout, so it does not replay the recurrence, and backpropagates through a single recurrence at each endpoint with a truncated Neumann series (\appref{app:posttraining}).

\subsection{Distillation for Looped Models}
\label{app:related-distill}
Distillation has trained looped students from non-looped teachers~\citep{bae2025relaxed} and adapted looped models to cheaper KV caches: MELT adapts Ouro to one gated cache per layer with attention-aligned distillation~\citep{melt2026}, and Looped Latent Attention distills a codec that compresses KV across loops~\citep{oneill2026lla}. LoopFormer trains shorter loop trajectories of the same model to match the final states of the full trajectory~\citep{jeddi2026loopformer}. In each case, the student still iterates the recurrent block. Distilled prefill instead trains a separate, non-recurrent student to predict the endpoint of a frozen looped teacher and uses it only for the prompt; decoding keeps the teacher's recurrence. The closest idea is the learned initializer of \mbox{HyperDEQ}, a small network that predicts the starting point of a frozen DEQ's fixed-point solver~\citep{bai2022hyperdeq}. Our student instead replaces the solver: one teacher recurrence from its prediction builds every KV bank that decoding reads under terminal sharing.

\FloatBarrier

\end{document}